\pdfoutput=1
\documentclass{article} 

\def\fixBookmarks{1}

\newcommand{\notarxiv}[1]{foo}
\newcommand{\arxiv}[1]{ba}
\ifdefined \arXiv
\renewcommand{\arxiv}[1]{#1}%
\renewcommand{\notarxiv}[1]{\ignorespaces}%
\else%
\renewcommand{\arxiv}[1]{\ignorespaces}%
\renewcommand{\notarxiv}[1]{#1}%
\fi%
\usepackage{microtype}
\usepackage{graphicx}
\usepackage{booktabs} %

\ifdefined \fixBookmarks
\else
\usepackage{hyperref}
\usepackage{url}
\fi

\usepackage{graphicx}
\usepackage[caption=false,font=normalsize,labelfont=sf,textfont=sf]{subfig}
\usepackage{amssymb}
\usepackage[export]{adjustbox}
\usepackage{enumitem}

\ifdefined \fixBookmarks
\else
\usepackage{hyperref}
\fi

\notarxiv{
	\usepackage[accepted]{icml2023}
}

\ifdefined \fixBookmarks
\usepackage{hyperref}
\usepackage{url}
\definecolor{mydarkblue}{rgb}{0,0.08,0.45} %
\hypersetup{
	colorlinks=true,
	linkcolor=blue!70!black,
	citecolor=blue!70!black,
	filecolor=blue!70!black,
	urlcolor=blue!70!black}
\fi

\usepackage{amsmath}
\usepackage{amssymb}
\usepackage{mathtools}
\usepackage{amsthm}

\arxiv{
	\usepackage[top=1in, right=1in, left=1in, bottom=1in]{geometry}
	\usepackage{algorithm}
	\usepackage{algorithmic}
	\usepackage{xcolor} 
	\usepackage{url}
	\hypersetup{
	colorlinks=true,
	linkcolor=blue!70!black,
	citecolor=blue!70!black,
	filecolor=blue!70!black,
	urlcolor=blue!70!black}
	\usepackage{natbib}
}

\usepackage[capitalize,noabbrev]{cleveref}

\theoremstyle{plain}
\newtheorem{theorem}{Theorem}[section]

\newtheorem{lemma}[theorem]{Lemma}

\theoremstyle{definition}
\newtheorem{definition}[theorem]{Definition}

\theoremstyle{remark}
\newtheorem{remark}[theorem]{Remark}

\usepackage[textsize=tiny]{todonotes}

\newcommand{\N}{\mathbb{N}}

\newcommand{\rbm}[1]{\left(#1\right)} %
\newcommand{\cb}[1]{\left\{#1\right\}} %
\newcommand{\norm}[1]{\left\Vert#1\right\Vert} %
\newcommand{\abs}[1]{\left\vert#1\right\vert} %
\newcommand{\vect}[1]{\boldsymbol{#1}} %

\newcommand{\xx}{\vect{x}}
\newcommand{\btheta}{\boldsymbol{\theta}}
\newcommand{\ip}[2]{\left\langle #1, #2\right\rangle} %
\newcommand{\bbeta}{\boldsymbol{\beta}}

\newcommand\remove[1]{}

\newcommand{\loss}{\mathcal{L}}
\newcommand{\R}{\mathbb{R}}
\newcommand{\A}{\mathcal{A}}

\newcommand{\ble}{\le_b}
\newcommand{\bequiv}{\hspace{-.05em}\overset{\text{b}}{\sim}\hspace{-.05em}}
\newcommand{\gequiv}{\hspace{-.05em}\overset{\text{GF}}{\sim}\hspace{-.05em}}
\newcommand{\logmaj}{\underset{\textup{log}}{\prec}}
\newcommand{\e}{e}
\newcommand{\maj}{\prec}

\newcommand{\Pgf}[1]{S_{\text{GF}}\left( #1 \right)}
\newcommand{\PgfNoArg}{S_{\text{GF}}}
\newcommand{\ps}[1]{\phi\rb{#1}}
\newcommand{\pss}[1]{\phi^2\rb{#1}}
\newcommand{\psc}{\phi_{\vect{b}}}
\newcommand{\rb}[1]{\left( #1\right)}
\newcommand{\sqb}[1]{\left[ #1\right]}
\newcommand{\iv}[2]{\left[ #1, #2\right]}
\newcommand{\ivr}[2]{\left( #1, #2\right)}
\newcommand{\ivi}[1]{\left[ #1, \infty\right)}

\newcommand{\gd}[2][\eta]{g_{#1}\rb{ {#2} }}
\newcommand{\gdb}[2][\eta]{\left[\gd[{#1}]{#2}\right]}
\newcommand{\qs}[1]{\tilde{g}_{\eta}\rb{ #1 }}
\newcommand{\qg}[1]{q_{\eta}\rb{ #1 }}
\newcommand{\qgs}{q_{\eta}}
\newcommand{\sm}[2][m]{\tilde{s}_{#1}\rb{ #2 }}
\newcommand{\smp}[3][m]{\tilde{s}^{#3}_{#1}\rb{ #2 }}
\newcommand{\sh}[1]{\lambda_{\max} \rb{ \nabla^2 \loss \rb{#1} }}

\newcommand{\SG}[1][\eta]{ \mathbb{S}^{D}_{#1} }
\DeclareMathOperator*{\argmin}{argmin} %
\DeclareMathOperator{\arcsinh}{arcsinh}
\DeclareMathOperator{\sign}{sign}

\newcommand{\crefi}[2]{%
	\hyperref[#2]{\namecref{#1}~\labelcref*{#1}.\ref*{#2}}%
}

\newcommand{\wt}[1][t]{\vect{w}^{\rb{#1}}}
\newcommand{\wst}[1][t]{\bar{\vect{w}}^{\rb{#1}}}
\newcommand{\ws}{\bar{w}}
\newcommand{\Rd}[1][D]{\R^{#1}}
\newcommand{\Rpd}[1][D]{\R_{+}^{#1}}
\newcommand{\D}[1][D]{\sqb{#1}}

\newcommand{\titleDamian}{\texorpdfstring{\citet{damian2022self}}{Damian et al. (2022)}~}
\newcommand{\titleSG}{\texorpdfstring{$\SG$}{the positive invariant set}}
\newcommand{\titleCrefi}[2]{%
	\texorpdfstring{\hyperref[#2]{\labelcref*{#1}.\ref*{#2}}}%
	{\ref*{#1}.\ref*{#2}}
}

\notarxiv{
\icmltitlerunning{Gradient Descent Monotonically Decreases the Sharpness of Gradient Flow Solutions}
}

\arxiv{
	\title{Gradient Descent Monotonically Decreases the Sharpness  of \\ Gradient Flow Solutions in Scalar Networks and Beyond}

	\author{Itai Kreisler\thanks{Equal contribution.}
		\thanks{Tel Aviv University,
			\href{mailto:kreisler@mail.tau.ac.il}{\texttt{kreisler@mail.tau.ac.il}},  \href{mailto:ycarmon@tauex.tau.ac.il}{\texttt{ycarmon@tauex.tau.ac.il}.}
		}~~~ Mor Shpigel Nacson\footnotemark[1] \thanks{
			Technion, \href{mailto:mor.shpigel@gmail.com}{\texttt{mor.shpigel@gmail.com}}, 
			\href{mailto:daniel.soudry@technion.ac.il}{\texttt{daniel.soudry@technion.ac.il}.}
		}~~~Daniel Soudry\footnotemark[3]~~~Yair Carmon\footnotemark[2]}

	\date{}
}

\begin{document}

\notarxiv{
\twocolumn[
\icmltitle{Gradient Descent Monotonically Decreases the Sharpness  of \\ Gradient Flow Solutions in Scalar Networks and Beyond}

\icmlsetsymbol{equal}{*}

\begin{icmlauthorlist}
\icmlauthor{Itai Kreisler}{equal,ta}
\icmlauthor{Mor Shpigel Nacson}{equal,technion}
\icmlauthor{Daniel Soudry}{technion}
\icmlauthor{Yair Carmon}{ta}
\end{icmlauthorlist}

\icmlaffiliation{ta}{Department of Computer Science,  Tel Aviv University, Israel}
\icmlaffiliation{technion}{Department of Electrical Engineering, Technion, Israel}

\icmlcorrespondingauthor{Itai Kreisler}{kreisler@mail.tau.ac.il}
\icmlcorrespondingauthor{Mor Shpigel Nacson}{mor.shpigel@gmail.com}

\icmlkeywords{Machine Learning, ICML}

\vskip 0.3in
]

\printAffiliationsAndNotice{\icmlEqualContribution} %
}

\arxiv{		
		\maketitle
}

\begin{abstract}
Recent research shows that when Gradient Descent (GD) is applied to neural networks, the loss almost never decreases monotonically. Instead, the loss oscillates as gradient descent converges to its ``Edge of Stability'' (EoS). Here, we find a quantity that does decrease monotonically throughout GD training: the sharpness attained by the gradient flow solution (GFS)---the solution that would be obtained if, from now until convergence, we train with an infinitesimal step size. Theoretically, we analyze scalar neural networks with the squared loss, perhaps the simplest setting where the EoS phenomena still occur. In this model, we prove that the GFS sharpness decreases monotonically. Using this result, we characterize settings where GD provably converges to the EoS in scalar networks. Empirically, we show that GD monotonically decreases the GFS sharpness in a squared regression model as well as practical neural network architectures.

\end{abstract}

\section{Introduction}
The conventional analysis of gradient descent (GD) for smooth functions assumes that its step size $\eta$ is sufficiently small so that the loss decreases monotonically in each step. In particular, $\eta$ should be such that the loss sharpness (i.e., the maximum eigenvalue of its Hessian) is no more than $\frac{2}{\eta}$ for the entire GD trajectory.
Such assumption is prevalent when analyzing GD in the context of general  functions \citep[e.g.,][]{lee2016gradient}, and even smaller steps sizes are used in theoretical analyses of GD in neural networks  \citep[e.g.,][]{du2019gradient, arora2018convergence, elkabetz2021continuous}.

However, recent empirical work \citep{cohen2021gradient, wu2018sgd, xing2018walk, gilmer2021loss} reveals that the descent assumption often fails to hold when applying GD to neural networks. Through an extensive empirical study, \citet{cohen2021gradient} identify two intriguing phenomena. The first is \emph{progressive sharpening}: the sharpness increases during training until reaching the threshold value of $\frac{2}{\eta}$. The second is the \emph{edge of stability} (EoS) phase: after reaching $\frac{2}{\eta}$, the sharpness oscillates around that value and the training loss exhibits \emph{non-monotonic} oscillatory behaviour while decreasing over a long time scale. 

Since the training loss and sharpness exhibit chaotic and oscillatory behaviours during the EoS phase \citep{zhu2022understanding}, we ask:

 \begin{center}
    \emph{During the EoS phase, is there a quantity \notarxiv{\\} that GD \emph{does} monotonically decrease?}
\end{center}

In this paper, we identify such a quantity, which we term the \emph{gradient flow solution (GFS) sharpness}. 
Formally, we consider minimizing a smooth loss function $\loss:\R^D\to \R$ using GD with step size $\eta>0$,
\[\vect{w}^{\rb{t+1}} = \vect{w}^{\rb{t}} -\eta \nabla\loss\rb{\vect{w}^{\rb{t}}}\]
 or gradient flow (GF)
\[\dot{\vect{w}}^{\rb{t}} = -\nabla\loss\rb{\vect{w}^{\rb{t}}}.\]
We denote by $\Pgf{\vect{w}}$ the gradient flow solution (GFS), i.e., the limit of the gradient flow trajectory when initialized at $\vect{w}$; see Figure \ref{Fig: gf trajectory} for illustration. Using this notion we define the GFS sharpness as follows.
\begin{definition}[GFS sharpness]\label{def: GFS sharpness}
The GFS sharpness of weight $\vect{w}$, written as $\ps{\vect{w}}$, is the sharpness of $\Pgf{\vect{w}}$, i.e., the largest eigenvalue of $\nabla^2 \loss \rb{\Pgf{\vect{w}}}$.
\end{definition}

\paragraph{Why is the GFS sharpness interesting?} GFS sharpness has two strong relations to the standard sharpness, a quantity central to the EoS phenomena and also deeply connected to generalization in neural networks~\citep[e.g.,][]{hochreiter1997flat,keskar2016large,foret2020sharpness,mulayoff2021implicit}. First, sharpness and GFS sharpness become identical when GD converges, since $\PgfNoArg$ approaches the identity near minima (where GF barely moves). Therefore, by characterizing the limiting behavior of GFS sharpness, we also characterize the limiting behavior of the standard sharpness. Second, if we use a common piecewise constant step-size schedule decreasing from large (near-EoS) to small (near-GF), then, at the point of the step size decrease, GFS sharpness is approximately the final sharpness. Thus, if GFS sharpness is monotonically decreasing during GD, longer periods of initial high step size lead to smaller sharpness at the final solution.

Finally, beyond the connection between GFS sharpness and standard sharpness, the GFS sharpness can also help address one of the main puzzles of the EoS regimes, namely, why does the loss converge (non-monotonically) to zero. For scalar neural networks, we show that once the projected sharpness decreases below the stability threshold, the loss decreases to zero at an exponential rate (see Section \ref{sec: GD Follows the GPGD Bifurcation Diagram}).

\begin{figure}[t]
	\centering
	\notarxiv{
	\captionsetup[subfloat]{farskip=2pt, captionskip=1pt}
	\subfloat{\includegraphics[width=0.45\textwidth]{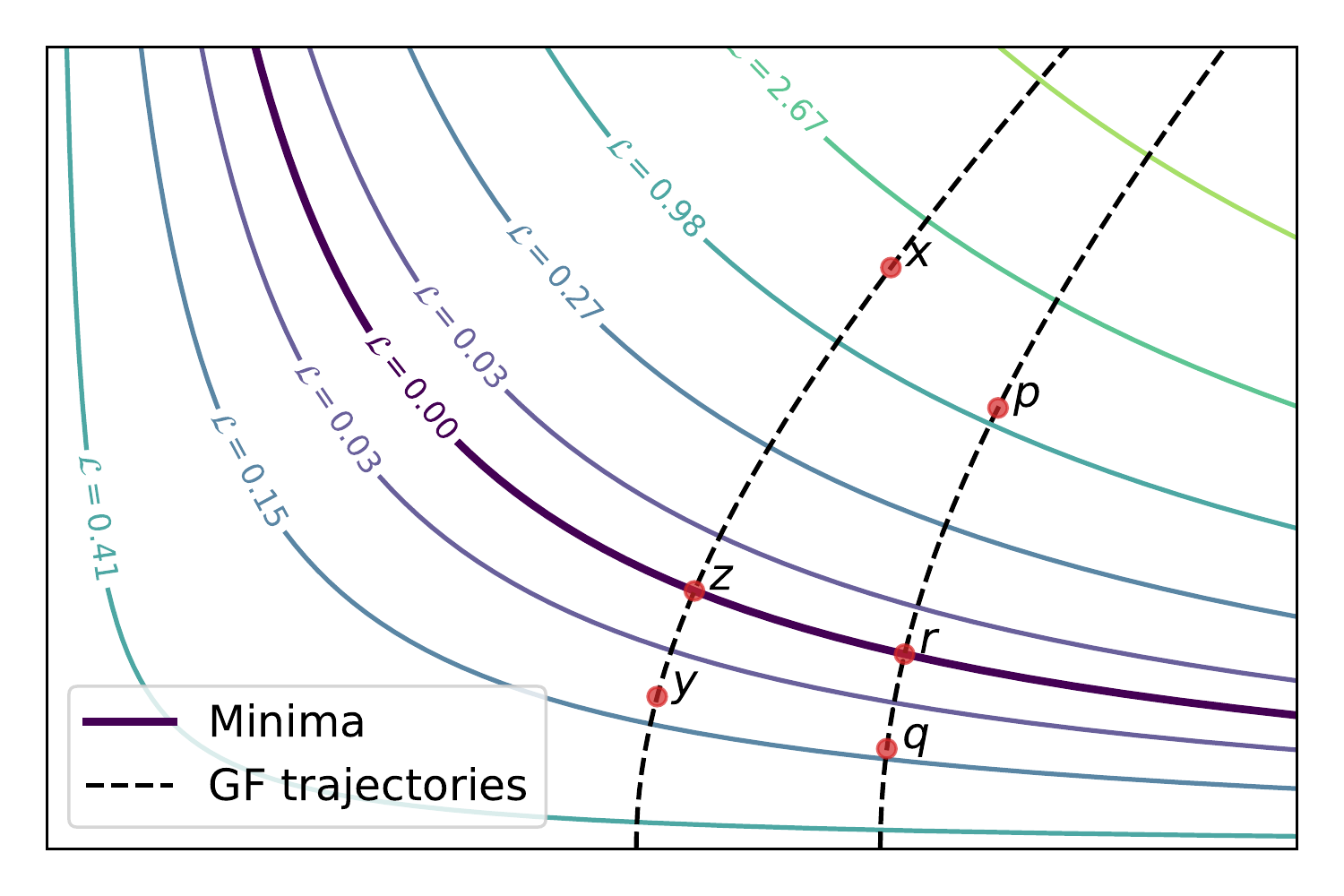}}
	}
	\arxiv{
		{\includegraphics[width=0.6\textwidth]{Figures/depth_2_scalar_gf.pdf}}
	}
	\caption{An illustration of the minima, loss level sets, and trajectories of gradient flow in an example loss landscape.
	Here $\Pgf{\vect{x}}=\Pgf{\vect{y}}=\Pgf{\vect{z}}=\vect{z}$ and $\Pgf{\vect{p}}=\Pgf{\vect{q}}=\Pgf{\vect{r}}=\vect{r}$.
	Thus, $\ps{\vect{x}}=\ps{\vect{y}}=\ps{\vect{z}}=\sh{\vect{z}}$ and $\ps{\vect{p}}=\ps{\vect{q}}=\ps{\vect{r}}=\sh{\vect{r}}$.
}

	\label{Fig: gf trajectory}
\end{figure}

Our \textbf{main contributions} are:
\begin{enumerate}[leftmargin=*]
    \item For scalar neural networks (i.e., linear networks of unit width and general depth) trained with GD on the quadratic loss, we prove that GD monotonically decreases the GFS sharpness (Theorem \ref{thm:porjected sharpness decrease}), for a large set of initializations and step sizes (see Figure \ref{Fig: illustration of when the assumption is satisfied}). 
    
    \item Still in the context of scalar networks, we leverage the monotonicity of GFS sharpness as well as a novel quasistatic analysis, and show that if the loss is sufficiently small when the GFS sharpness crosses the stability threshold $\frac{2}{\eta}$, then the final GFS sharpness (and standard sharpness) will be close to $\frac{2}{\eta}$, establishing the EoS phenomenon (\Cref{Thm: convergence theorem}). This result improves on  \citet{zhu2022understanding} as it holds for a larger class of instances, as we further discuss in \Cref{sec:related}.

    \item Finally, we demonstrate empirically that the monotone decrease of the theoretically-derived GFS sharpness extends beyond scalar networks. Specifically, we demonstrate that the monotonic behaviour and convergence to the stability threshold also happens in the squared regression model (Section \ref{Sec: Illustration of results on squared regression model}) and modern architectures, including fully connected networks with different activation functions, VGG11 with batch-norm, and Resnet20 (Section \ref{Sec: Illustration of results on realistic neural networks}).
\end{enumerate} 
\section{Related Work}\label{sec:related}

The last year saw an intense theoretical study of GD dynamics in the EoS regime. Below, we briefly survey these works, highlighting the aspects that relate to ours.

\paragraph{Analysis under general assumptions.}
Several works provide general---albeit sometimes difficult to check---conditions for EoS convergence and related phenomena.  \citet{ahn2022understanding} relate the non-divergence of GD to the presence of a forward invariant set: we explicitly construct such set for scalar neural networks. 
\citet{arora2022understanding,lyu2022understanding} relate certain modifications of GD, e.g., normalized GD or GD with weight decay on models with scale invariance, to gradient flow that minimizes the GFS sharpness.
In these works, the relation is approximate and valid for sufficiently small step size $\eta$. In contrast, we show \textit{exact} decrease of the GFS sharpness for fairly finite $\eta$.
\citet{ma2022multiscale} relate the non-divergence of unstable GD to sub-quadratic growth of the loss. However, it is not clear whether this is true for neural networks losses; for example, linear neural network with the square loss (including the scalar networks we analyze) are positive polynomials of degree above 2 and hence super-quadratic.
\citet{damian2022self} identify a \emph{self-stabilization} mechanism under which GD converges close to the EoS (similar to the four-stage mechanism identified by~\citet{wang2022analyzing}), under several assumptions.
For example, their analysis explicitly assumes progressive sharpening and implicitly assumes a certain ``stable set'' $\mathcal{M}$ to be well-behaved.
While progressive sharpening is easy to test, the existence of a nontrivial $\mathcal{M}$ is less straightforward to verify. In Appendix
\ref{sec: M in scalar networks}, we explain why the set $\mathcal{M}$ is badly-behaved for scalar networks, meaning that the results of \citet{damian2022self}  cannot explain the EoS phenomenon for this case.

\paragraph{Analysis of specific objectives.}
A second line of works seeks stronger characterizations by examining specific classes of objectives.
\citet{chen2022gradient} characterize the periodic behavior of GD with sufficiently large step size in a number of models, including a two layer scalar network.
\citet{agarwala2022second} consider squared quadratic functions, and prove that there exist 2-dimensional instances of their model where EoS convergence occur.
In order to gain insight into the emergence of threshold neurons, \citet{ahn2022learning} study 2-dimensional objectives of the form $\ell(xy)$ for positive and symmetric loss function $\ell$ satisfying assumptions that exclude the square loss we consider.
They provide bounds on the final sharpness of the GD iterates that approach the EoS as $\eta$ decreases.
None of these three results have direct implications for the scalar networks we study: \citet{chen2022gradient} are concerned with non-convergent dynamics while \citet{agarwala2022second,ahn2022learning} consider different models.

\citet{wang2022large} theoretically studies the balancing effect of large step sizes in a matrix factorization problem of depth 2.
They show that the ``extent of balancing" decreases, but not necessarily monotonically.
Moreover, the model they analyze does not to exhibit the dynamics of the EoS regime (where the sharpness stays slightly above $2/\eta$ for a large number of iterations).
Instead, the GFS sharpness and sharpness very quickly decrease below $2/\eta$.

The work most closely related to ours is by~\citet{zhu2022understanding}, who study particular 2d slices of 4-layer scalar networks.
In our terminology, they re-parameterize the domain into the GFS sharpness and the weight product, and then derive a closed-form approximation for GD's two-step trajectory, that becomes tight as the step size decreases.
Using this approximation they show that, at very small step sizes\footnote{The largest step size for which their results apply is $<10^{-12}$.}, GD converges close to the EoS. Furthermore, they interpret the bifuricating and chaotic behavior of GD using a quasistatic approximation of the re-parameterized dynamics.

Compared to~\citet{zhu2022understanding}, the key novelty of our analysis is that we identify a simple property (GFS sharpness decrease) that holds for \emph{all} scalar networks and a \emph{large range of step sizes}. This property allows us to establish near-EoS convergence without requiring the step size to be very small.
Moreover, by considering a more precise quasistatic approximation of GD we obtain a much tighter reconstruction of its bifurcation diagram (see \Cref{Sec: GD Follows the GPGD Bifurcation Diagram}) that is furthermore valid for all scalar neural networks.

\section{Analysis of Scalar Linear Networks}

\subsection{Preliminaries}
\paragraph{Notation.}%
 We use boldface letters for vectors and matrices. For a vector $\vect{x}\in\R^D$, we denote by $x_i$ its $i$-th coordinate, $\pi\rb{\vect{x}}\triangleq\prod_{i=1}^{D}x_i$ is the product of all the vector coordinates, and $x_{[i]}$ denotes the vector's $i$-th largest element, i.e., $x_{[1]}\ge x_{[2]}\ge\dots\ge x_{[D]}$. In addition, we denote by $\vect{x}^2$, $\vect{x}^{-1}$, and $\abs{\vect{x}}$ the element-wise square, inverse, and absolute value, respectively. 
 We use $\gd{\vect{v}}\triangleq \vect{v} -\eta \nabla\loss\rb{\vect{v}}$ to 
the GD update of parameter vector $\vect{v}$ using step size $\eta$, and we let $\{\vect{w}^{\rb{i}}\}_{i\ge 0}$ denote the sequence of GD iterates, i.e., $\vect{w}^{\rb{t+1}} \triangleq  \gd{\vect{w}^{\rb{t}}}$ for every $t\ge 0$. 
We let $\lambda_{\max}(\vect{v})$ denote the sharpness at $\vect{v}$, i.e., the maximal eigenvalue of $\nabla^2 \loss(\vect{v})$. Finally, we use the standard notation $[N]\triangleq\{1,\dots,N\}$ and $\R_{+}\triangleq\{x\in\R| x>0\}$, and take  $\norm{\cdot}$ to be the Euclidean norm throughout.

\begin{figure*}[t]
	\centering
	\captionsetup[subfloat]{farskip=2pt, captionskip=1pt}
    \subfloat[]{\includegraphics[width=0.33\textwidth]{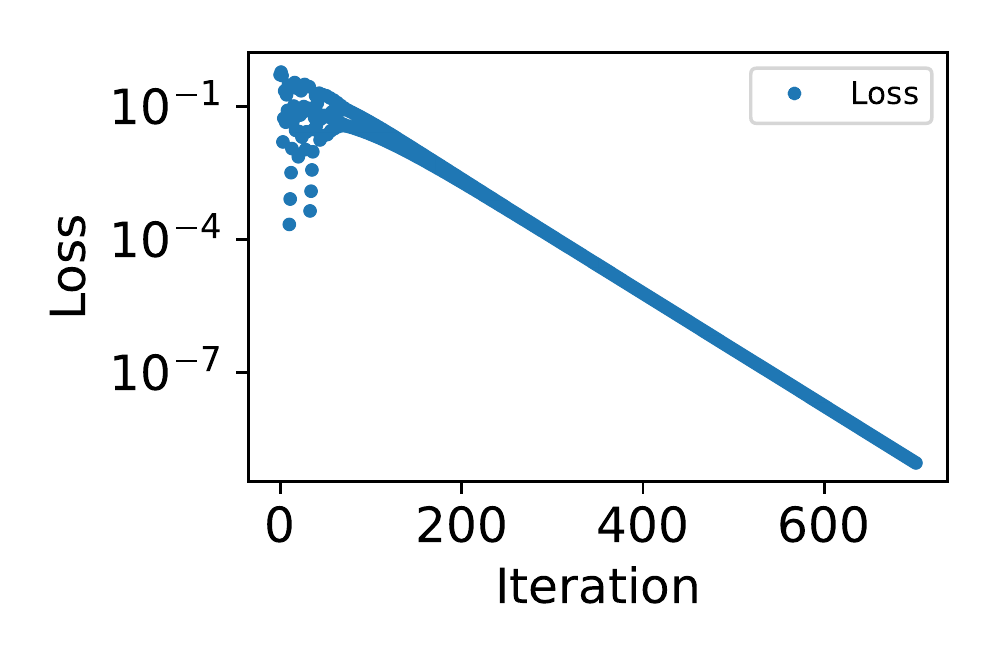}\label{Fig: Scalar network EoS, loss}}
	\hfill
	\subfloat[]{\includegraphics[width=0.33\textwidth]{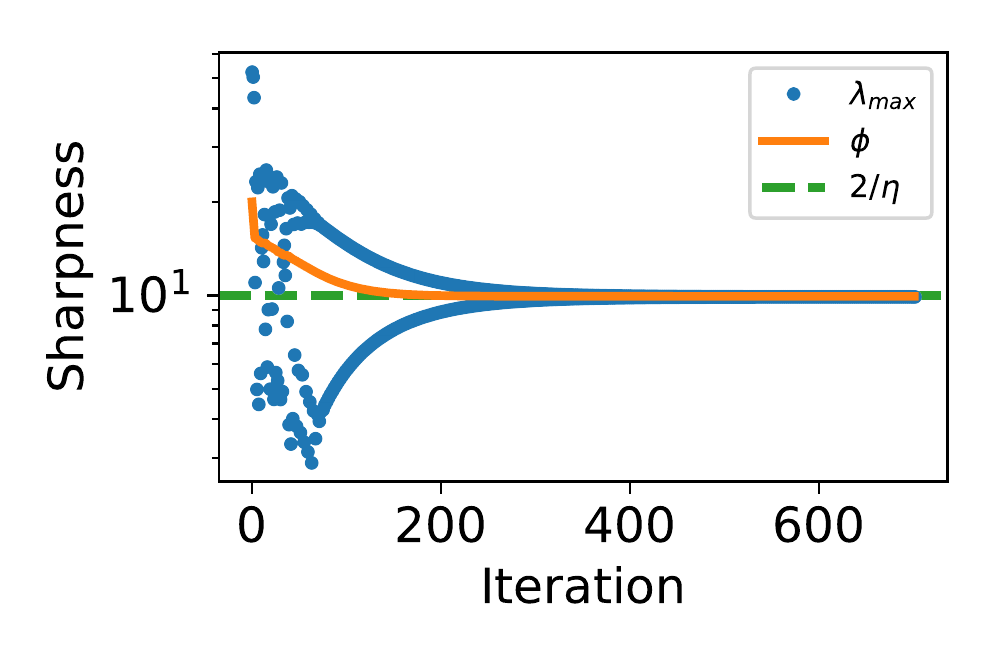}\label{Fig: Scalar network EoS sharpness}}
    \hfill
	\subfloat[]{\includegraphics[width=0.33\textwidth]{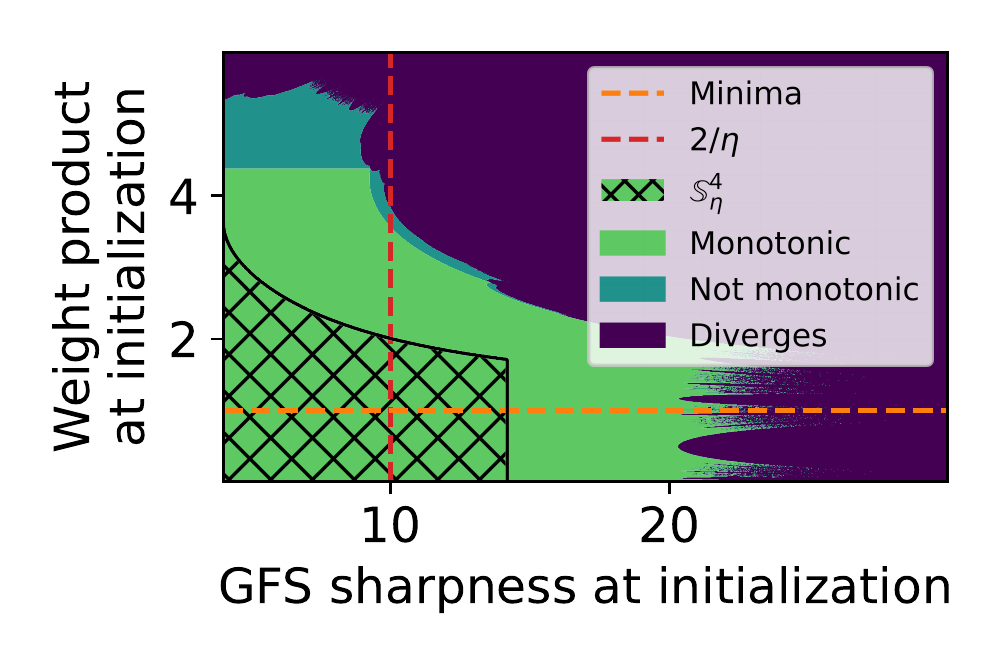}\label{Fig: illustration of when the assumption is satisfied}}
 	\caption{
  For a depth 4 scalar network, we apply GD with  $\eta=0.2$ for $10^4$ steps. 
  In panel \protect\subref{Fig: Scalar network EoS, loss} the loss exhibits oscillatory behavior while converging to zero.
  In contrast, in panel \protect\subref {Fig: Scalar network EoS sharpness} the GFS sharpness is decreasing monotonically until converging to slightly below $\frac{2}{\eta}$; see a zoom-in figure in \cref{Fig: scalar network converges slightly below two over the step size}. 
  In panel \protect\subref{Fig: illustration of when the assumption is satisfied}, 
  we run GD for different initializations and plot the regions in which the GFS sharpness converges monotonically, non-monotonically, and diverges. We also display the set $\SG$ (see \Cref{asmp:proj_sharp_dec}) from which we prove that GFS sharpness decreases monotonically. We observe that $\SG$ covers a significant portion of the region where GFS sharpness decreases monotonically.
  \label{Fig: scalar network example}
}
\end{figure*}

For the theoretical analysis, we consider a scalar linear network with the quadratic loss, i.e., for depth $D\in \N$ and weights $\vect{w}\in \R^D$
the loss function is
\begin{align}
	\label{eq:loss function}
	\loss \left( \vect{w} \right) \triangleq \frac{1}{2}\left( \pi\rb{ \vect{ w } } - 1 \right)^2 \,.
\end{align}
This model exhibits EoS behavior, as demonstrated in Figures \ref{Fig: Scalar network EoS, loss} and \ref{Fig: Scalar network EoS sharpness}, and is perhaps the simplest such model.

Our goal in this section is to prove that (for scalar networks) the GFS sharpness (Definition \ref{def: GFS sharpness}) of GD iterates decreases monotonically to a value nor far below the stability threshold $\frac{2}{\eta}$. 
However, we cannot expect this to hold for all choices of initial weights, since GD diverges for some combinations of step size and initialization.
Therefore,  to ensure stability for a given step size, we need to make an assumption on the initialization.

To this end, for any weight $\vect{w}$  we define its GF equivalence class in the interval $I\subseteq \R$ as
\begin{equation}
	\label{Eq: E_I}
	\begin{aligned}
		E_{I}(\vect{w}) \triangleq \left\{\vect{w'} \mid \vect{w'} \gequiv \vect{w} ~\mbox{and}~ \pi\rb{ \vect{ w' } } \in I \right\}
		\,,
	\end{aligned}
\end{equation}
where $\vect{w'} \gequiv \vect{w}$ if and only if $\Pgf{\vect{w'}} = \Pgf{\vect{w}}$, e.g., if both vectors lie on the same GF trajectory. Also, for depth $D\ge2$ and step size $\eta>0$ we define

\begin{definition}[Positive invariant set]
	\label{asmp:proj_sharp_dec}
	A weight $\vect{ w }\in \R^D$ is in the positive invariant set
	$\SG$ if and only if there exists $B>1$ such that
	\begin{enumerate}[topsep=0pt, itemsep=0.2em]
		\item The coordinate product $\pi(\vect{ w }) \in \ivr{0}{B}$.
		\item For all $\vect{w}'\in E_{\ivr{0}{B}}(\vect{w})$ we have $\pi\rb{\gd{\vect{w}'}}\in \ivr{0}{B}$ (with $E_{\ivr{0}{B}}$ defined in \cref{Eq: E_I}).
		\item The GFS sharpness $\ps{\vect{w}} \le \frac{2\sqrt{2}}{\eta}$.
	\end{enumerate}
\end{definition}
\noindent
Roughly, a point $\vect{w}$ is in $\SG$ if applying a gradient step on weights in the GF trajectory from $\vect{w}$ does not change its coordinate product $\pi(\vect{w})$ too much (conditions 1 and 2) and that the GFS sharpness is not large than $\sqrt{2}$ times the stability threshold (condition 3).

In \Cref{thm:porjected sharpness decrease} below, we assume that the GD initialization satisfies $\vect{w}^{(0)} \in \SG$.
We note that $\SG$ is non-empty whenever there exists a minimizer with sharpness below $\frac{2}{\eta}$ and empirically appears to be fairly large, as indicated by \Cref{Fig: illustration of when the assumption is satisfied}. 
We provide additional discussion of the definition of $\SG$ and the parameter $B$ associated with  it in Appendix \ref{Sec: Discussion on assumption}.

To gain further intuition about the set $\SG$ we define the \emph{GFS-preserving GD} (GPGD) update step from $\vect{w}$ to $\vect{w}'$ as:
\begin{equation}\label{Eq: qs GD}
	\qs{\vect{w}}\triangleq \vect{w}'\gequiv \vect{w}\text{ such that } \pi\rb{\vect{w}'} = \pi\rb{\gd{\vect{w}}}.
\end{equation}

That is, in GPGD the next iterate is chosen so that it lies on the GF trajectory from $\vect{w}$ and its weight product is equal to the weight product of a single gradient step applied to $\vect{w}$. Note that conditions 1 and 2 in Definition \ref{asmp:proj_sharp_dec} can be interpreted as requiring that GPGD initialized at points in $\SG$ does not diverge or change the sign of the product of the weights.

The definition of GPGD is based on the premise that the GFS changes more slowly than the product of the weights, and therefore locally we can approximate GD by keeping the GF projection constant.
We discuss GPGD in more detail in Section \ref{Sec: Proof Outline of convergence Theorem}, where it plays a key role in the proof of Theorem \ref{Thm: convergence theorem}.

\subsection{Main Results}
We can now state our main results.
\begin{theorem}
	\label{thm:porjected sharpness decrease}
	Consider GD with step size $\eta>0$ and initialization $\vect{w}^{(0)} \in \SG$. Then,
	\begin{enumerate}[topsep=0pt, itemsep=0.2em]
		\item For all $t\ge0$, the GD iterates $\vect{w}^{(t)}$ satisfy $\vect{w}^{(t)} \in \SG$.
		\item The GFS sharpness is monotonic non-inscreasing, i.e., $\ps{\vect{w}^{(t+1)}}\le\ps{\vect{w}^{(t)}}$ for all $t\ge 0$.
	\end{enumerate} 
\end{theorem}

Theorem \ref{thm:porjected sharpness decrease} shows that the set $\SG$ is indeed positive invariant (since GD never leaves it) and moreover that GFS sharpness decreases monotonically for GD iterates in this set.

Figure \ref{Fig: scalar network example} demonstrates Theorem \ref{thm:porjected sharpness decrease}. Specifically, in Figure \ref{Fig: Scalar network EoS sharpness} we examine the sharpness and the GFS sharpness when training a scalar neural network with depth $4$. We observe that the GFS sharpness decreases monotonically until reaching $\frac{2}{\eta}$. 
In addition, Figure \ref{Fig: illustration of when the assumption is satisfied} illustrates that $\vect{w}^{(0)} \in \SG$ is indeed sufficient for the GFS sharpness to decrease monotonically and that $\SG$ covers a large portion of the region in GFS sharpness is monotonic.
Furthermore, GFS sharpness non-monotonicity tends to occur only in the first few iterations, after which GD enters $\SG$. 
In \Cref{Sec: GD Follows the GPGD Bifurcation Diagram} (\Cref{Fig:GPGD periods in S}) we argue that $\SG$ may even contain regions where GD is chaotic.

While Theorem \ref{thm:porjected sharpness decrease} guarantees that GFS sharpness decreases monotonically, we would also like to understand its value at convergence, which equals to the sharpness of the point GD converges to. The next theorem states that once the GFS sharpness 
reaches below $\frac{2}{\eta}$ (and provided the loss has also decreased sufficiently),\footnote{
The GFS sharpness is guaranteed to go below $\frac{2}{\eta}$ for any convergent GD trajectory, since sharpness and GFS sharpness become 
 identical around GD's points of convergence, and GD cannot converge to points with sharpness larger than $\frac{2}{\eta}$  \citep[see, e.g.,][]{ahn2022understanding}.
} then it does not decrease much more and moreover the loss decreases to zero at an exponential rate.
\newcommand{\lossBound}{\bar{L}}
\begin{theorem}
	\label{Thm: convergence theorem}
	If for some $t\ge 0$ and $\delta \in (0,0.4]$ we have that $\wt \in \SG$, $\ps{\wt} = \frac{2-\delta}{\eta}$ and $\loss(\wt) \le \delta^2/200$ then
    \begin{enumerate}[topsep=0pt, itemsep=0.2em]
        \item $\ps{\wt}
		\ge \frac{2}{\eta} \rb{1 - \delta}$ for all $t\ge0$.
        \item $\loss\rb{\wt[k+t]} \le 2(1-\delta)^{2k} \loss \rb{\wt[t]}$ for all $k\ge 1$.
        \item The sequence $\{\wt\}$ converges. %
    \end{enumerate}
\end{theorem}

In Figure \ref{Fig: illustration of sharpness at convergence} we plot the sharpness at GD convergence as a function of initialization, parameterized by initial GFS sharpness and weights product. (Note again that sharpness and GFS sharpness are equal at convergence.) We observe that the sharpness converges to ${2}/{\eta}$ when the GFS sharpness at initialization was larger than ${2}/{\eta}$ and the product of the weights was relatively close to $1$, i.e., the loss was not too large. This demonstrates Theorem \ref{Thm: convergence theorem}. Additionally, we can see that the condition of the loss being sufficiently small is not only sufficient but also necessary. That is, for initialization with GFS sharpness close to ${2}/{\eta}$ and high loss (i.e., weight product far from 1), the GFS sharpness can converge to values considerably below ${2}/{\eta}$. We also show a specific example in Figure \ref{Figure: GD converge below two over the step size}.

\begin{figure}[t]
	\centering
	\captionsetup[subfloat]{farskip=2pt, captionskip=1pt}
	\notarxiv{\hfill}
	\subfloat
	{\includegraphics[width=0.37\textwidth,valign=t]{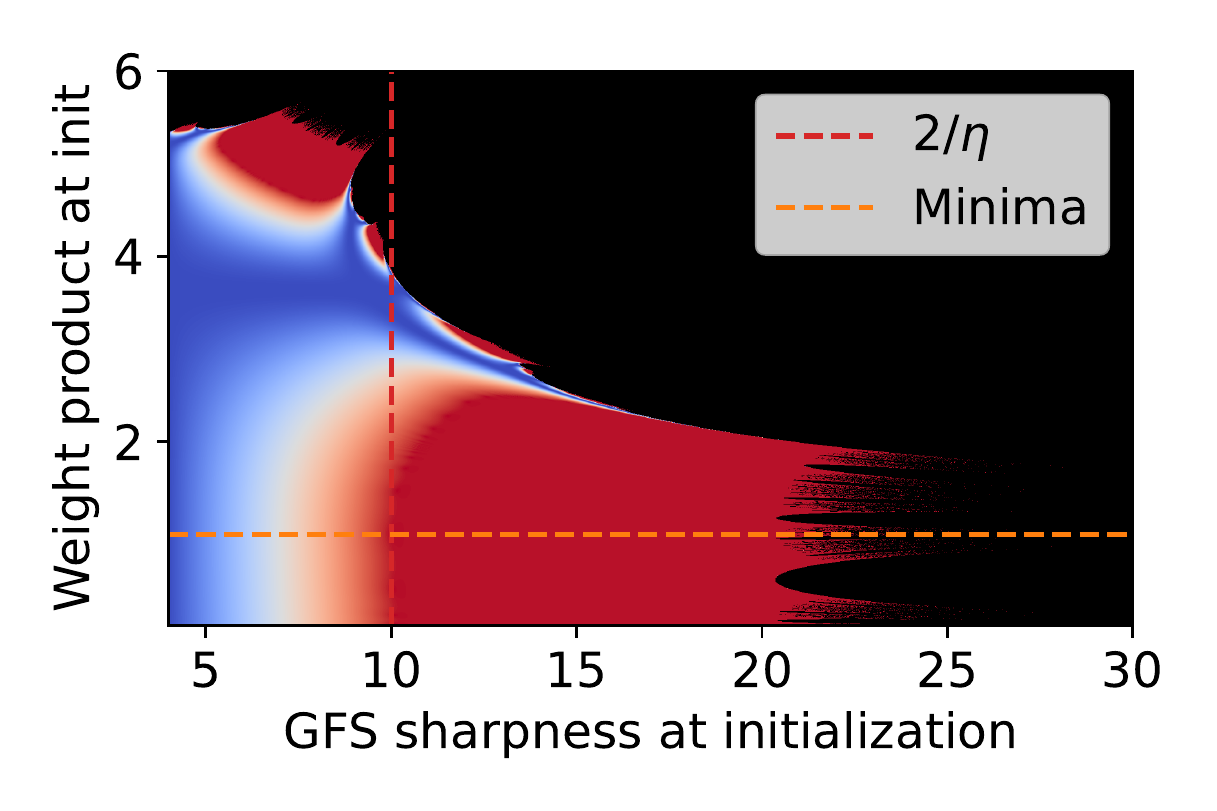}}
	\notarxiv{\hspace{0.00cm}}
	\subfloat
	{\includegraphics[width=0.075\textwidth,valign=t]{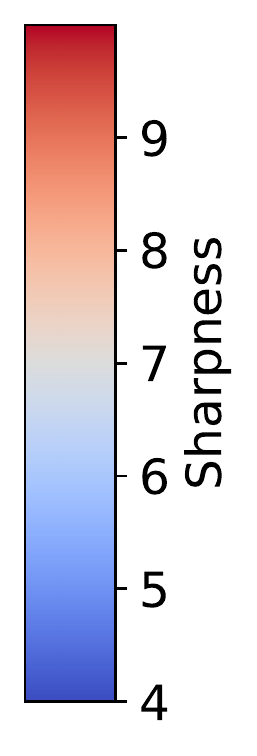}}
	\notarxiv{\hfill}
	
	\caption{Illustration of the sharpness at GD convergence, in the same settings as \cref{Fig: scalar network example}.}
	
	\label{Fig: illustration of sharpness at convergence}
\end{figure}

\subsection{Toward the Proof of Theorem \ref{thm:porjected sharpness decrease}}
To prove the monotonic decrease of GFS sharpness, we identify a quasi-order on scalar linear networks that is monotonic under GD. We call it the ``balance'' quasi-order and define it below. (Recall that $w_{[i]}$ to denotes the $i$'th largest element in the sequence $w_1,\ldots,w_D$).

\begin{definition}[Balance quasi-order]\label{Def: Balance quasi-order}
	For two scalar networks $\vect{w},\vect{v}\in \R^D$ we say that $\vect{w}$ is less unbalanced than $\vect{v}$, written as $\vect{w} \ble \vect{v}$, if
	\notarxiv{
	\[ \forall i \in [D-1] : ~ w^2_{[i]} - w^2_{[i+1]} \le v^2_{[i]} - v^2_{[i+1]} \,. \]
	}
	\arxiv{
	\[ w^2_{[i]} - w^2_{[i+1]} \le v^2_{[i]} - v^2_{[i+1]} ~\mbox{for all}~i\in[D-1] .\]
	}
\end{definition}

\begin{remark}[Balance invariance under GF]
	\label{rem: balance invariance under GF}
    The ordered balance $b_i\rb{\vect{w}}\triangleq w^2_{[i]} - w^2_{[i+1]}$ is invariant under GF, i.e., for any $\vect{w}'\gequiv \vect{w}$ we have $b_i\rb{\vect{w}'}=b_i\rb{\vect{w}}$ is satisfied  for all $i\in [D-1]$ \citep{arora2018optimization, du2018algorithmic}.
\end{remark}

To leverage this quasi-order,  we require the related concepts of (log) majorization and a Schur-convex function \cite{marshall2011inequalities}. 
\begin{definition}[Log majorization]
	For vectors $\vect{u},\vect{v}\in \R_{+}^D$ we say that $\vect{v}$ log majorizes $\vect{u}$, written as $\vect{u} \logmaj \vect{v}$, if
	\begin{align*}
		& \prod_{i=1}^{D} u_{\sqb{i}} = \prod_{i=1}^{D} v_{\sqb{i}} ~ \text{and}\\
		& \prod_{i=1}^{k} u_{\sqb{i}} \le \prod_{i=1}^{k} v_{\sqb{i}} ~ \text{, for every } k\in[D] \,.
	\end{align*}
\end{definition}
\noindent
The balance quasi-order and the log majorization quasi-order are related; 
we formalize this relation in the following lemma (proof in Appendix \ref{sec: Proof of balance to majorization lemma}).
\begin{lemma}
	\label{lem:balance to majorization}
	For $\vect{u},\vect{v}\in \R^D$, if $\vect{u} \ble \vect{v}$ and $\prod_{i=1}^{D} u_i = \prod_{i=1}^{D} v_i$ then $\abs{\vect{u} }\logmaj \abs{\vect{v}}$.
\end{lemma}
\noindent
Schur-convex functions are monotonic with respect to majorization, and we analogously define log-Schur convexity.\footnote{This definition of log-Schur-convex function as the composition of a Schur-convex function and an elementwise log (see \cref{lem:Schur-convex to log-Schur-convex}).}

\begin{definition}[Log Schur-convexity]
	\label{def:log Schur-convexity}
	A function $f:\A\mapsto\R$ is log-Schur-convex on $\A\subseteq\R_{+}^n$ if for every $\vect{u},\vect{v}\in \A$ such that $\vect{u} \logmaj \vect{v}$ we have $f(\vect{u})\le f(\vect{v})$.
\end{definition}
The proof of \Cref{thm:porjected sharpness decrease} relies on log-Schur-convexity of the following functions (proof in Appendix \ref{Sec: proof of sc function lemma}). 
\begin{lemma}
	\label{lem:sc functions}
	The following functions from $\vect{x}\in\R_{+}^{D}$ to $\R$ are log-Schur-convex:
	\begin{enumerate}[topsep=0pt, itemsep=0.2em]
		\item\label{lem idx:sc functions 1} The function $s_1\rb{\vect{x}}\triangleq \pi^2 \rb{\vect{x}} \norm{\vect{x}^{-1}}^{2}$ on $\R_{+}^{D}$.
		\item\label{lem idx:sc functions 2} The function $-\gdb{\vect{x}}_{\sqb{D}}$ on $\left\{ \vect{x}\in \R_{+}^{D} | \pi\rb{\vect{x}} \ge 1 \right\}$. 
		\item\label{lem idx:sc functions 3} The function $\pi\rb{\gd{ \vect{x} }}$ on $\left\{ \vect{x}\in \R_{+}^{D} | \pi\rb{\vect{x}} \le 1 \right\}$.
	\end{enumerate}
\end{lemma}

For a full description of all the definitions, lemmas, and theorems related to majorization and Schur-convexity used in this paper, see \cref{sec:majorization_and_schur-convexity}.

\subsection{Proof of \Cref{thm:porjected sharpness decrease}}
We first calculate the Hessian of the loss \eqref{eq:loss function} at an optimum 
   \begin{align}
		\label{eq:Hessian}
		\loss \rb{\vect{w}} = 0 \implies 
		\nabla^2 \loss \rb{\vect{w}}
		=& \pi^2 \rb{\vect{w}} \vect{w}^{-1} \rb{\vect{w}^{-1}}^{T}
		\,.
\end{align}
Consequently, if $\vect{w}^\star$ is an optimum, its sharpness is
\begin{equation}
	\begin{aligned}
		\label{eq:sharpness at opt main text}
		\lambda_{\max}\rb{\vect{w}^\star}
		&= \pi^2 \rb{\vect{w}^\star} \norm{\vect{w^\star}^{-1}}^{2} = s_1(\vect{w}^\star),
	\end{aligned}
 \end{equation}
for the function $s_1$ defined in \Cref{lem:sc functions}.

The proof of Theorem \ref{thm:porjected sharpness decrease} relies on the following key lemma that shows that for weights $\vect{w}\in \SG$ a single step of GD makes the network more balanced (proof in Appendix \ref{Sec: proof of gd balance lemma}). 

\begin{lemma}
	\label{lem:gd balance}
	If $\vect{w}\in \SG$ then $\gd{\vect{w}} \ble \vect{w}$.
\end{lemma}
\noindent
 In addition, weights in $\SG$ do not change their sign under GD with step size  $\eta$, formally expressed as follows (proof in \cref{sec:GD does not change sign}). 
\begin{lemma}
	\label{lem:GD does not change sign}
	For any $\vect{w} \in \SG$ and  $i \in \D$ we have $\sign(\gdb{\vect{w}}_i)=\sign(\vect{w}_i)$.
\end{lemma}
\noindent
Combining these results with Lemmas \ref{lem:balance to majorization} and \ref{lem:sc functions}
we prove \cref{thm:porjected sharpness decrease}.

\begin{proof}[Proof of \Cref{thm:porjected sharpness decrease}]
	We begin by assuming $\wt\in\SG$ and showing that $\ps{\wt[t+1]} \le \ps{\wt[t]}$. To see this, note that $\wt[t+1] \ble \wt$ by \Cref{lem:gd balance}. Since gradient flow preserves the balances, this implies that $\Pgf{\wt[t+1]} \ble \Pgf{\wt[t]}$. Moreover, since $\pi\rb{\Pgf{\wt[t+1]}} = \pi\rb{\Pgf{\wt[t]}}=1$, \cref{lem:balance to majorization}  gives $\Pgf{\wt[t+1]} \logmaj \Pgf{\wt[t]}$. Applying \crefi{lem:sc functions}{lem idx:sc functions 1}, we get that $s_1\rb{\Pgf{\wt[t+1]}} \le s_1\rb{\Pgf{\wt[t]}}$. Recalling \cref{eq:sharpness at opt main text}, we note that $\ps{\vect{v}} = s_1(\Pgf{\vect{v}})$ for all $\vect{v}\in\R^D$, completing the proof that $\wt\in\SG$ implies $\ps{\wt[t+1]} \le \ps{\wt[t]}$.
	
	It remains to show that $\wt\in\SG$ also implies $\wt[t+1]\in\SG$; combined with $\wt[0]\in\SG$ this immediately yields $\wt\in\SG$ for all $t$ and, via the argument above, monotonicity of $\ps{\wt}$. 
	
	Given  $\wt\in\SG$, we verify that $\wt[t+1]\in\SG$ by checking the three conditions in \Cref{asmp:proj_sharp_dec} of $\SG$. The third condition is already verified, as we have shown that $\ps{\wt[t+1]} \le \ps{\wt[t]}$, and $\ps{\wt} \le \frac{2\sqrt{2}}{\eta}$ since $\wt\in\SG$. We proceed to verifying the first and second conditions on $\wt[t+1]$, assuming they hold on $\wt$.
	
	As $\vect{w}^{\rb{t}}\in\SG$, there exist $B>1$ such that $\pi\rb{\vect{w}^{\rb{t}}}\in\ivr{0}{B}$ and for every $\vect{u}\in E_{\ivr{0}{B}}(\vect{w}^{\rb{t}})$ we have $\pi\rb{\gd{\vect{u}}}\in \ivr{0}{B}$. In particular, we may take $\vect{u}=\wt=E_{\{\pi\rb{\wt}\}}\rb{\wt}$ and conclude that $\pi\rb{\vect{w}^{\rb{t+1}}}=\pi\rb{\gd{\vect{w}^{\rb{t}}}}\in\rb{0,B}$. Hence, $\wt[t+1]$ also satisfies the first condition of \Cref{asmp:proj_sharp_dec}, with the same $B$ as $\wt$. 

	To verify the second condition in \Cref{asmp:proj_sharp_dec}, we fix any 
    $\vect{v}\in E_{\ivr{0}{B}}(\vect{w}^{\rb{t+1}})$ and argue that $\pi\rb{\gd{\vect{v}}} \in (0,B)$. 
    Using \cref{lem:gd balance} and the fact that balances are invariant under GF, we get that $\vect{v}\ble \vect{u}$, for any $\vect{u}\in E_{\ivr{0}{B}}(\vect{w}^{\rb{t}})$.
    In particular, we take $\vect{u}'\in E_{\{\pi\rb{\vect{v}}\}}(\vect{w}^{\rb{t}})$ so that $\pi\rb{\vect{v}}=\pi\rb{\vect{u}'}$.
    Then by using \cref{lem:balance to majorization} we get that $\vect{v} \logmaj \vect{u}'$.
    We note that because $\vect{w}^{\rb{t}} \in \SG$, $\Pgf{\vect{u}'} = \Pgf{\vect{w}^{\rb{t}}}$ and $\pi\rb{\vect{u}'} \in \ivr{0}{B}$ then, from \cref{asmp:proj_sharp_dec}, we have that $\vect{u}' \in \SG$.
    Without loss of generality, we assume that $\vect{v},\vect{u'}\in  \R_{+}^{D}$ (see \cref{sec:Equivalence of weights} for justification).
    
    We now consider two cases: 
    \begin{enumerate}[leftmargin=*, topsep=0pt, itemsep=0.2em]
        \item If $\pi\rb{\vect{v}}\ge1$ then, by using \crefi{lem:sc functions}{lem idx:sc functions 2} and the definition of log-Schur-convexity (\cref{def:log Schur-convexity}), we get that $\gdb{\vect{v}}_{\sqb{D}} \ge \gdb{\vect{u}'}_{\sqb{D}}$.
        In addition, as a consequence of \cref{lem:GD does not change sign} and that $\vect{u}'\in\R_+^D$, we obtain that $\gdb{\vect{v}}_{\sqb{D}} \ge \gdb{\vect{u}'}_{\sqb{D}} > 0$.
        Therefore, $\pi\rb{\gd{\vect{u}'}} > 0$ and also $\pi\rb{\gd{\vect{v}}}>0$.
        Moreover, since $\gd{\vect{v}} = \vect{v} - \eta (\pi(\vect{v})-1)\pi(\vect{v})\vect{v}^{-1}$ (see \cref{eq:gradient} in \cref{sec:gradient and Hessian}) and $\pi\rb{\vect{v}}\ge1$, we have $\gd{\vect{v}} \le \vect{v}$ elementwise, and therefore (since $\gd{\vect{v}}\in \R_{+}^D$) we have $\pi\rb{\gd{\vect{v}}} \le \pi\rb{\vect{v}} < B$.
        Therefore $\pi\rb{\gd{\vect{v}}} \in (0,B)$.
        
        \item If $\pi\rb{\vect{v}}\le1$ then, by using \crefi{lem:sc functions}{lem idx:sc functions 3}, we get that $\pi\rb{\gd{ \vect{v} }} \le \pi\rb{\gd{ \vect{u}' }} < B$.
        Moreover,  since $\gd{\vect{v}} = \vect{v} - \eta (\pi(\vect{v})-1)\pi(\vect{v})\vect{v}^{-1}$ and $\pi\rb{\vect{v}}\le 1$, we have $\gd{\vect{v}} \ge \vect{v}$ elementwise, and therefore (since $\vect{v}\in \R_{+}^D$) we have $0 < \pi\rb{\vect{v}}  \le \pi\rb{\gd{\vect{v}}}$.
        Therefore $\pi\rb{\gd{\vect{v}}} \in (0,B)$.
    \end{enumerate}
	\noindent
	We conclude that $\pi\rb{\gd{\vect{v}}} \in (0,B)$ for all $\vect{v}\in E_{\ivr{0}{B}}(\vect{w}^{\rb{t+1}})$, establishing the second condition in \Cref{asmp:proj_sharp_dec} and completing the proof.
\end{proof}

The proof shows that GD monotonically decreases not only the GFS sharpness, but also the balance quasi-order.

\subsection{Proof Outline for \Cref{Thm: convergence theorem}}\label{Sec: Proof Outline of convergence Theorem}
From Theorem \ref{thm:porjected sharpness decrease} we know that the GFS sharpness is decreasing monotonically. Recall that the sharpness and GFS sharpness become identical when GD converges and that GD cannot converge while the sharpness is above $\frac{2}{\eta}$. Thus, unless GD diverges, the GFS sharpness must decrease below $\frac{2}{\eta}$ at some point during the trajectory of GD.

The following lemma lower bounds the change in the GFS sharpness after a single GD step (proof in Appendix \ref{Sec: GFS sharpness decrease lemma proof}).

\begin{lemma}
	\label{lem:GFS sharpness decrease}
	For any $t$, if $\wt \in \SG$, then
	\begin{equation*}
			\ps{\wt[t+1]} \ge 
			\frac{\ps{\wt}}{1 + 8\rb{\frac{\phi(\wt)}{2/\eta}  \max\{1,\pi(\wt)\}}^2 \loss(\wt )}. %
	\end{equation*}
\end{lemma}
\noindent
\Cref{lem:GFS sharpness decrease} implies that if the loss 
vanishes sufficiently fast after the GFS sharpness reaches below $\frac{2}{\eta}$, then the GFS sharpness will remain close to $\frac{2}{\eta}$.

Therefore, in order to prove Theorem \ref{Thm: convergence theorem} our next goal will be to show that the loss vanishes sufficiently fast. To attain this goal, we use the notion of GPGD defined in Eq. \eqref{Eq: qs GD}. The motivation for using GPGD is that a single step of GD does not change the GFS sharpness too much (Lemma \ref{lem:GFS sharpness decrease}) and therefore GD and GPGD dynamics are closely related (more on this in Section \ref{sec: GD Follows the GPGD Bifurcation Diagram}). We denote the GPGD update step from $\vect{w}$ by $\qs{\vect{w}}$.  

In the next lemma, we consider a point $\tilde{\vect{w}}^{\rb{0}}\in \R^D$ with GFS sharpness bounded below the stability threshold, and show GPGD iterates starting from $\tilde{\vect{w}}^{\rb{0}}$  converge to zero at an exponential rate (proof in \ref{Sec: qs GD convergance lemma proof}).

\begin{lemma}
	\label{lem:qs GD convergance}
	For GPGD iterates $\tilde{\vect{w}}^{\rb{t+1}} = \qs{\tilde{\vect{w}}^{\rb{t}}}$. If $\ps{\tilde{\vect{w}}^{\rb{0}}}=\frac{2-\delta}{\eta}$ for $\delta\in(0,0.5]$ and
	\begin{align*}
		1 - 0.1 \frac{\delta}{1 - \delta} \le \pi\rb{\tilde{\vect{w}}^{\rb{0}}} \le 1
	\end{align*}
	then for all $t\ge0$,
	\begin{align*}
		\loss\rb{\tilde{\vect{w}}^{\rb{t}}} \le \rb{1-\delta}^{2t} \loss\rb{\tilde{\vect{w}}^{\rb{0}} }
	\end{align*}
    and the error $\pi\rb{\tilde{\vect{w}}^{\rb{t}}}-1$ changes sign at each iteration.
\end{lemma}

In the next lemma, we show that the GPGD loss can be used to upper bound the loss of GD (proof in Appendix \ref{Sec: GPGD to GD lemma proof}).
\begin{lemma}
	\label{lem:GPGD to GD}
	For any $\vect{w}\in\SG$, if $\pi\rb{\qs{\vect{w}}} \ge 1$ and $\ps{\gd{\vect{w}}} \ge \frac{1}{\eta}$ then
	\begin{align*}
		&\loss\rb{\gd{\gd{\vect{w}}}} \le \loss\rb{\qs{\qs{\vect{w}}}}
		\,,
	\end{align*}
	and $\pi\rb{\gd{\gd{\vect{w}}}} \le 1$.
\end{lemma}
Note that we are interested in comparing the losses after two GD steps (in contrast to a single step) since, from the definition of GPGD, we have that $\loss\rb{\gd{\vect{w}}} = \loss\rb{\qs{\vect{w}}}$.

Overall, combining \cref{lem:qs GD convergance} and \cref{lem:GPGD to GD} we prove that GD loss vanishes exponentially fast and combining this result with \cref{lem:GFS sharpness decrease} we obtain the lower bound on GFS sharpness, given in \cref{Thm: convergence theorem} (see \Cref{app:convergence-theorem-proof} for the full proof).

\begin{figure}[t]
	\centering
	\captionsetup[subfloat]{farskip=2pt, captionskip=1pt}
	\includegraphics[width=0.45\textwidth,valign=t]{{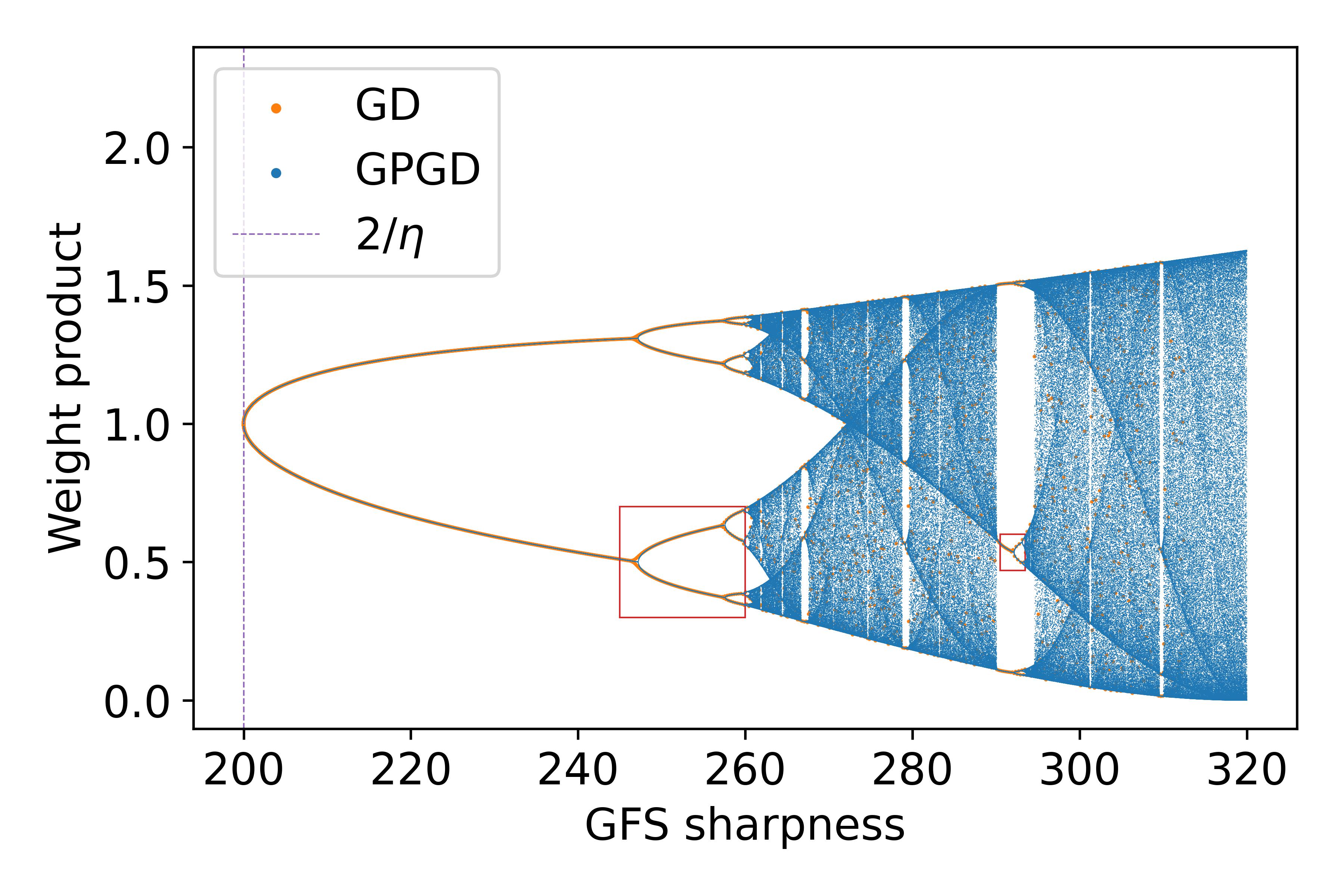}}
	
	\caption{\label{Fig:bifurcation-main} \textbf{GD follows the GPGD bifurcation diagram.} See \Cref{Fig:bifurcation} for additional description and zoomed-in plots.}
\end{figure}

\subsection{GD Follows the GPGD Bifurcation Diagram}\label{sec: GD Follows the GPGD Bifurcation Diagram}

Lemma \ref{lem:qs GD convergance} implies that, when the GFS sharpness is below $\frac{2}{\eta}$,  GPGD converges to the global minimum of the loss. 
In contrast, if the GFS sharpness is above $\frac{2}{\eta}$ then either GPGD behaves chaotically or it converges to a periodic sequence. 

In \ref{Fig:bifurcation-main} we summarize the behavior of GPGD with a \emph{bifurcation diagram}, showing the weight product of its periodic points as a function of GFS sharpness (which is constant for each GPGD trajectory). The figure also shows the values of $(\phi(\wt),\pi(\wt))$ for a GD trajectory---demonstrating that they very nearly coincide with the GPGD periodic points. 

The observation that GD follows the GPGD bifurcation diagram gives us another  perspective on the convergence to the edge of stability and on the non-monotonic convergence of the loss to zero:  
when GD is initialized with GFS sharpness above the stability threshold, it ``enables'' GPGD to converge to loss zero by slowly decreasing the GFS sharpness until reaching the stability threshold $\frac{2}{\eta}$. Since GD closely approximates the GPGD periodic points throughout, it follows that when reaching GFS sharpness close to $\frac{2}{\eta}$, GD must be very close to the period 1 point of GPGD, which is exactly the minimizer of the loss $\loss$  with sharpness $\frac{2}{\eta}$. 

In \Cref{Sec: GD Follows the GPGD Bifurcation Diagram} we provide more details on the GPGD bifurcation diagram, as well as zoomed-in plots and a comparison with the bifurcation diagram obtained by the approximate dynamics of \citet{zhu2022understanding}.

\begin{figure*}[t]
	\centering
	\captionsetup[subfloat]{farskip=2pt, captionskip=1pt}
	\subfloat[]
	{\includegraphics[width=0.31\textwidth,valign=t]{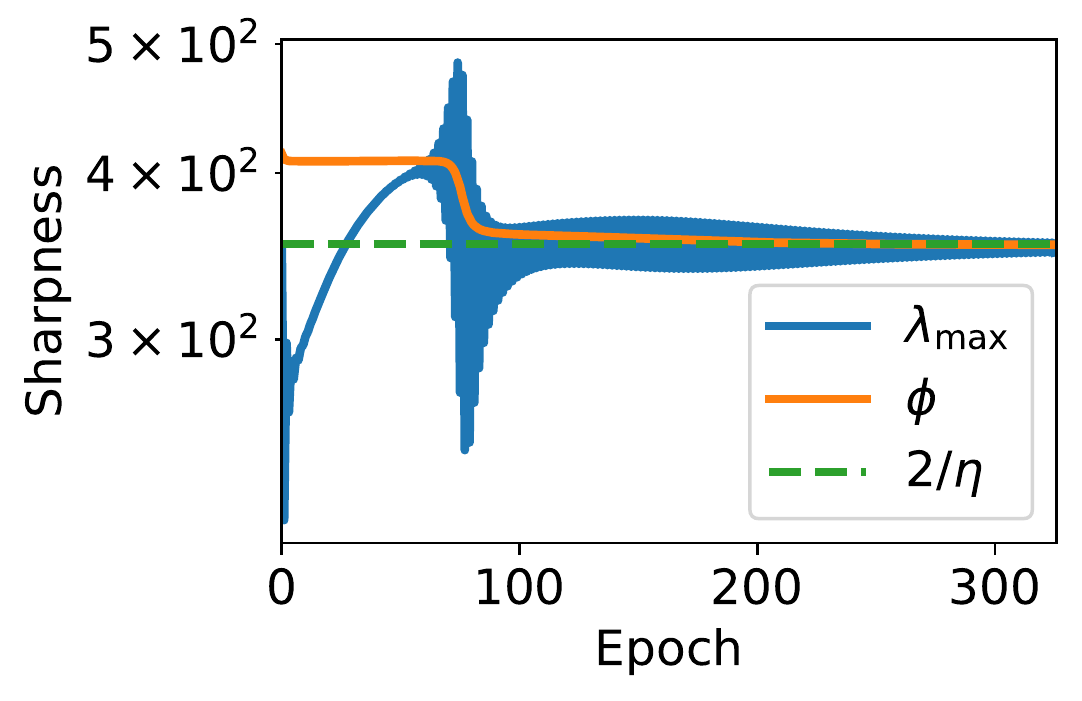} \label{subfig: sharpness and ps for the squared regression}}
	\hfill
	\subfloat[]
	{\includegraphics[width=0.31\textwidth,valign=t]{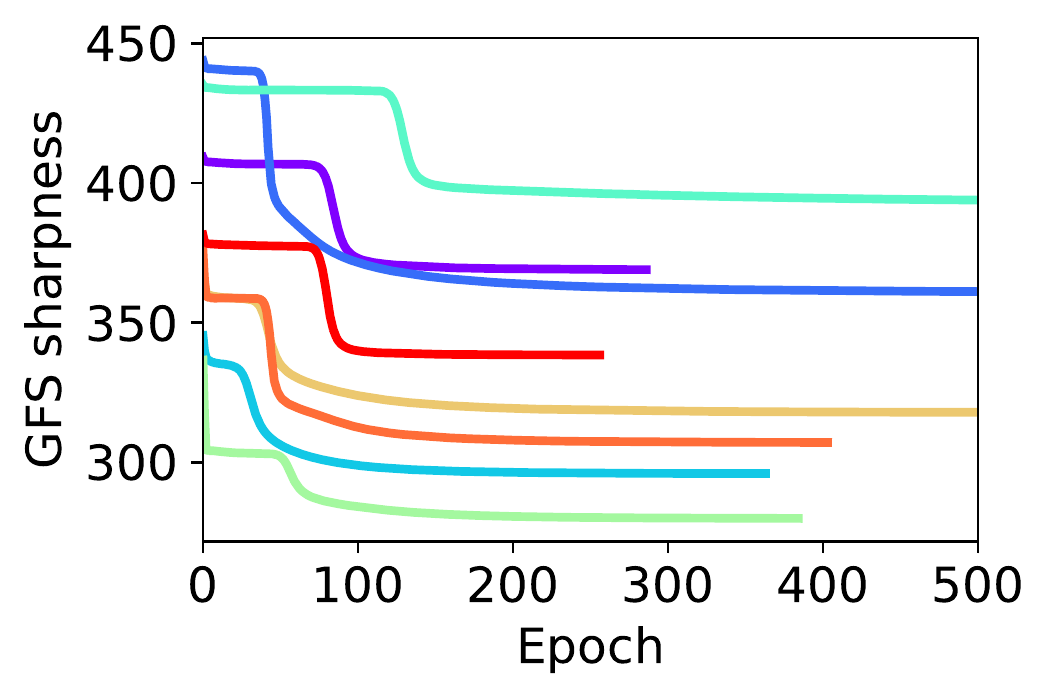}\label{subfig: ps for multiple seed and the squared regression}}
	\hfill
	\subfloat[]
	{\includegraphics[width=0.31\textwidth,valign=t]{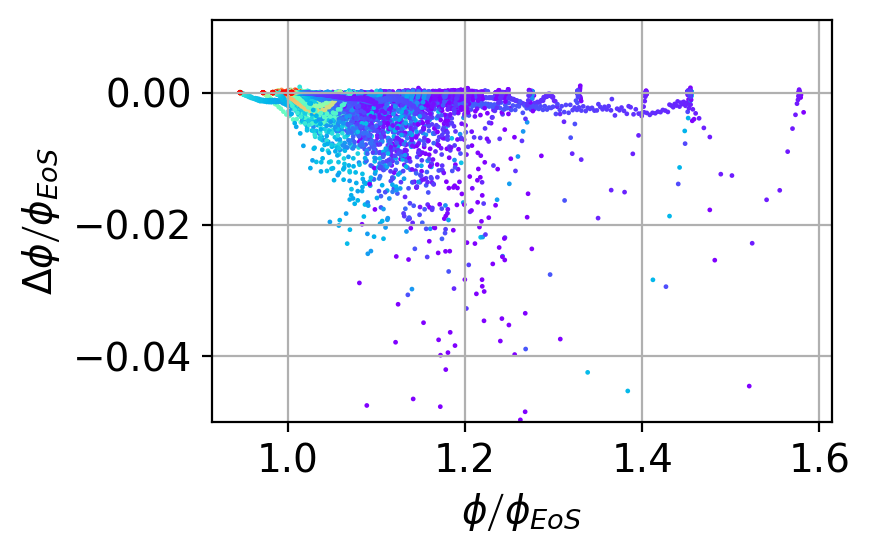}\label{subfig: consistecny test squared regression}}
	\caption{\textbf{The GFS sharpness exhibits consistent monotonic decrease for the squared regression model.} For MSE loss and the squared regression model with random synthetic data, we apply GD until the training loss decreases below $0.01$. We calculate the GFS sharpness at each iteration using Eq. \eqref{Eq: GF projection for squared regression}. We repeat this experiment with $50$ different seeds (i.e., 50 different random datasets and initializations) and two large learning rates per seed: $\eta_1=0.85\cdot\frac{2}{\min_{\btheta}\lambda_{\max}\rb{\btheta}}, \, \eta_2=0.99\cdot\frac{2}{\min_{\btheta}\lambda_{\max}\rb{\btheta}}$. Note that the sharpness of the flattest implementation $\min_{\btheta}\lambda_{\max}\rb{\btheta}$ varies for each random dataset. 
		Full implementation details are given in Appendix \ref{Sec: Implementation details for the squared regression experiments}. In all three figures, we observe that GFS sharpness decreases monotonically for various seeds and step sizes. See detailed discussion in Section \ref{Sec: Illustration of results on squared regression model}.}
	
	\label{Fig: illustration on the squared regression model}
\end{figure*}

\begin{figure*}[t]
	\centering
	\captionsetup[subfloat]{farskip=2pt, captionskip=1pt}
	\subfloat[FC-hardtanh]{\includegraphics[height=3.5cm]{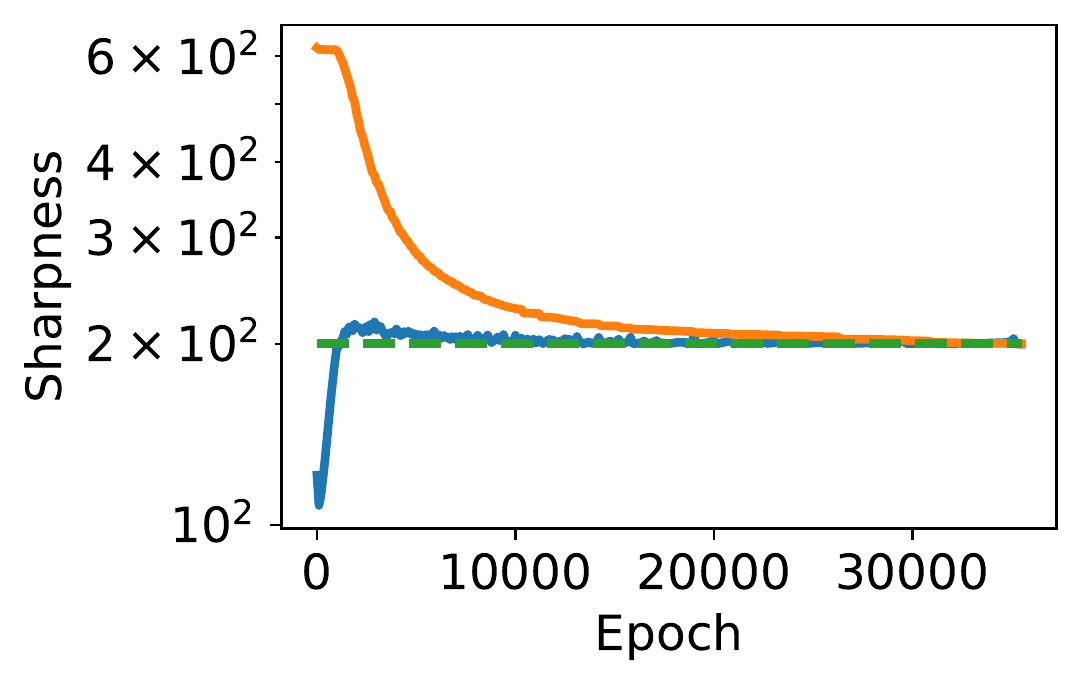}}
	\hfill
	\hspace{-0.2cm}
	\subfloat[VGG11-bn]{\includegraphics[height=3.5cm]{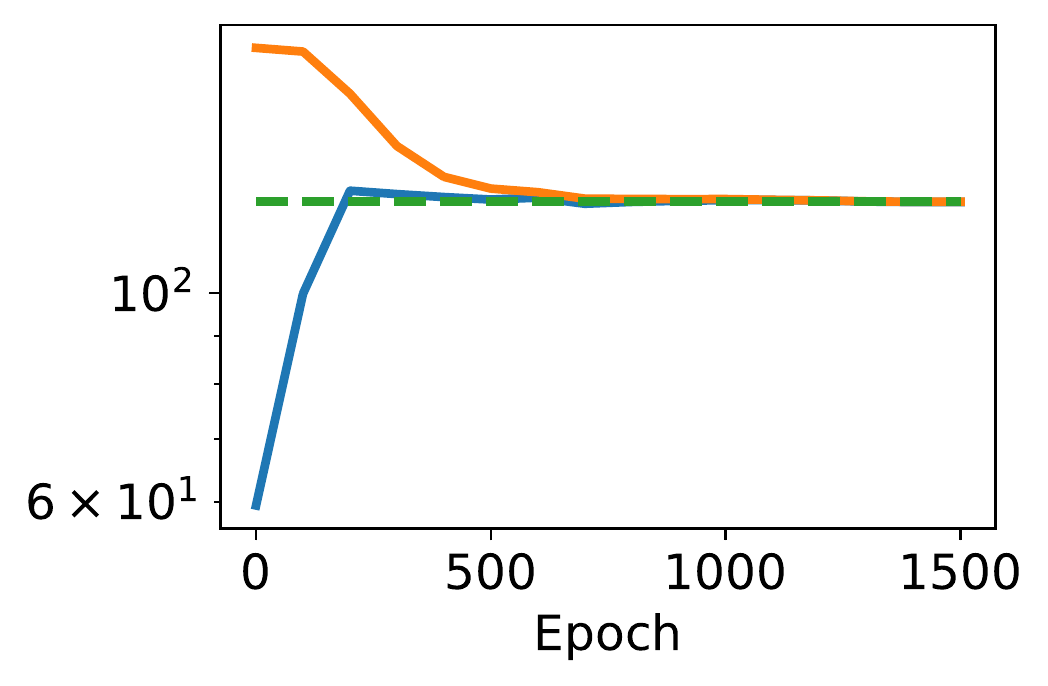}}
	\hfill
	\subfloat[ResNet20]{\hspace{-0.2cm}\includegraphics[height=3.5cm]{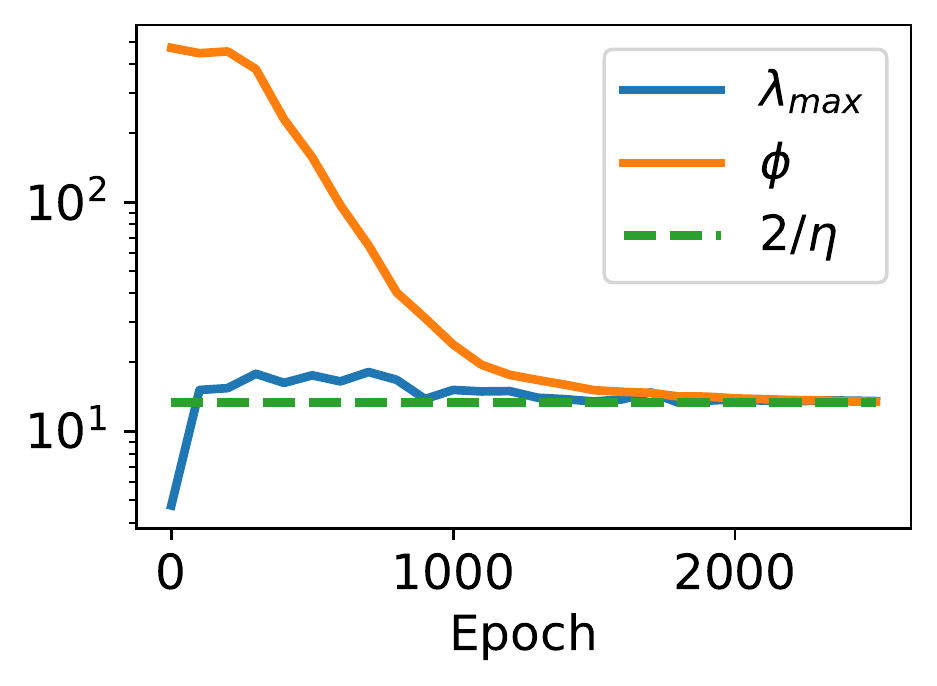}\label{subfig:resnet GFS sharpness}}
	\\
	\subfloat[FC-hardtanh]{\hspace{0.4cm}\includegraphics[height=3.5cm]{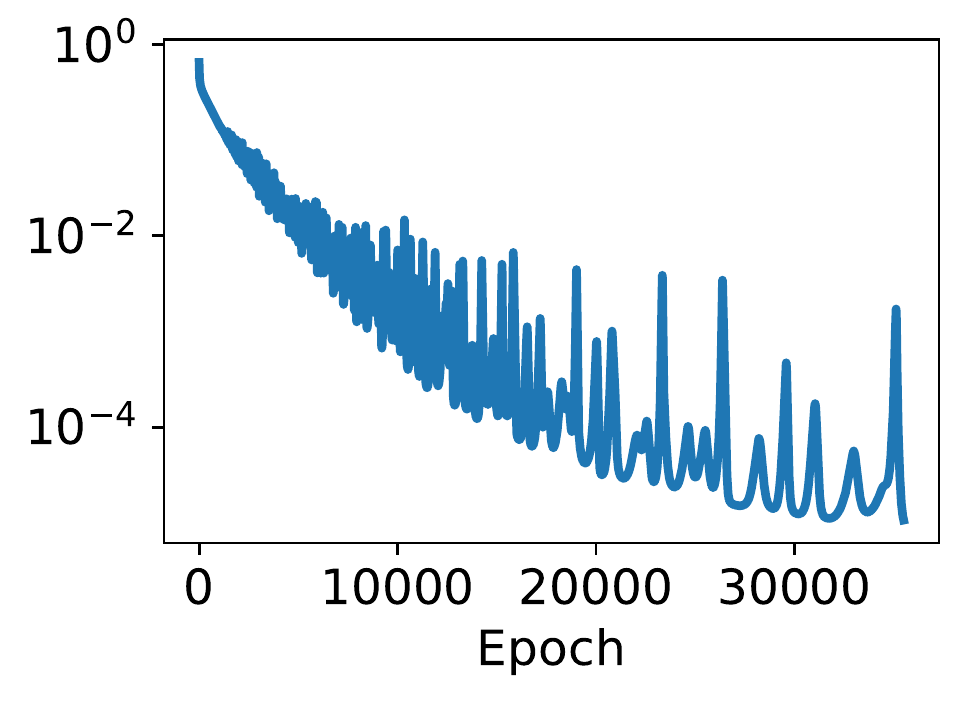}}
	\hfill
	\subfloat[VGG11-bn]{\hspace{0.1cm}\includegraphics[height=3.5cm]{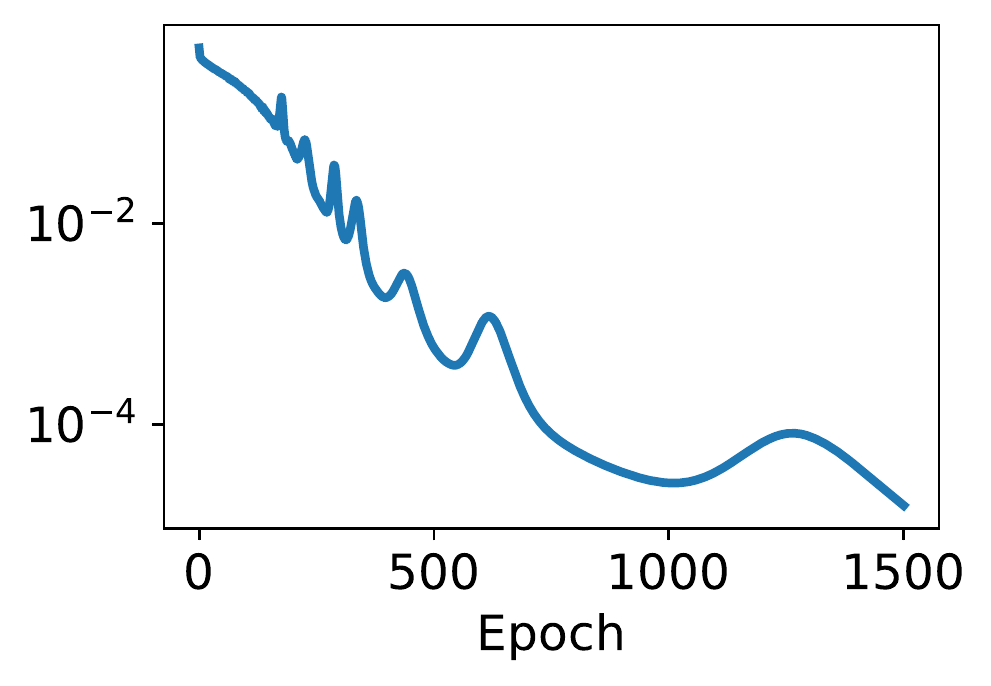}}
	\hfill
	\subfloat[ResNet20]{\hspace{-0.4cm}\includegraphics[height=3.5cm]{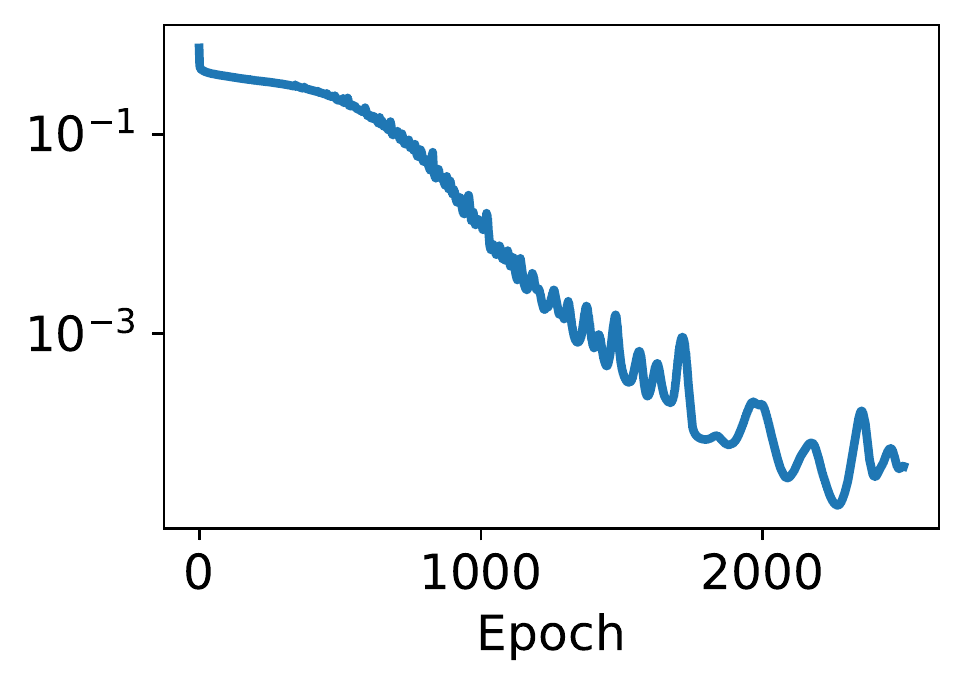}}
	\caption{\textbf{The GFS sharpness decrease monotonically to $2/\eta$ for common neural network architectures.} On three architectures, we run GD on a subset of CIFAR10 and calculate the GFS sharpness every 100 iterations using the Runge-Kutta algorithm. The sharpness (blue line) non-monotonically rises to the EoS $2/\eta$ (green dashed line) and the GFS sharpness (orange line) decreases monotonically to the same value. In contrast, at the bottom figures, we observe that the training loss exhibits non-monotonic and chaotic behavior.
		\label{Fig: sharpness, ps and train loss on three architectures}
	}
\end{figure*}

\section{Experiments}

Our goal in this and the following section is to test whether monotonic decrease of GFS sharpness holds beyond scalar linear networks. In \Cref{Sec: Illustration of results on squared regression model} consider a simple model where we can compute GFS sharpness exactly, while in \Cref{Sec: Illustration of results on realistic neural networks} we approximate GFS sharpness for practical neural networks. Overall, we find empirically that GD monotonically decreases GFS sharpness well beyond scalar networks.

\subsection{Squared Regression Model}\label{Sec: Illustration of results on squared regression model}

We consider the squared regression model
\begin{equation}
    f_{\btheta}\rbm{\xx} = \ip{\vect{u}_+^2 - \vect{u}_-^2}{\xx}=\ip{\bbeta}{\xx}\,.
\end{equation}
This 2-positive homogeneous model is analyzed in several previous works \citep[e.g.,][]{woodworth2020kernel, gissin2019implicit, moroshko2020implicit, pesme2021implicit, azulay2021implicit}. Importantly, for the MSE loss, \citet{azulay2021implicit} show that the linear predictor associated with the interpolating solution obtained by GF initialized at some $\vect{w}_0=\begin{bmatrix}\vect{u}_{+,0}^\top & \vect{u}_{-,0}^\top\end{bmatrix}^\top\in\R^{2d}$ where $\vect{w}_{0}\neq \vect{0}$ (element-wise) can be obtained by solving the following  problem:
\begin{equation}\label{Eq: GF projection for squared regression}
    \bbeta_{GF}\rb{\vect{w}_0} = \argmin_{\vect{\beta}} Q_{\vect{w}_0}\rb{\vect{\beta}} \text{ s.t. } \vect{X\beta}=\vect{y}\,,
\end{equation}
where $\vect{X}$ and $\vect{y}$ are the training data and labels respectively, 
\begin{align*}
    Q_{\vect{w}_0}\rb{\vect{\beta}} =& \sum_{i=1}^d \abs{w_{0,i}w_{0,i+d}} q\rb{\frac{\beta_i}{2\abs{w_{0,i}w_{0,i+d}}}}\notarxiv{\\
&}-\frac{1}{2}\vect{\beta}^\top\arcsinh\rb{\frac{w_{0,i}^2-w_{0,i+d}^2}{2\abs{w_{0,i}w_{0,i+d}}}}\,,
\end{align*} and \[q\rb{z}=1-\sqrt{1+z^2}+z\arcsinh\rb{z}\,.\]
Using this result, we can calculate the GFS sharpness efficiently in this model. (Full details in Appendix \ref{sec: squared regression model calc}.)

In Figure \ref{Fig: illustration on the squared regression model}, we summarize the results of running GD on synthetic random data until the training loss decreased below $0.01$ and calculating the GFS sharpness at each iteration using Eq. \eqref{Eq: GF projection for squared regression}. We repeat this experiment with $50$ different seeds and two large learning rates per seed.

\paragraph{Qualitative behavior of GFS sharpness.} In Figure \ref{subfig: sharpness and ps for the squared regression} we examine the sharpness ($\lambda_{\max}$, blue line), GFS sharpness ($\phi$, orange line) and the stability threshold ($2/\eta$, green line) for a single experiment. We observe that the GFS sharpness remains constant until the sharpness crosses the stability threshold $2/\eta$. This is expected since GD tracks the GF trajectory when the sharpness is much smaller than the stability threshold. Then, once the sharpness increased beyond $2/\eta$, we observe the GFS sharpness decreasing monotonically until converging at approximately $2/\eta$. In Figure \ref{subfig: ps for multiple seed and the squared regression} we examine the GFS sharpness obtained in eight different experiments and observe the same qualitative behavior. 

\paragraph{Quantitative behavior of GFS sharpness.} 
In Figure \ref{subfig: consistecny test squared regression}, to examine the consistency of the GFS sharpness monotonic behavior, we calculate for $100$ experiments the normalized change in the GFS sharpness, i.e., \[\frac{\Delta \phi}{\phi_{EoS}}(t)=\frac{\phi\rb{\vect{w}\rb{t+1}}-\phi\rb{\vect{w}\rb{t}}}{2/\eta},\] and produce a scatter plot of all $70,539$ points of the form
 $(\frac{\phi(\wt)}{\phi_{\mathrm{EoS}}}, \frac{\Delta \phi}{\phi_{EoS}}(t))$, 
 colored by $t$. We observe that the normalized difference change in the GFS sharpness is always below $0.00104$, with roughly $78\%$ of the points being negative.
That is, the GFS sharpness consistently exhibits decreasing monotonic behavior, with the few small positive values occurring early in the optimization.

\subsection{Realistic Neural Networks}\label{Sec: Illustration of results on realistic neural networks}

We now consider common neural network architectures. For such networks, there is no simple expression for the GF solution when initialized at some $\vect{w}_0$.
Therefore, we used  Runge-Kutta RK4 algorithm \citep{press2007numerical} to numerically approximate the GF trajectory, similarly to \cite{cohen2021gradient}.

In Figure \ref{Fig: sharpness, ps and train loss on three architectures}, we plot the sharpness, GFS sharpness, and training loss on three different architectures: three layers fully connected (FC) network with hardtanh activation function (the same architecture used in \cite{cohen2021gradient}), VGG11 with batch normalization (BN), and ResNet20 with BN. The fully connected network and the VGG11 networks were trained on a $1$K example subset of CIFAR10, and the ResNet was trained on a $100$ example subset of CIFAR10. We only used a subset of CIFAR10 for our experiments since each experiment required many calculations of the GFS sharpness by running Runge-Kutta algorithm until the training loss is sufficiently small (we use a convergence threshold of $10^{-5}$) which is computationally difficult. Full implementation details are given in Appendix \ref{Sec: Implementation details for the modern Architectures}.

In the figure, we observe that GFS sharpness monotonicity also occurs in modern neural networks. Similar to the scalar network and the squared regression model, we observe that the GFS sharpness decreases until reaching the stability threshold $2/\eta$. In contrast, we observe that the training loss oscillates while decreasing over long time scales.

In Figures \ref{Fig: sharpness and ps on more architectures} and \ref{Fig: sharpness and ps on svhn dataset} in Appendix \ref{Sec: Additional experiments on modern architectures}, we observe the same qualitative behavior on more architectures (FC with tanh and ReLU activations) and another dataset (SVHN). We note that we occasionally see a slight increase in GFS sharpness early in the optimization (see Figures \ref{subfig:resnet GFS sharpness} and \ref{subfig: relu GFS sharpness}).
\section{Discussion}

In this work, for scalar linear networks, we show that GD monotonically decreases the GFS sharpness. In addition, we use the fact that GD tracks closely the GPGD dynamics to show that if the loss is sufficiently small when the GFS sharpness decreases below the stability threshold $\frac{2}{\eta}$, then the GFS sharpness will stay close to the $\frac{2}{\eta}$ and the loss will converge to zero. This provides a new perspective on the mechanism behind the EoS behaviour. Finally, we demonstrate empirically that GFS sharpness monotonicity  extends beyond scalar networks. A natural future direction is extending our analysis beyond scalar linear networks.

There are several additional directions for future work. First, it will be interesting to explore how the GFS sharpness (or some generalization thereof) behaves for other loss functions. Note that this extension is not trivial since for losses with an exponential tail, e.g., cross-entropy, the Hessian vanishes as the loss goes to zero. Therefore, on separable data, the GFS sharpness vanishes. A second direction is the stochastic setting. Namely, it will be interesting to understand whether stochastic gradient descent and Adam also exhibit some kind of GFS sharpness monotonicity, as the EoS phenomenon has been observed for these methods as well (\citet{lee2023a} and \citet{cohen2022adaptive}).

Finally, a third direction is studying GFS sharpness for non-constant step sizes, considering common practices such as learning rate schedules or warmup. We generally expect GFS sharpness to remain monotonic even for non-constant step sizes. More specifically, for scalar neural networks and GD with schedule $\eta_t$, the GFS sharpness will be monotone when $\boldsymbol{w}^{(t)} \in \mathbb{S}_{\eta_t}^D$ for all $t$. In particular, since the set $\mathbb{S}_{\eta}^D$ increases as $\eta$ decreases (see Lemma \ref{lem:subset order of SG}), i.e., $\mathbb{S}_{\eta_{t_1}}^D\subseteq \mathbb{S}_{\eta_{t_2}}^D$ for any $\eta_{t_1}\ge \eta_{t_2}$, we get that for a scalar network with $\boldsymbol{w}^{(0)} \in \mathbb{S}_{\eta_0}^D$, the GFS sharpness will decrease monotonically for any decreasing step size schedule. 
In addition, it may be possible to view warmup (i.e., starting from a small step size and then increasing it gradually) as a technique for ensuring that $\wt\in\mathbb{S}_{\eta_t}^D$ for all $t$ and a larger set of initializations $\wt[0]$, giving a complementary perspective to prior work such as \citet{gilmer2021loss}. 

 \section*{Acknowledgments}  %
IK and YC were supported by Israeli Science
Foundation (ISF) grant no. 2486/21, the Alon Fellowship, and the Len Blavatnik and the Blavatnik Family Foundation. The research of DS was funded by the European Union (ERC, A-B-C-Deep, 101039436). Views and opinions expressed are however those of the author only and do not necessarily reflect those of the European Union or the European Research Council Executive Agency (ERCEA). Neither the European Union nor the granting authority can be held responsible for them. DS also acknowledges the support of Schmidt Career Advancement Chair in AI. 
\clearpage

\bibliographystyle{icml2023}

\clearpage
\appendix
\onecolumn

\section{Discussion of \titleDamian in the Setting of Scalar Networks}
\label{sec: M in scalar networks}

\Citet{damian2022self} study the EoS phenomenon and introduce a “self-stabilization” theory. In this section, we explain why this analysis cannot adequately explain the EoS phenomenon in scalar networks.

The analysis of \citet{damian2022self} relies on the notation of the “stable set” \[\mathcal{M}\triangleq\left\{ \vect{w} ~~ : ~~ \lambda_{\max}\left(\nabla^2\mathcal{L}\left(\vect{w}\right)\right) \le 2/\eta ~ \text{and} ~ \nabla\mathcal{L}\left(\vect{w}\right)\cdot u\left(\vect{w}\right)=0 \right\},\] where $\mathcal{L}$ is the loss and $u(\vect{w})$ is the top eigenvector of the loss Hessian (assumed to be unique). They argue that, under certain assumptions, GD tracks the constrained trajectory of GD projected to $\mathcal{M}$. 

However, for scalar networks with square loss, we observe that $\mathcal{M}$ consists of two disjoint sets: all stable global minima (points with zero loss), and another set of points with strictly positive loss. Theoretically, we can see this from the expressions for the loss gradient and Hessians (equations \eqref{eq:gradient} and \eqref{eq:Hessian at opt}, respectively). When the loss is small we have:
\[ \nabla\mathcal{L}\left(\vect{w}\right)\cdot u\left(\vect{w}\right) \approx \left( \pi(\boldsymbol{w})-1\right)\pi(\boldsymbol{w}) \left\Vert \boldsymbol{w}^{-1}\right\Vert \]
which is non-zero when the loss is non-zero. Empirically, in \cref{Fig:M in scalar network}, we numerically compute $\nabla\loss\rb{\vect{w}}\cdot u\rb{\vect{w}}$ for depth 3 scalar networks, and observe that the set $\mathcal{M}$ is contained in 
\[\left\{\boldsymbol{w} : \mathcal{L}\left(\boldsymbol{w}\right)=0 \right\} \cup \left\{\boldsymbol{w}~~ : ~~ w_{[2]}=w_{[3]} ~ \text{and} ~ w_1 w_2 w_3 = \frac{1}{2} \right\},\]
Confirming that the set of minima is a disjoint component in $\mathcal{M}$.

Therefore, for scalar networks the projected GD considered by \citet{damian2022self} cannot smoothly decrease the loss toward zero. This contradicts their theoretical results (e.g., their Lemma 8) meaning that their assumptions do not hold for scalar networks.

Finally, we also note that Equations \eqref{eq:loss partial derivative 2 ij} and \eqref{eq:loss partial derivative 3 ii} in our paper imply that, if $w_{[D-1]}=w_{[D]}$ and the product of the weights $\pi\left(\boldsymbol{w}\right)$ is $\frac{1}{2}$, then the Hessian is a diagonal matrix where the top eigenvalue is not unique.
This contradicts an assumption made on the top eigenvalue function $u$ \citep[Definition 1]{damian2022self}.

\begin{figure*}[h]
	\centering
	
	\begin{minipage}{.85\linewidth}
		\subfloat[$w_3=2.0$]
		{\includegraphics[width=0.33\textwidth,valign=t]{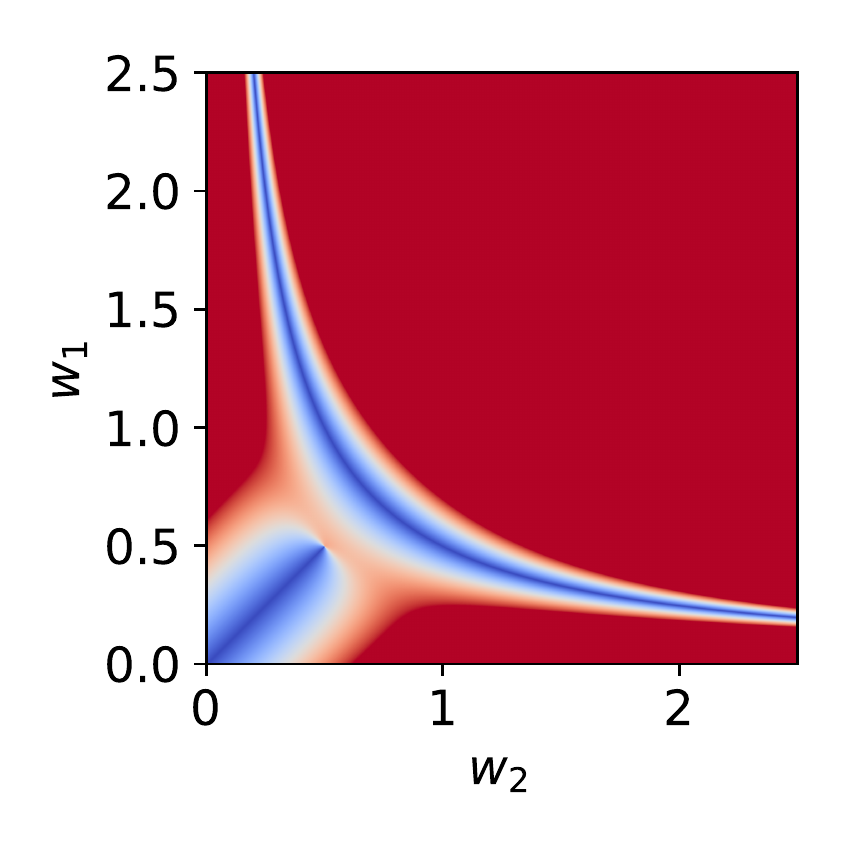}}
		\hfill
		\subfloat[$w_3=1.5$]
		{\includegraphics[width=0.33\textwidth,valign=t]{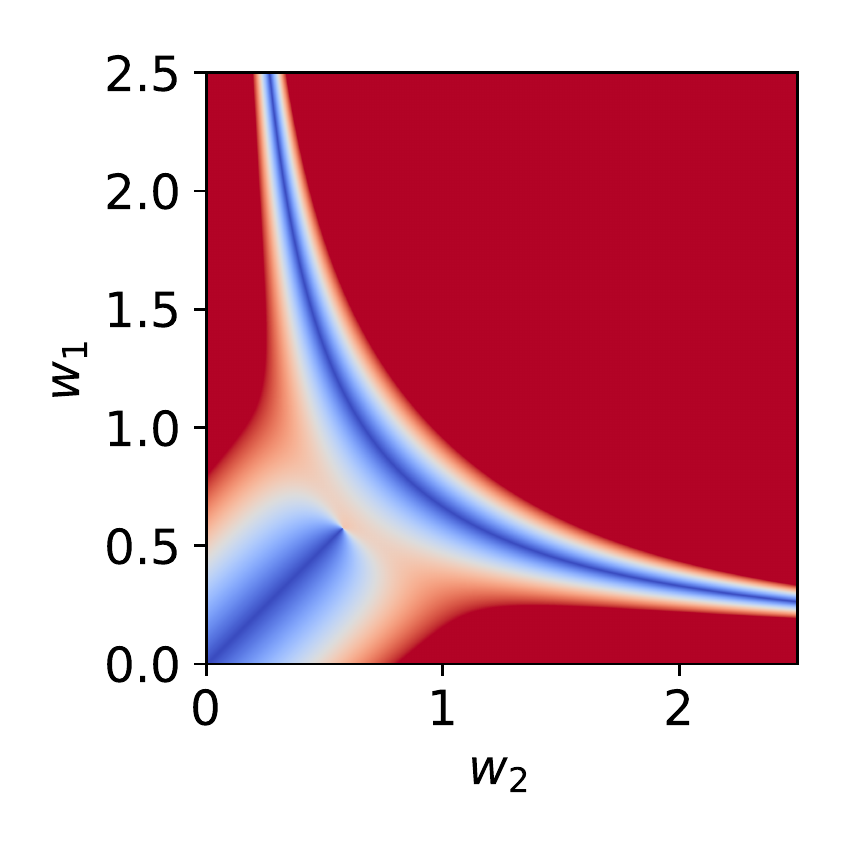}}
		\hfill
		\subfloat[$w_3=1.0$]
		{\includegraphics[width=0.33\textwidth,valign=t]{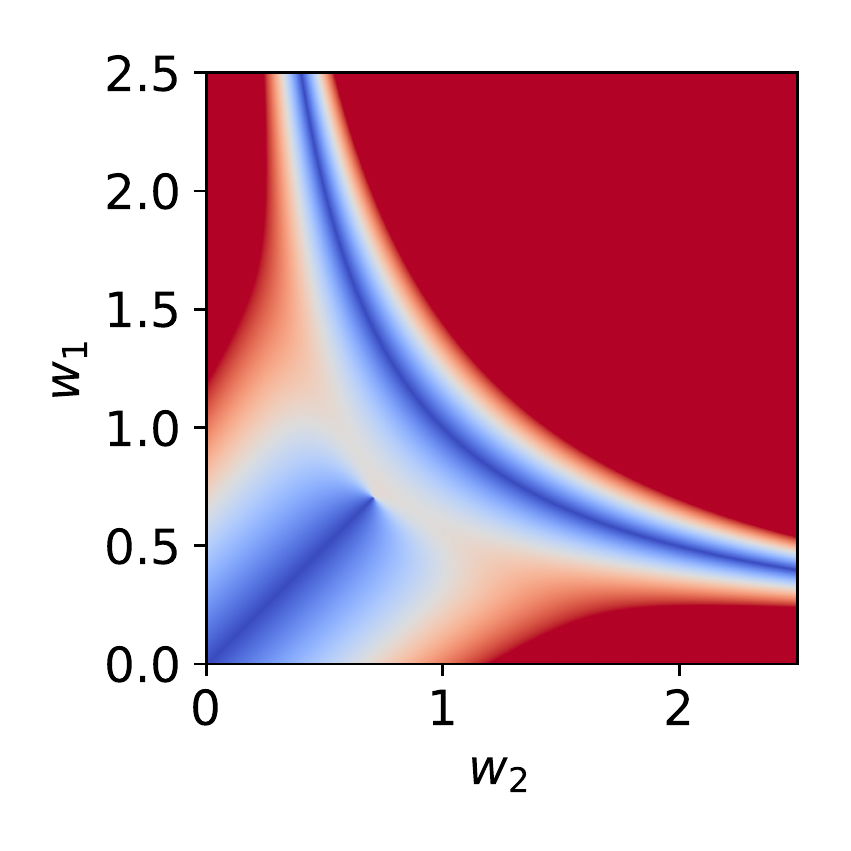}}
		\\
		\subfloat[$w_3=0.9$]
		{\includegraphics[width=0.33\textwidth,valign=t]{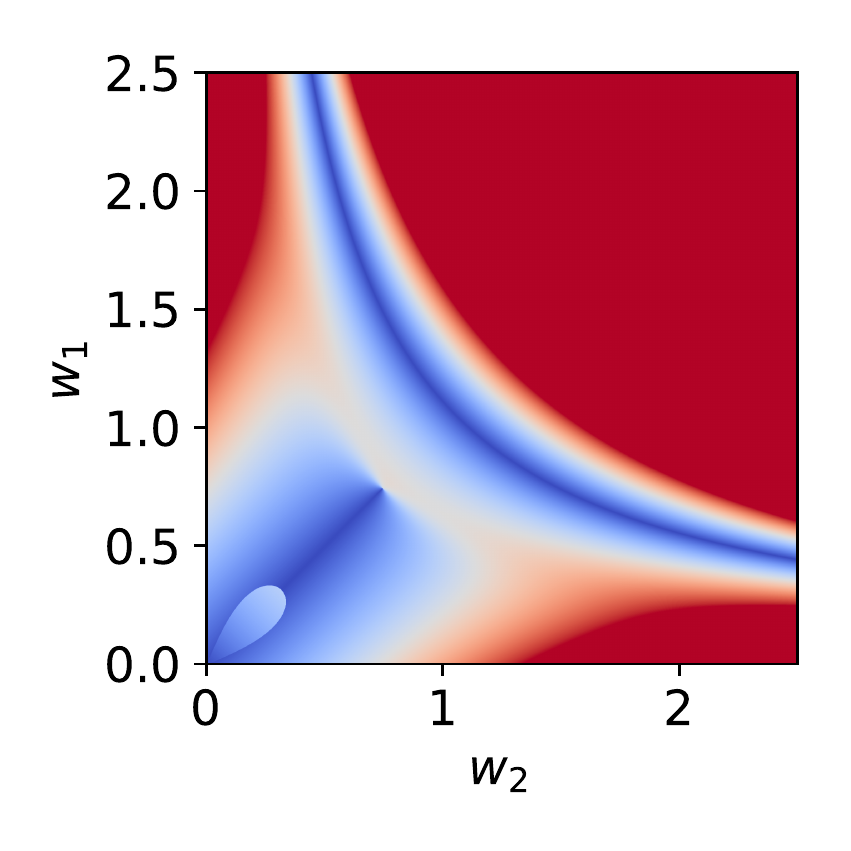}}
		\hfill
		\subfloat[$w_3=0.7$]
		{\includegraphics[width=0.33\textwidth,valign=t]{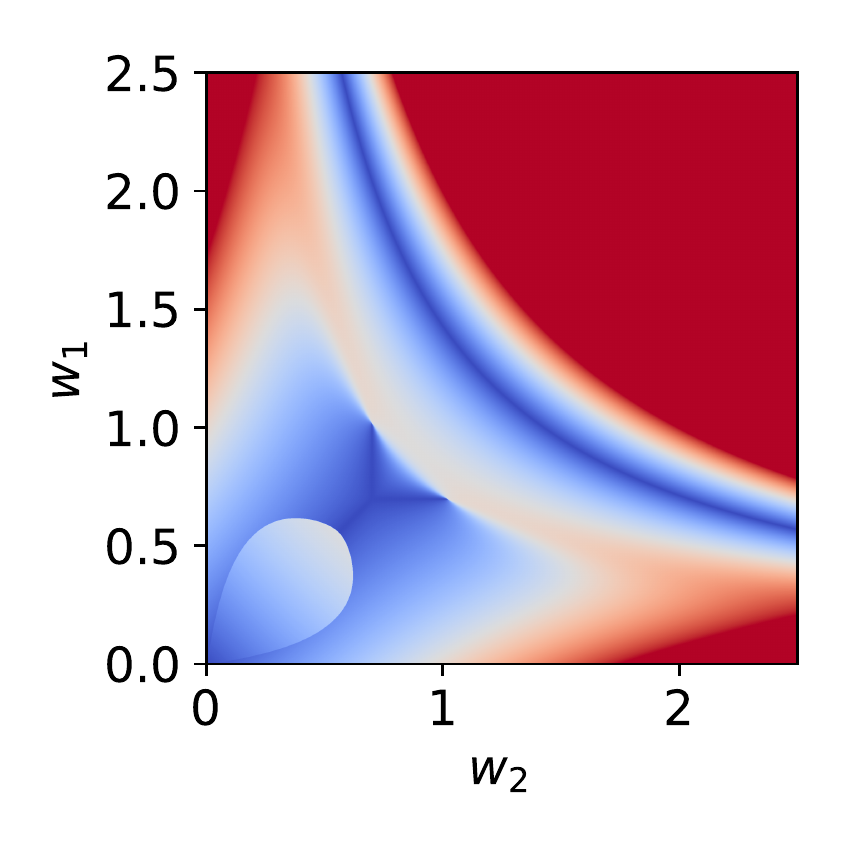}}
		\hfill
		\subfloat[$w_3=0.5$]
		{\includegraphics[width=0.33\textwidth,valign=t]{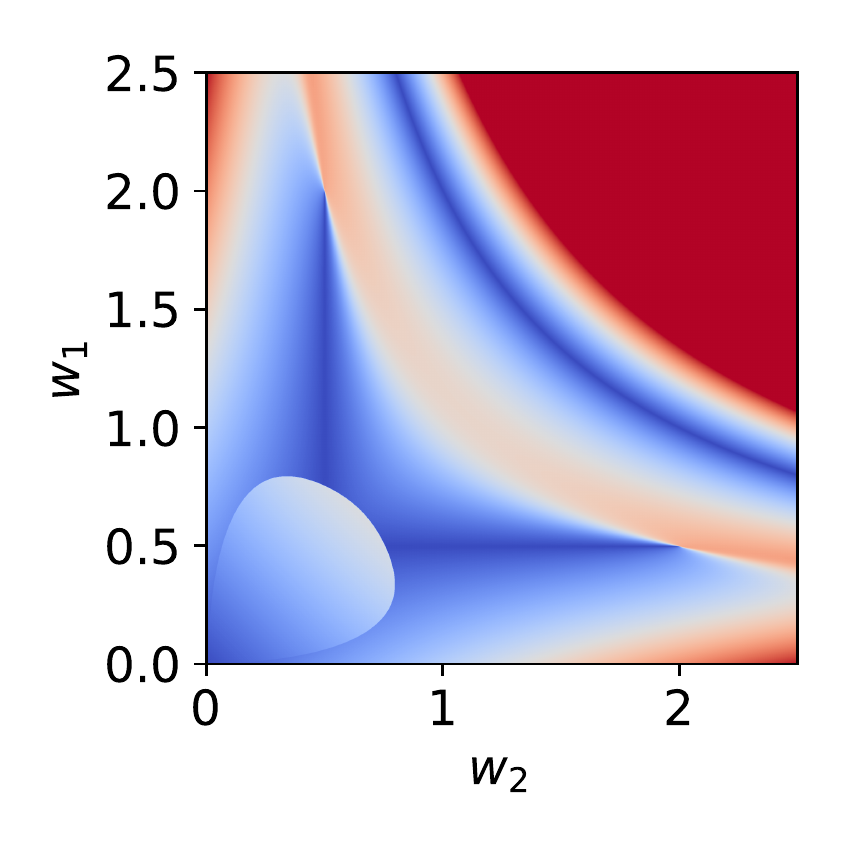}}
	\end{minipage}
	\hfill
	\begin{minipage}{.14\linewidth}
		\subfloat
		{\includegraphics[width=\textwidth,valign=t]{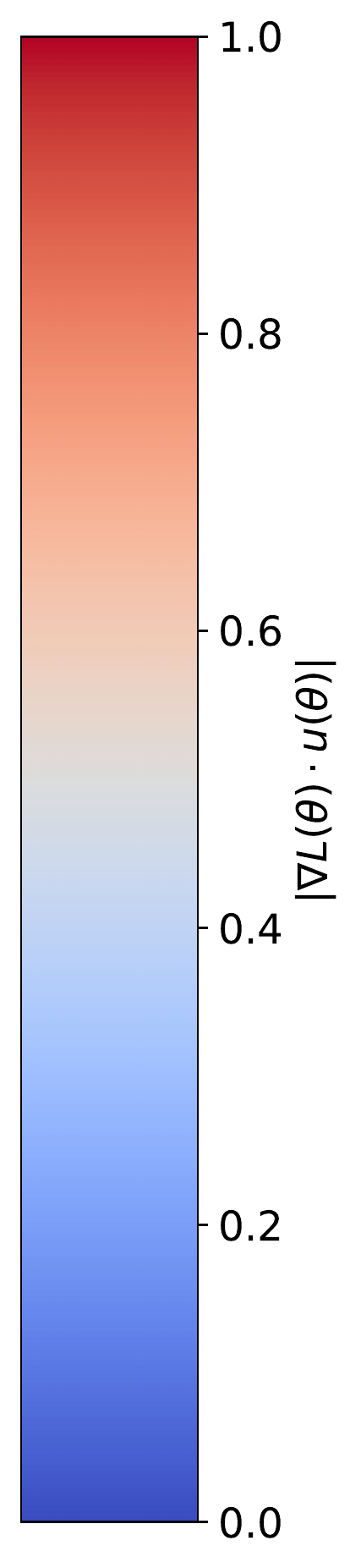}}
	\end{minipage}

	\caption{
		Illustration of the values of $\abs{\nabla\loss\rb{\vect{w}}\cdot u\rb{\vect{w}}}$ for different weights for scalar networks of depth $3$. We observe that the condition $\nabla\loss\rb{\vect{w}}\cdot u\rb{\vect{w}}=0$ implies that $\vect{w}\in\left\{\boldsymbol{w} : \mathcal{L}\left(\boldsymbol{w}\right)=0 \right\} \cup \left\{\boldsymbol{w}~~ : ~~ w_{[2]}=w_{[3]} ~ \text{and} ~ w_1 w_2 w_3 = \frac{1}{2} \right\}.$
	}
	
	\label{Fig:M in scalar network}
\end{figure*}

\section{Additional Experiments and Implementation Details}

\subsection{Sharpness at Convergence for Scalar Networks}
Note that it is not possible to converge to a minimum with sharpness larger than $\frac{2}{\eta}$ \citep[e.g.,][]{ahn2022understanding}.
In \cref{Thm: convergence theorem} we show that, under certain conditions, scalar network converge to a minimum with sharpness that is only slightly smaller than $\frac{2}{\eta}$.
In \cref{Figure: GD converge slightly below two over the step size}, we show a zoomed-in version of \cref{Fig: Scalar network EoS sharpness} to illustrate that indeed the scalar network converged to sharpness slightly below $\frac{2}{\eta}$.

\begin{figure*}[h]
	\centering
	
	\subfloat[]{\includegraphics[height=3.9cm]{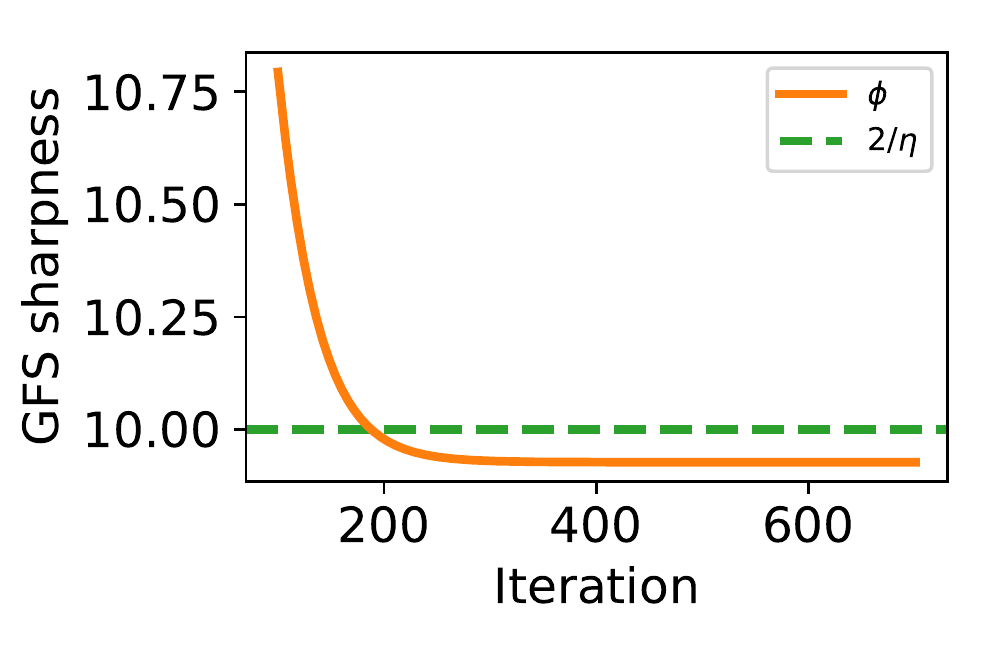}\label{Fig: scalar network converges slightly below two over the step size}}
	
	\caption{\textbf{The GFS sharpness converges to slightly below $\frac{2}{\eta}$}. \protect This is a zoomed-in version of \cref{Fig: Scalar network EoS sharpness}.}
	\label{Figure: GD converge slightly below two over the step size}
\end{figure*}

In Figure \ref{Figure: GD converge below two over the step size}, we run GD in the same setting as in Figure \ref{Fig: scalar network example}, i.e. step size of $0.2$ and a scalar network of depth 4. The specific initialization of $[2.57213954, 2.57213954, 0.65589001, 0.65589001]$ was chosen by examining  \cref{Fig: illustration of sharpness at convergence} and selecting an initialization with GFS sharpness above $\frac{2}{\eta}$ but converging to sharpness that is relatively far below $\frac{2}{\eta}$.
In the figure we can see an example in which the loss was fairly large (loss of 5.1) when the GFS sharpness crossed  $\frac{2}{\eta}$, thus not satisfying the condition in Theorem \ref{Thm: convergence theorem} and converging to a sharpness value relatively far below $\frac{2}{\eta}$.

\begin{figure*}[h]
	\centering
	
	\subfloat[]{\includegraphics[height=3.9cm]{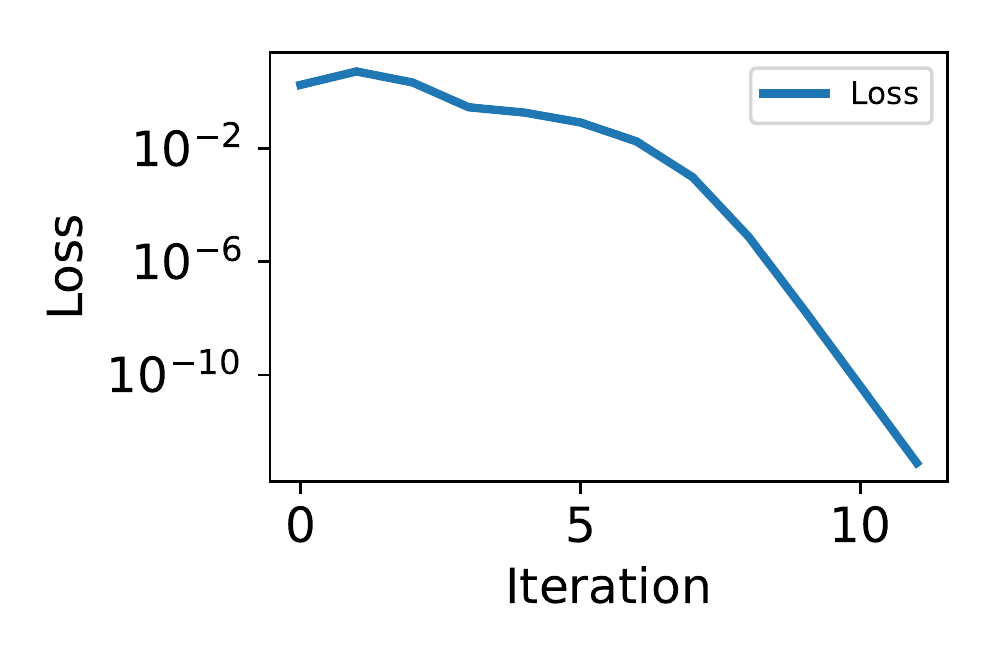}}
	\subfloat[]{\includegraphics[height=3.9cm]{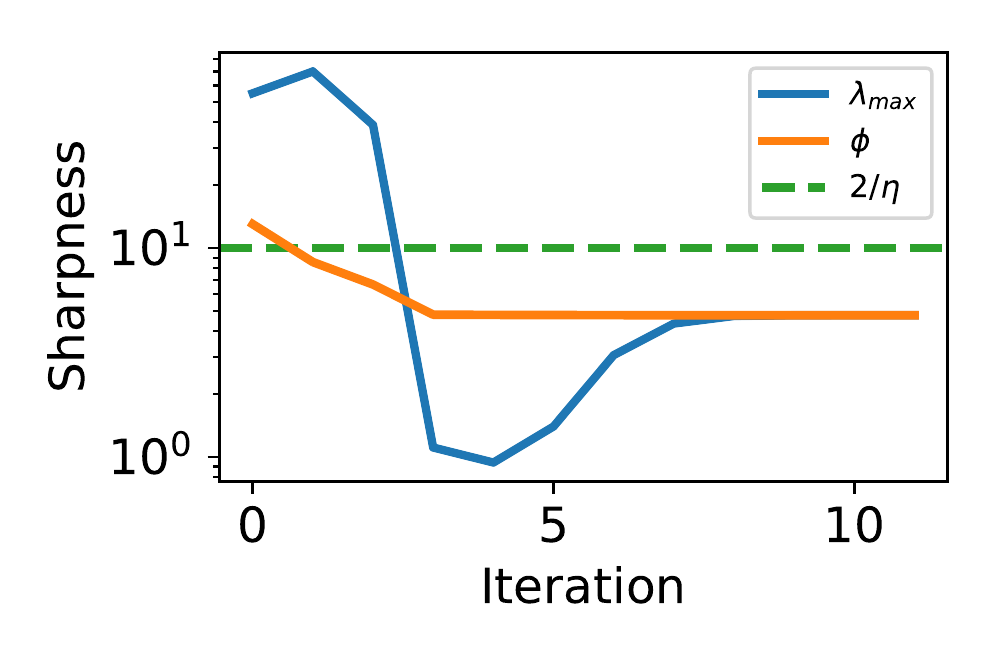}}
	
	\caption{\textbf{The GFS sharpness can converge significantly below $\frac{2}{\eta}$}. The figure shows a depth 4 network with initialization $[2.57213954, 2.57213954, 0.65589001, 0.65589001]$ and GD step size $0.2$.We observe that GFS sharpness converges relatively far below $\frac{2}{\eta}$.\label{Figure: GD converge below two over the step size}}
	
\end{figure*}

\subsection{GD Follows the GPGD Bifurcation Diagram\label{Sec: GD Follows the GPGD Bifurcation Diagram}}

This section supplements \Cref{sec: GD Follows the GPGD Bifurcation Diagram},
providing additional details on the GFS-preserving GD (GPGD) bifurcation diagram, a comparison with \citep{zhu2022understanding}, and evidence that the set $\SG$ may contain point where GD exhibits chaotic behavior.

\paragraph{Computation of GPGD bifurication diagram.}
To compute the GPGD bifurcation diagram, we first compute the GD trajectory $\wt$. For each $t$, we run $K=5\cdot10^6$ steps of GPGD as defined in \cref{Eq: qs GD}, starting from $\wt$, resulting in the sequence $\tilde{\vect{w}}^{(t,1)}, \ldots, \tilde{\vect{w}}^{(t,K)}$.
We then add all the points of the form $\{(\phi(\wt), \pi(\tilde{\vect{w}}^{(t,K-k)}))\}_{t\ge0, k\in[0,10^4)}$ to the scatter plot forming the GPGD bifurcation diagram (note that $\phi(\tilde{\vect{w}}^{(t,k)})=\phi(\wt)$ for all $t$ and $k$ by definition of GPGD).
That is, we let GPGD converge to a limiting period (when it exists), and then plot the last few iterates.

\paragraph{Comparison with \citet{zhu2022understanding}.}
\citet{zhu2022understanding} analyze scalar networks of depth 4, initialized so that $w_1=w_2$ and $w_3=w_4$ for all points in the GD trajectory. They approximate the (two-step) dynamics of the GFS sharpness $\phi(\wt)$ and the weight product $\pi(\wt)$ assuming that $\eta$ is small, and observe that the $\phi(\wt)$ updates are of a lower order in $\eta$.\footnote{More precisely, \citet{zhu2022understanding} consider a re-parameterization of the form $(a,b)$ such that $a=h(\phi(\vect{w}))$ for some one-to-one function $h$, and $b=\pi(\vect{w})-1$.} Consequently,
they consider an approximation wherein the GFS sharpness is fixed to some value $\phi$, and the weight product evolves according to a scalar recursion parameterized by $\phi$. Similar to GPGD, for different values of $\phi$ the approximated $\pi$ dynamics either converge to a periodic solution or becomes chaotic. \Cref{Fig:bifurcation GD vs B} shows this bifurcation diagram of the approximate dynamics superimposed on values of $(\phi(\wt),\pi(\wt))$ of the true GD trajectory, reproducing Figure 6c from \cite{zhu2022understanding}. The figure shows that the bifurcation diagram is qualitatively similar to the GD trajectory, but does not approximate it well. By contrast, GPGD provides a close approximation of the GD dynamics for a wide of range GFS sharpness values.

\begin{figure*}[h!]
	\centering
	
	\subfloat[GD iterates compared to approximation from \citet{zhu2022understanding}]{\includegraphics[width=0.5\textwidth]{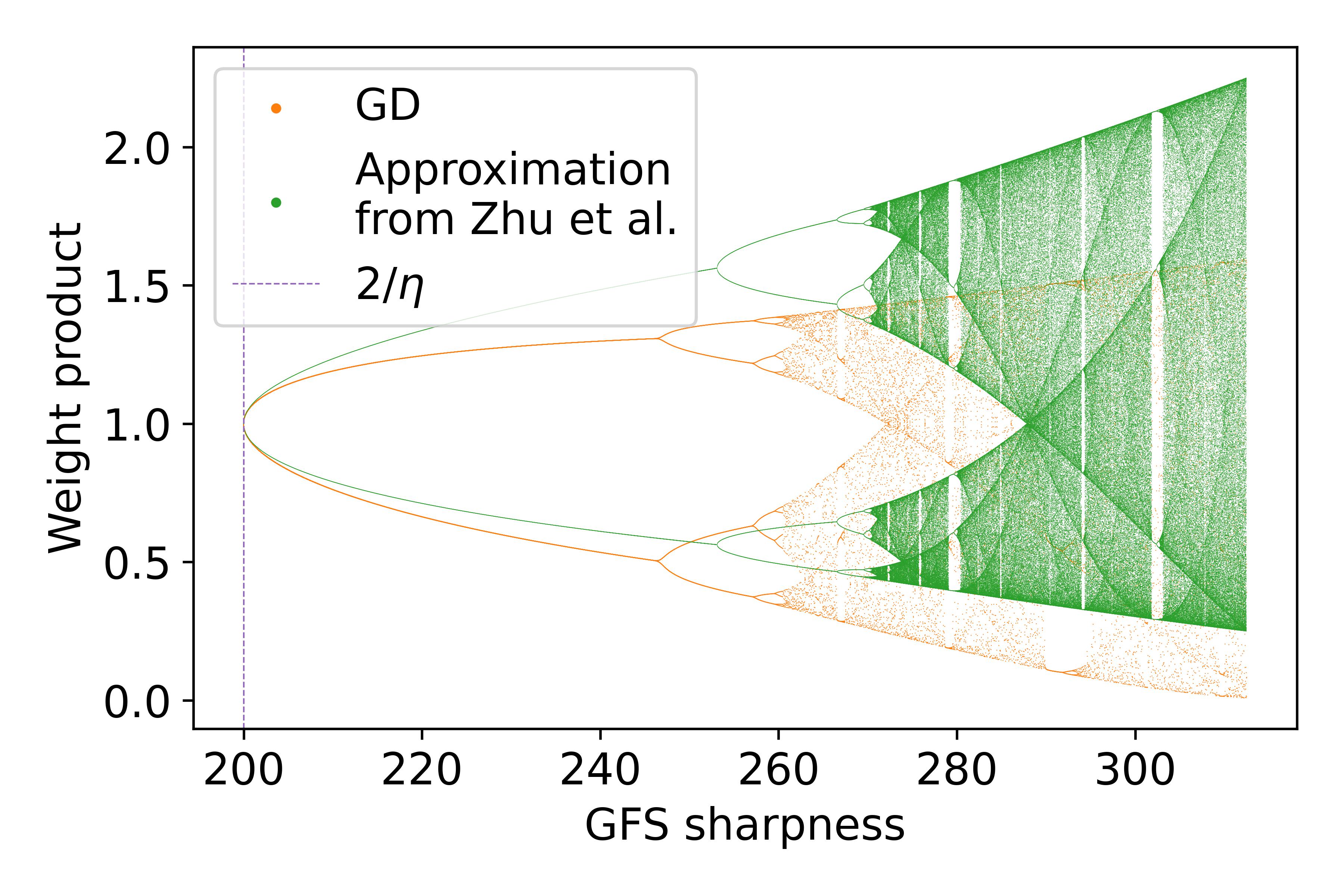}\label{Fig:bifurcation GD vs B}}
	\hfill
	\subfloat[GD iterates compared to GPGD]{\includegraphics[width=0.5\textwidth]{Figures/bifrucation_D=4_GPGDvsReal.jpg}\label{Fig:bifurcation GD vs GPGD}}
	\\
	\subfloat[Zoom in (left)]{\includegraphics[width=0.5\textwidth]{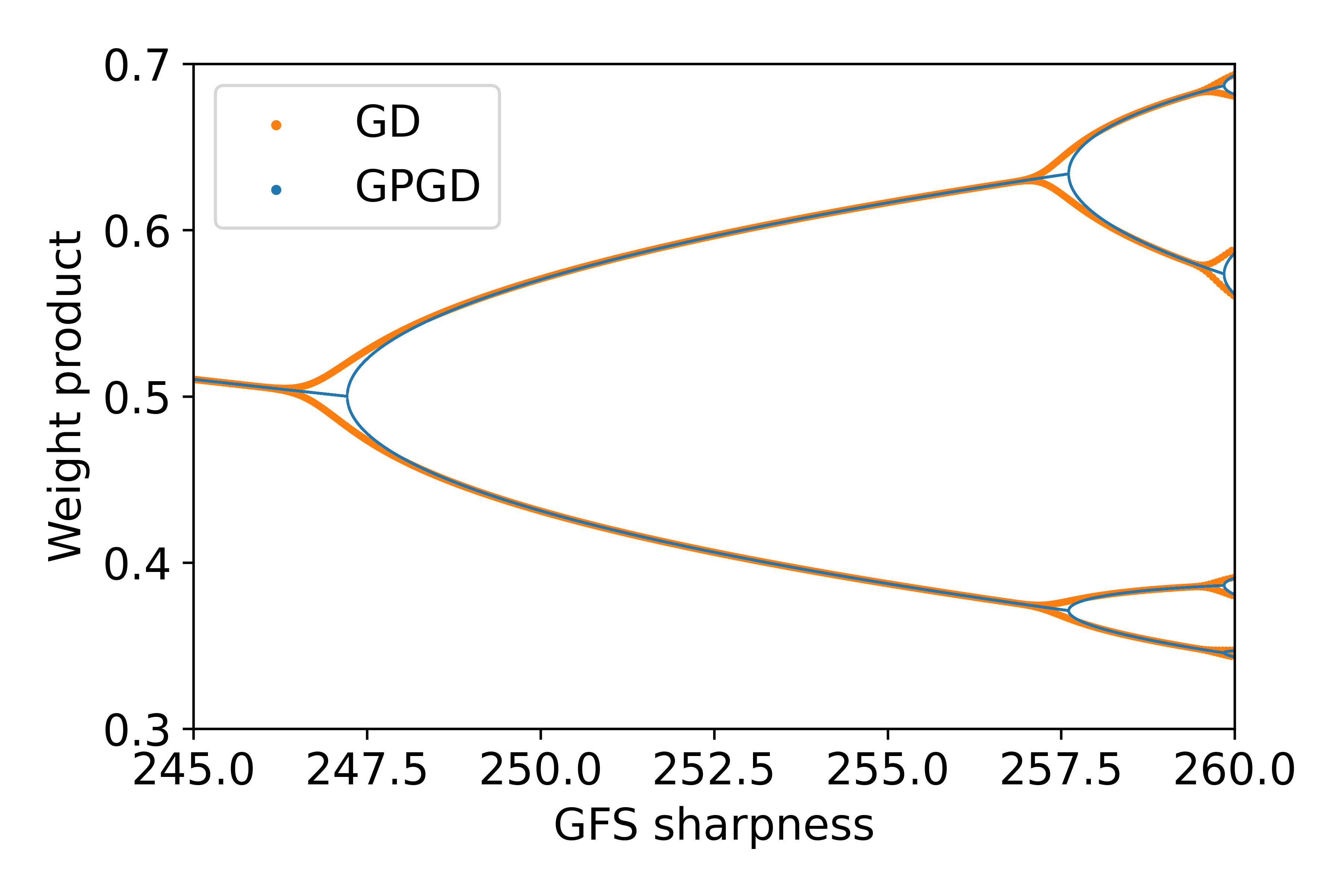}\label{Fig:bifurcation zoom in left}}
	\hfill
	\subfloat[Zoom in (right)]{\includegraphics[width=0.5\textwidth]{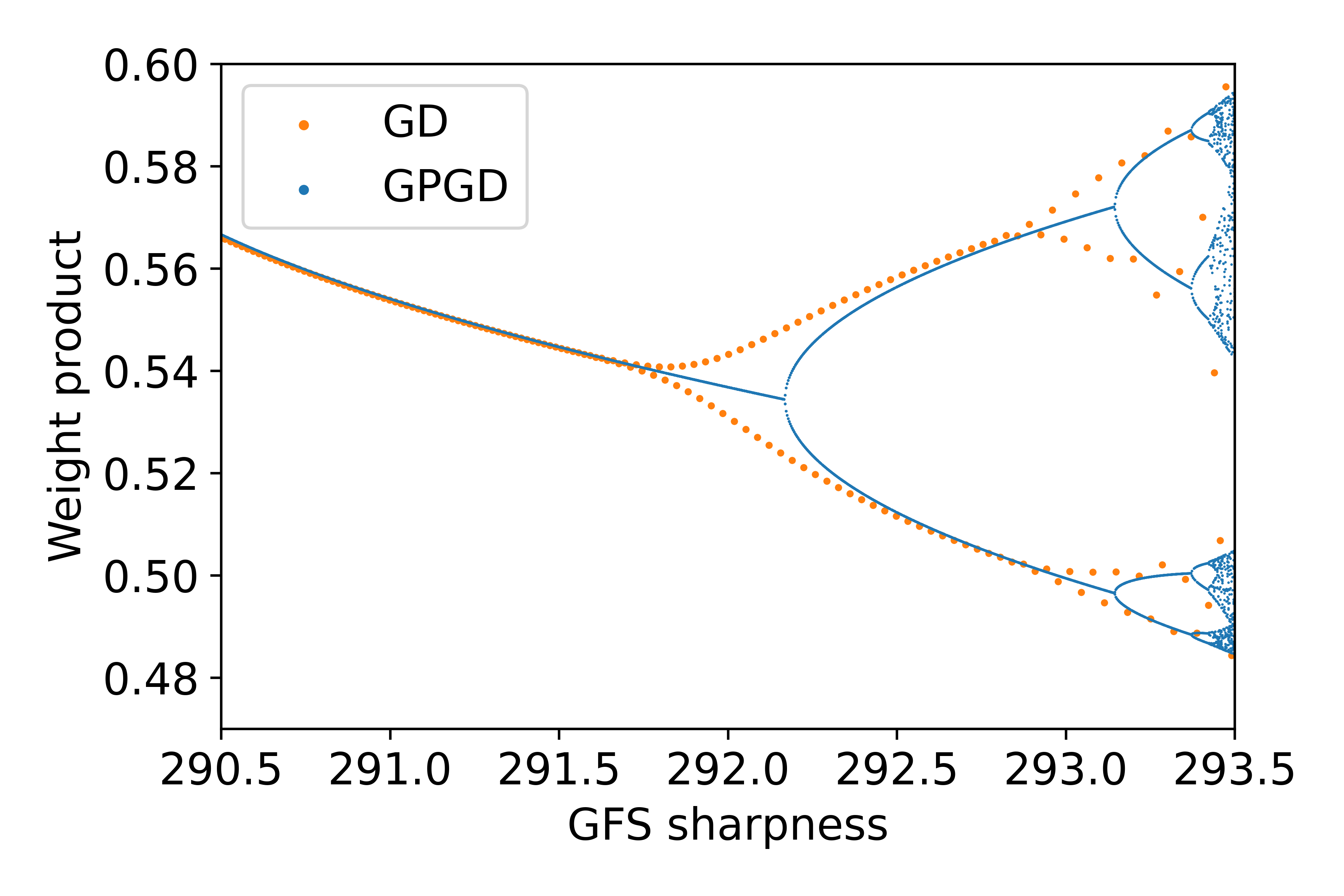}\label{Fig:bifurcation zoom in right}}
	
	\caption{\textbf{GD tracks the GPGD bifurcation diagram.} For a scalar linear network with depth $4$, initialization of $[12.5, 12.5, 0.05, 0.05]$, and $\eta=\frac{2}{200}=0.01$. 
		\protect\subref{Fig:bifurcation GD vs B} The bifurcation diagram of the approximate dynamics proposed by \citet{zhu2022understanding} is qualitatively similar to the GD trajectory, but does not approximate it well. 
		\protect\subref{Fig:bifurcation GD vs GPGD} The GPGD bifurcation diagram approximates the GD trajectory very closely. 
		When the GFS sharpness is close to $\frac{2}{\eta}=200$ both GP and GPGD attain loss $0$ (y axis value $1$), meaning we converge close to the stability threshold.
		\protect\subref{Fig:bifurcation zoom in left} and \protect\subref{Fig:bifurcation zoom in right} Zoom-in on regions highlighted in panel \protect\subref{Fig:bifurcation GD vs B}, showing the GPGD bifurcation diagram and GPGD diverge slightly near bifurcation points.
		\label{Fig:bifurcation}
	}
\end{figure*}

\paragraph{$\SG$ contains points where GD is chaotic.}
The similarity between GD and GPGD allows up to identify the chaotic parts of the GD trajectory with the chaotic states of GPGD. More precisely we say that GD is chaotic at iterate $\wt$ if starting iterating GPGD starting from $\wt$ does not converge to a periodic sequence. \Cref{Fig:GPGD periods in S} demonstrate that the set $\SG$ in which the GFS sharpness is guaranteed to decrease can contain such chaotic points. By contrast, the approximation of \citet{zhu2022understanding} is only valid when the approximate dynamics are periodic with period 2 or less.

\begin{figure*}[h]
	\centering
	\subfloat
	{\includegraphics[width=0.391\textwidth,valign=m]{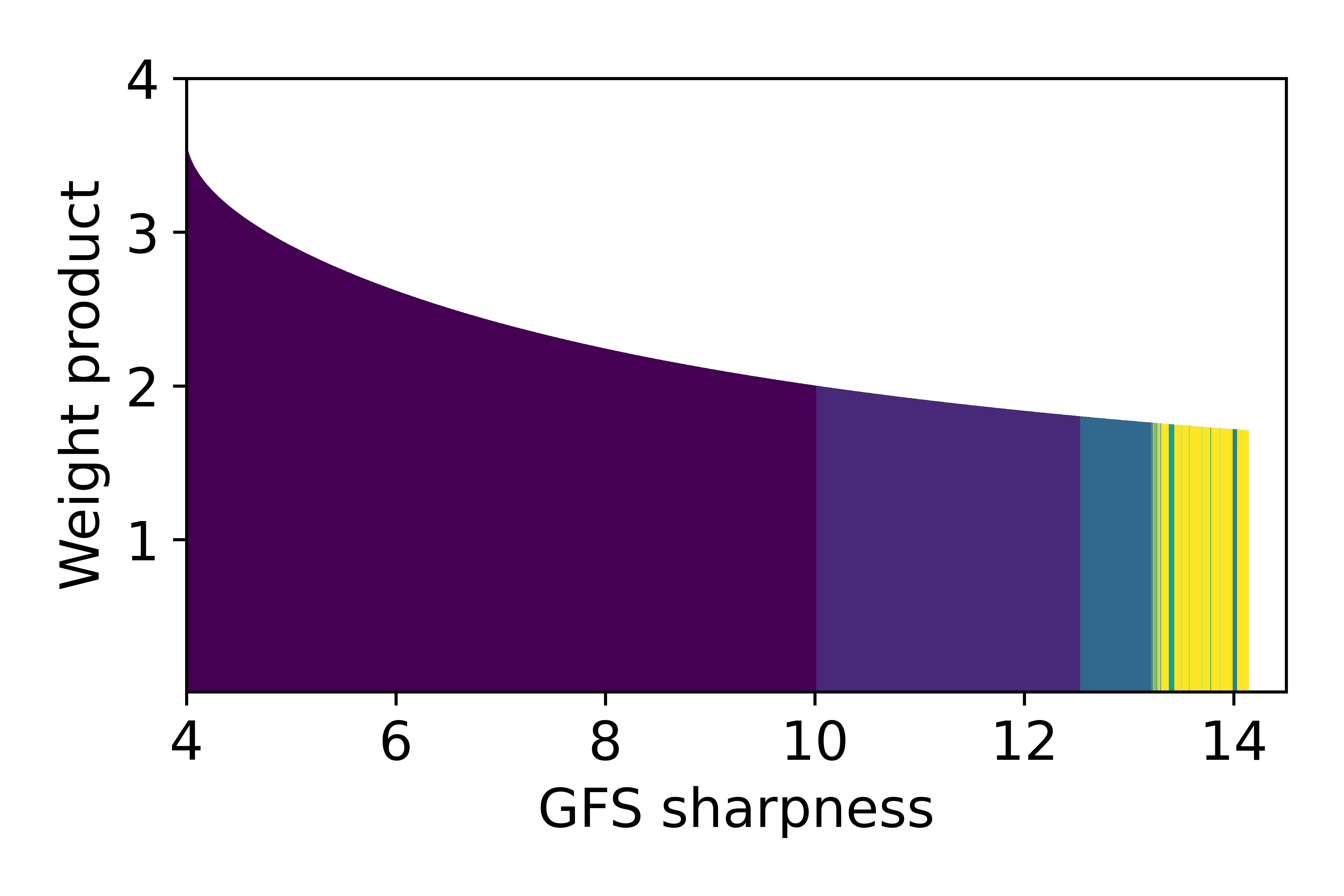}}
	\subfloat
	{\includegraphics[width=0.11\textwidth,valign=m]{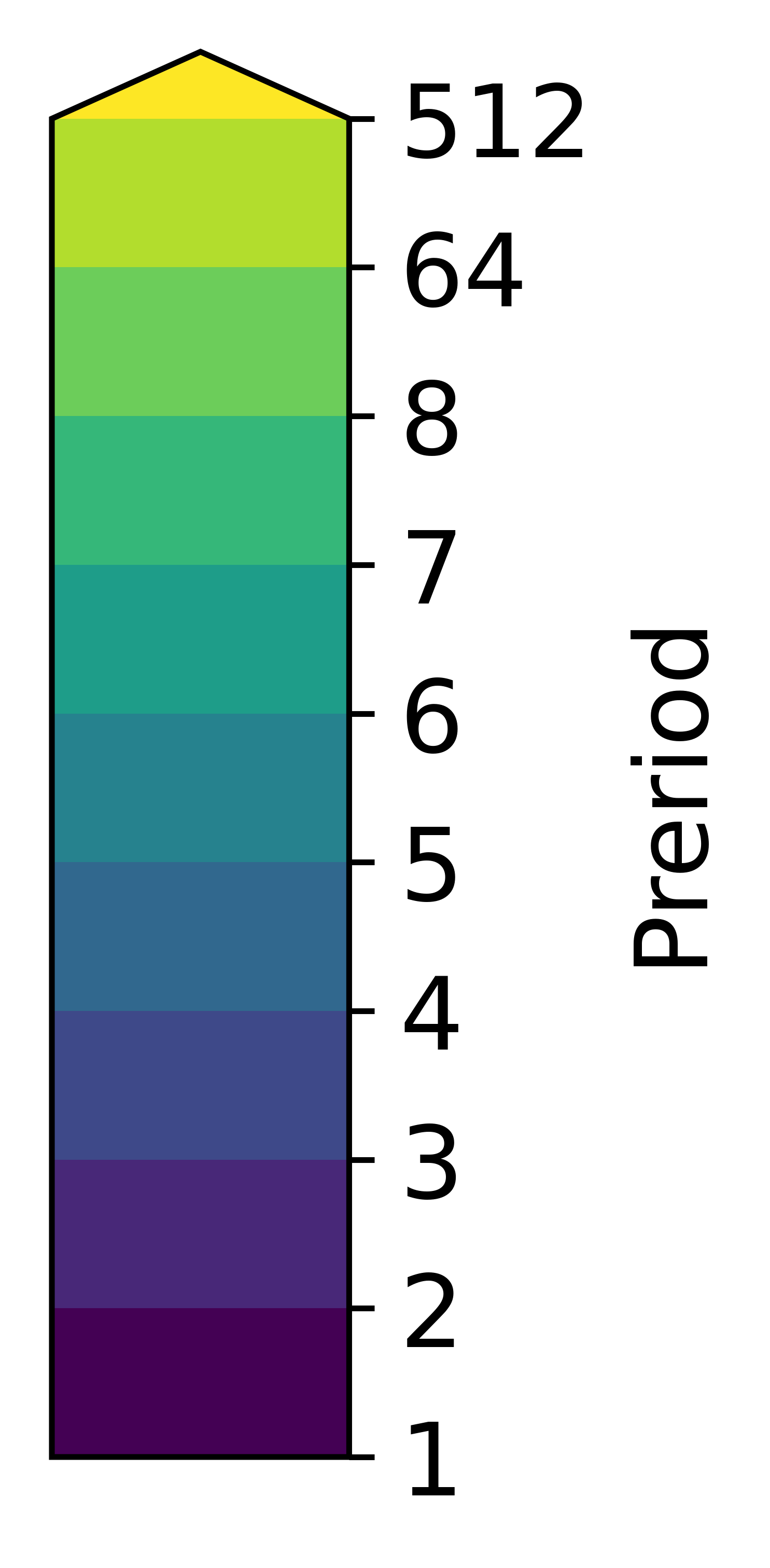}}
	
	\caption{\textbf{$\SG$ contains points where GD is chaotic.}
		For every point in $\SG$ the figure presents the period of the bifurcation diagram of the GPGD for the same GFS sharpness.
		We observe that $\SG$ contains points with a large range of periods going from 1 (convergence) to above 512 (indicating chaotic behavior). The figure is generated with the same settings as \cref{Fig: illustration of when the assumption is satisfied} and \cref{Fig: illustration of sharpness at convergence}.
	}
	
	\label{Fig:GPGD periods in S}
\end{figure*}

\subsection{Implementation Details for the Squared Regression Experiments (Figure \ref{Fig: illustration on the squared regression model}) \label{Sec: Implementation details for the squared regression experiments}}
To examine the GFS sharpness behavior for the squared regression model we run GD on a synthetic Gaussian i.i.d. dataset $\rb{\vect{x}_n, y_n}_{n=1}^N$ where $\vect{x}_n\sim \mathcal{N}\rb{\vect{\mu},\vect{\Sigma}}$ with $\vect{\mu}=5\cdot\vect{1}_d$, $\Sigma=5\cdot\vect{I}_{d \times d}$, $d=100$ and $N=50$. Here, $\vect{1}_d\in\R^d$ is a vector containing all ones and $\vect{I}_{d\times d}\in\R^{s\times s}$ is the identity matrix.
Our stopping criterion was the loss decreasing below a threshold of $0.01$. At each iteration, we calculated the GFS sharpness using Eq. \eqref{Eq: GF projection for squared regression} and the the SciPy optimization package. We repeat this experiment with $50$ different seeds and two large learning rates per seed: $\eta_1=0.85\cdot(2/\min_{\btheta}\lambda_{\max}\rb{\btheta})$ and $\eta_2=0.99\cdot(2/\min_{\btheta}\lambda_{\max}\rb{\btheta})$ per seed. Note that $\min_{\btheta}\lambda_{\max}\rb{\btheta}$, i.e., the sharpness of the flattest implementation, varies for each random dataset, i.e., for each seed. To obtain the sharpness of the flattest implementation we use projected GD with Dykstra's projection algorithm. 

\subsection{Implementation Details for Realistic Neural Networks (Figure \ref{Fig: sharpness, ps and train loss on three architectures})
\label{Sec: Implementation details for the modern Architectures}}

Generally, we followed a similar setting to \citet{cohen2021gradient}, with some necessary adjustments due to the computational cost of calculating the GFS sharpness repeatedly.

\textbf{Dataset.} The dataset consists of the first $1000$ samples from CIFAR-10, except for the ResNet experiment in which the dataset consists of a subset of the $100$ first samples. 
We used the same preprocessing as \citet{cohen2021gradient}. That is, we centered each channel, and then normalized the channel by dividing it by the standard deviation
(where both the mean and standard deviation were computed over the full CIFAR-10 dataset).

\textbf{Architectures.} We experiment with five architectures: 
three layers fully connected networks with hardtanh, tanh, and ReLU activation function (the same architecture as in \citet{cohen2021gradient}), VGG-11 with batch normalization (implemented here: \url{https://github.com/chengyangfu/pytorch-vgg-cifar10/blob/master/vgg.py}), and ResNet20 (implemented here: \url{https://github.com/hongyi-zhang/Fixup}).

\textbf{Loss.} The MSE loss was used in the experiments.

\textbf{GFS sharpness computation.} We calculated the GFS sharpness every $100$ iterations since it was too computationally expensive to do this calculation at each iteration, as in the squared regression model experiments. In order to calculate the GFS sharpness at a given iteration of GD,  we used the  Runge-Kutta RK4 algorithm \citep{press2007numerical} to numerically approximate the GF trajectory initialized at the current iterate of GD. In other words, given GD iterate at time $t$: $\vect{w}^{\rb{t}}$, we used Runge-Kutta until reaching training loss below $10^{-5}$ to compute $\Pgf{\vect{w}^{\rb{t}}}$. Then, we calculated the maximal eigenvalue of $\Pgf{\vect{w}^{\rb{t}}}$ using Lanczos algorithm to obtain the GFS sharpness. Similar to \citet{cohen2021gradient} we used $1/\lambda_{\max}(\wt)$ as the step size for Runge-Kutta algorithm.

\subsection{Additional Experiments on Realistic Architectures \label{Sec: Additional experiments on modern architectures}}
In Figure \ref{Fig: sharpness and ps on more architectures}, we experiment in the same setting as in Figure \ref{Fig: sharpness, ps and train loss on three architectures} for two additional activation functions. We observe the same qualitative behavior as in Figure \ref{Fig: sharpness, ps and train loss on three architectures}.  

\begin{figure*}[h]
	\centering
    \subfloat[FC-tanh]{\includegraphics[height=3.9cm]{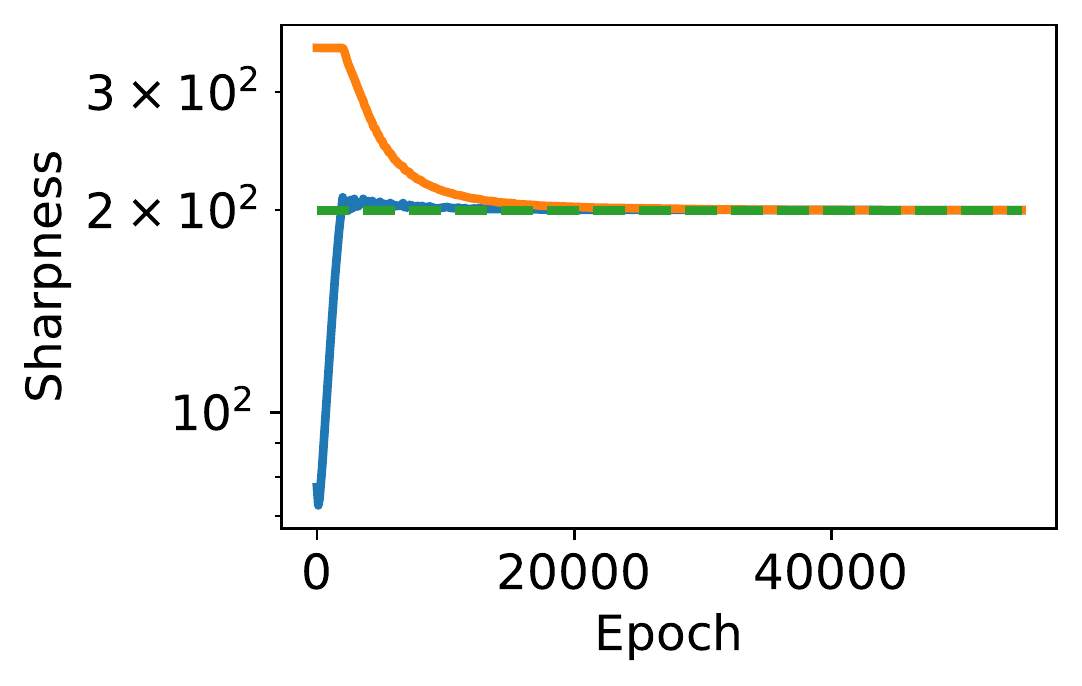}}
	\subfloat[FC-ReLU]{\includegraphics[height=3.9cm]{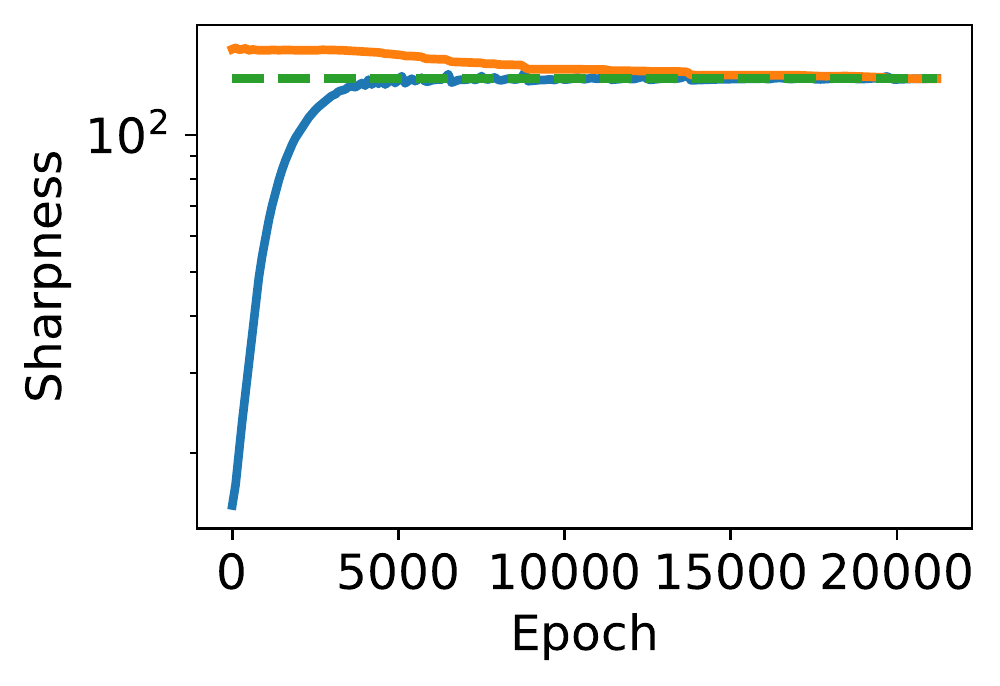}\label{subfig: relu GFS sharpness}}
	\subfloat{\includegraphics[width=0.13\textwidth, trim=0 -2.3cm 0cm 0, clip]{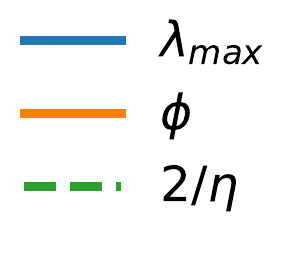}}
 \caption{\textbf{The GFS sharpness decrease monotonically to $2/\eta$ for common neural network architectures.} On two additional architectures, we run GD on a subset of Cifar10 and calculated the GFS sharpness every 100 iterations using Runge-Kutta algorithm. We observe the same qualitative behavior as in Figure \ref{Fig: sharpness, ps and train loss on three architectures}: the sharpness (blue line) non-monotonically rises to the EoS, i.e., to $2/\eta$ (green dashed line) and the GFS sharpness (orange line) decrease monotonically to the same value.\label{Fig: sharpness and ps on more architectures}}
\end{figure*}

In Figure \ref{Fig: sharpness and ps on svhn dataset} we repeat the experiments on a subset of $1000$ samples from the SVHN dataset. Again, we observe the same qualitative behavior as in Figure \ref{Fig: sharpness, ps and train loss on three architectures}.
\begin{figure*}[h]
	\centering
	\subfloat[FC-ReLU-svhn]{\hspace{0.8cm}\includegraphics[height=3.9cm]{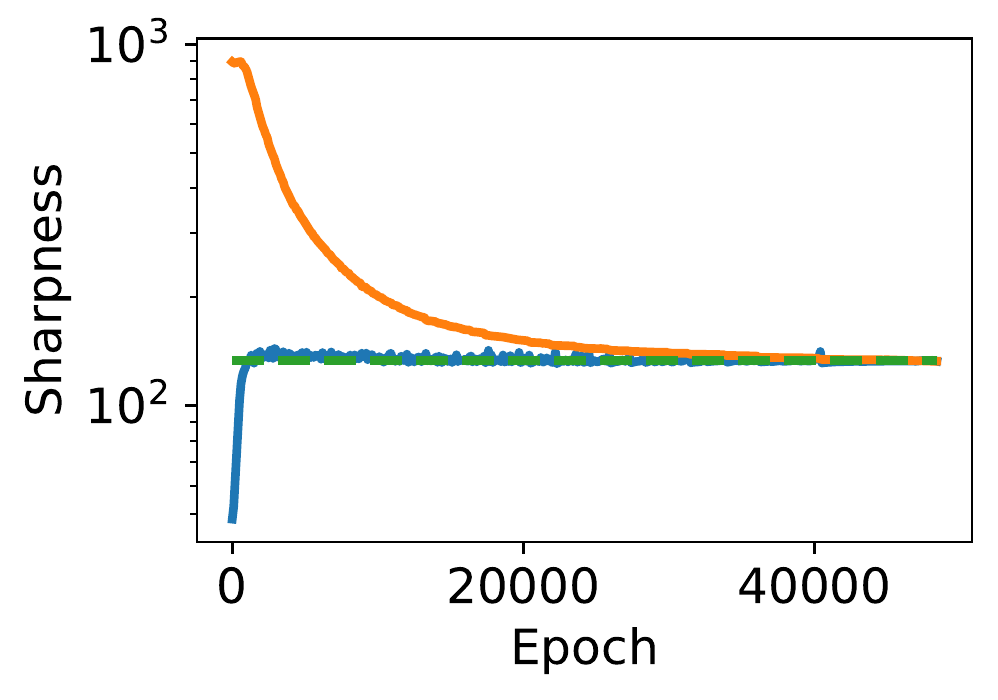}}
	\subfloat[vgg11-bn-svhn]{\includegraphics[height=3.9cm]{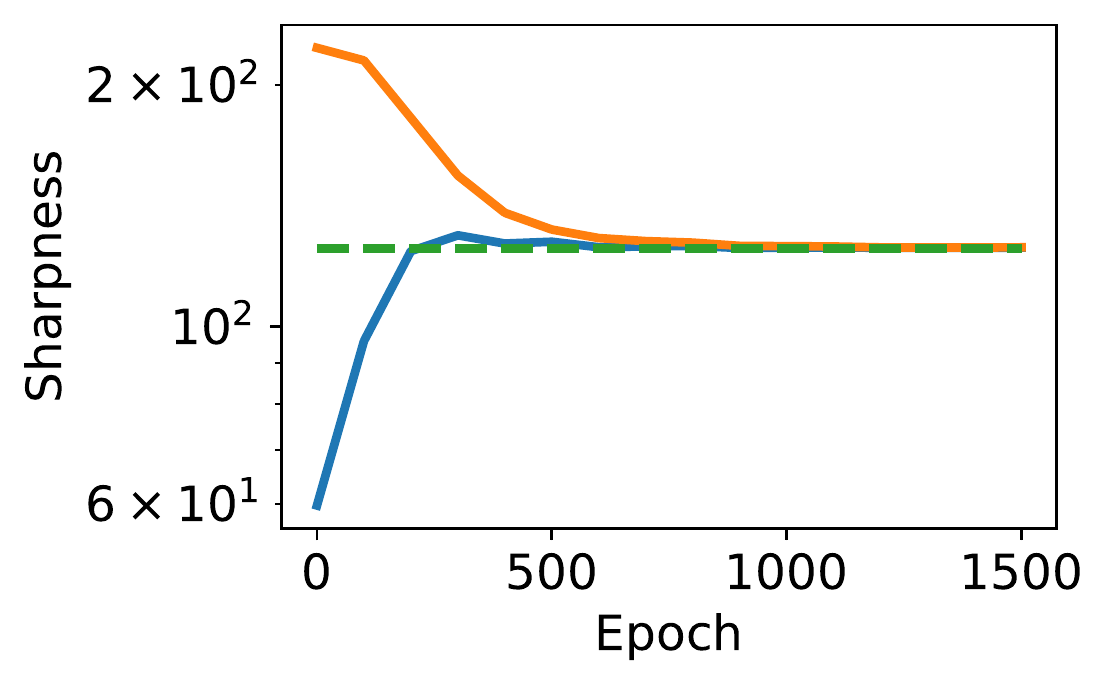}}
	\subfloat{\includegraphics[width=0.13\textwidth, trim=0 -2.3cm 0cm 0, clip]{Figures/legend.pdf}}
    \caption{\textbf{The GFS sharpness decreases monotonically to $2/\eta$ for common neural network architectures on an additional dataset.} Using a subset of the SVHN dataset, we run GD  and calculated the GFS sharpness every 100 iterations using Runge-Kutta algorithm. We observe the same qualitative behavior as in Figure \ref{Fig: sharpness, ps and train loss on three architectures}: the sharpness (blue line) non-monotonically rises to the EoS, i.e., to $2/\eta$ (green dashed line) and the GFS sharpness (orange line) decrease monotonically to the same value.\label{Fig: sharpness and ps on svhn dataset}}
\end{figure*}

\section{Calculation of the GFS and its Sharpness for the Squared Regression Model\label{sec: squared regression model calc}}

\paragraph{Setting.}
Given a training set $\cb{\vect{x}_n,y_n}_{n=1}^N$, we define the square loss
\[
\mathcal{L}\rb{\btheta}=\frac{1}{2N}\sum_{n=1}^N \rb{y_n-f_{\btheta}\rb{\vect{x}_n}}^2
\]
where $f_{\btheta}\rbm{\xx} = \ip{\vect{u}_+^2 - \vect{u}_-^2}{\xx}=\ip{\bbeta}{\xx}.$

In this section, we explain how we obtain the GFS sharpness for this setting, i.e.,  the squared regression model. There are three main steps:
\begin{enumerate}
	\item Obtaining the optimization problem for finding $\bbeta_{GF}\rb{\vect{w}_0}$ (the linear predictor associated with
	the interpolating solution obtained by GF initialized at $\vect{w}_0$) using the result from \cite{azulay2021implicit} Appendix A.
	\item Obtaining $\Pgf{\vect{w}_0}^2$ from $\bbeta_{GF}\rb{\vect{w}_0}$.
	\item Obtaining the GFS sharpness.
\end{enumerate}

\paragraph{Obtaining the optimization problem for finding $\bbeta_{GF}\rb{\vect{w}_0}$.}

Following the calculations in \cite{azulay2021implicit} Appendix A, and substituting $\vect{v}_+=\vect{u}_+$ and $\vect{v}_-=\vect{u}_-$, we obtain that
\[
\bbeta_{GF}\rb{\vect{w}_0} = \argmin_{\vect{\beta}} Q_{\vect{w}_0}\rb{\vect{\beta}} \text{ s.t. } \vect{X\beta}=\vect{y}\,,
\]
where $\vect{X}$ and $\vect{y}$ are the training data and labels respectively, 
\begin{align*}
	Q_{\vect{w}_0}\rb{\vect{\beta}} =& \sum_{i=1}^d q_{k_i}\rb{\beta_i}\,,
\end{align*} for $ k_i=16w_{0,i}^{2}w_{0,i+d}^{2}$ and 
\[
q''_{k_i}\rb{\beta_i}=\frac{1}{\sqrt{k_i+4\beta_i^2}}
\,.
\]
Integrating the above, and using the constraint $q'_{k_i}\rb{\beta_{0,i}} = 0$ where $\beta_{0,i}=w_{0,i}^2-w_{0,i+d}^2$ we get:
\begin{align*}
	q_{k_{i}}'\left(\beta_{i}\right)	=&\frac{1}{4}\left[\log\left(\frac{2\beta_{i}}{\sqrt{k_{i}+4\beta_{i}^{2}}}+1\right)-\log\left(1-\frac{2\beta_{i}}{\sqrt{k_{i}+4\beta_{i}^{2}}}\right)\right]\\
	&-\frac{1}{4}\left[\log\left(\frac{2\beta_{0,i}}{\sqrt{k_{i}+4\beta_{0,i}^2}}+1\right)-\log\left(1-\frac{2\beta_{0,i}}{\sqrt{k_{i}+4\beta_{0,i}^2}}\right)\right]
\end{align*}
Simplifying the above we obtain:
\begin{align*}
	q_{k_{i}}'\left(\beta_{i}\right)	&=\frac{1}{4}\left[\log\left(\frac{2\beta_{i}+\sqrt{k_{i}+4\beta_{i}^{2}}}{\sqrt{k_{i}+4\beta_{i}^{2}}-2\beta_{i}}\right)-\log\left(\frac{2\beta_{0,i}+\sqrt{k_{i}+4\beta_{0,i}^2}}{\sqrt{k_{i}+4\beta_{0,i}^2}-2\beta_{0,i}}\right)\right]\\
	&=\frac{1}{2}\left[\log\left(2\beta_{i}+\sqrt{k_{i}+4\beta_{i}^{2}}\right)-\log\left(2\beta_{0,i}+\sqrt{k_{i}+4\beta_{0,i}^2}\right)\right]\\
	&=\frac{1}{2}\left[\log\left(\frac{2\beta_{i}}{\sqrt{k_{i}}}+\sqrt{1+\frac{4\beta_{i}^{2}}{k_{i}}}\right)-\log\left(\frac{2\beta_{0,i}}{\sqrt{k_{i}}}+\sqrt{1+\frac{4\beta_{0,i}^2}{k_{i}}}\right)\right]\\
	&=\frac{1}{2}\left[\arcsinh\left(\frac{2\beta_{i}}{\sqrt{k_{i}}}\right)-\arcsinh\left(\frac{2\beta_{0,i}}{\sqrt{k_{i}}}\right)\right]
\end{align*}
Finally, we integrate again and obtain 
\begin{align*}
	q_{k_{i}}\left(\beta_{i}\right)&=\int q_{k_{i}}'\left(\beta_{i}\right)=\frac{\sqrt{k_{i}}}{4}\left[1-\sqrt{1+\frac{4\beta_{i}^{2}}{k_{i}}}+\frac{2\beta_{i}}{\sqrt{k_{i}}}\arcsinh\left(\frac{2\beta_{i}}{\sqrt{k_{i}}}\right)\right]-\frac{1}{2}\arcsinh\left(\frac{2\beta_{0,i}}{\sqrt{k_{i}}}\right)\beta_{i}\,.
\end{align*}
Substituting $k_i=16w_{0,i}^{2}w_{0,i+d}^{2}$ and $\beta_{0,i}=w_{0,i}^2-w_{0,i+d}^2$ we obtain the desired result, i.e.,
\begin{align*}
	Q_{\vect{w}_0}\rb{\vect{\beta}} =& \sum_{i=1}^d \abs{w_{0,i}w_{0,i+d}} q\rb{\frac{\beta_i}{2\abs{w_{0,i}w_{0,i+d}}}}\\
	&-\frac{1}{2}\vect{\beta}^\top\arcsinh\rb{\frac{w_{0,i}^2-w_{0,i+d}^2}{2\abs{w_{0,i}w_{0,i+d}}}}\,,
\end{align*} where 
\[q\rb{z}=1-\sqrt{1+z^2}+z\arcsinh\rb{z}\,.\]

\paragraph{Obtaining $\Pgf{\vect{w}_0}^2$ from $\bbeta_{GF}\rb{\vect{w}_0}$.}

Using the squared regression model definition and Eq. $(17)$ in \cite{azulay2021implicit} (the preserved quantity for this model) we obtain
\begin{align*}
	\begin{cases}
		u_{+,i}u_{-,i} = u_{+,0,i}u_{-,0,i}\\
		\beta_{i} =  u_{+,i}^2 - u_{-,i}^2\,,
	\end{cases}
\end{align*}
where, from definition, $u_{+,0,i}=w_{0,i}\,,u_{-,0,i}=w_{0,i+d}$.
From these two equations, we get
\begin{align*}
	u_{+,i}^2\rb{\beta_i} &= \frac{\beta_i+\sqrt{\beta_i^2+4u_{+,0,i}^2u_{-,0,i}^2}}{2}\\
	u_{-,i}^2\rb{\beta_i} &= \frac{-\beta_i+\sqrt{\beta_i^2+4u_{+,0,i}^2u_{-,0,i}^2}}{2}
\end{align*}
This immediately gives us $\Pgf{\vect{w}_0}^2$ since from definition $\Pgf{\vect{w}_0}^2=\begin{bmatrix}\vect{u}_{+}^2\rb{\beta_{GF}} & \vect{u}_{-}^2\rb{\beta_{GF}}\end{bmatrix}$.

\paragraph{Obtaining the GFS sharpness.}

We denote $\btheta=\begin{bmatrix}
	\vect{u}_+^\top & \vect{u}_-^\top
\end{bmatrix}^\top$.
The Hessian matrix for the squared regression model is
\[
\nabla^{2}\mathcal{L}\left(\btheta\right)=\frac{1}{N}\sum_{n=1}^N\nabla_{\btheta}f_{\btheta}\rb{\vect{x}_n} \nabla_{\btheta}f_{\btheta}\rb{\vect{x}_n}^\top + \frac{1}{N}\sum_{n=1}^N\rb{f_{\btheta}\rb{\vect{x}_n}-y_n}\nabla_{\btheta}^2 f_{\btheta}\rb{\vect{x}_n}\,,
\]
where
\begin{align*}
	\nabla_{\btheta}f_{\btheta}\rb{\vect{x}_n} &= 2\btheta\circ \begin{bmatrix}
		\vect{x}_n\\-\vect{x}_n
	\end{bmatrix}\,,\\
	\nabla_{\btheta}^2 f_{\btheta}\rb{\vect{x}_n} & = 2\mathrm{diag}\rb{\begin{bmatrix}
			\vect{x}_n\\-\vect{x}_n
	\end{bmatrix}}
\end{align*}
and $\circ$ denotes element-wise multiplication. Thus, to obtain the GFS sharpness all we need to do is substitute $\btheta=\abs{\Pgf{\vect{w}_0}}$ and calculate the maximal eigenvalue. (Note that the sign of the elements of $\Pgf{\vect{w}_0}$ does not affect the maximal eigenvalue.)

\section{Properties of GD in Scalar Networks}

\subsection{Gradient and Hessian Calculation}
\label{sec:gradient and Hessian}

The partial derivative of the loss defined in \cref{eq:loss function} is
\begin{equation*}
	\begin{aligned}
		\frac{ \partial \loss\rb{\vect{w}} }{\partial w_i} &= \rb{ \pi\rb{\vect{w}} - 1} \frac{\partial \pi\rb{\vect{w}}}{\partial w_i}\\
		&= \rb{ \pi\rb{\vect{w}} - 1} \prod_{k\in[D]-\{i\}}w_k\\
		&= \rb{ \pi\rb{\vect{w}} - 1} \frac{\pi\rb{\vect{w}}}{w_i}
		\,.
	\end{aligned}
\end{equation*}
Therefore, the gradient is
\begin{equation}
	\begin{aligned}
		\label{eq:gradient}
		\nabla\loss\rb{\vect{w}} = \rb{ \pi\rb{\vect{w}} - 1} \pi\rb{\vect{w}} \vect{w}^{-1}
		\,.
	\end{aligned}
\end{equation}

The second order partial derivative, if $i\neq j$, is
\begin{equation}
	\begin{aligned}
		\label{eq:loss partial derivative 2 ij}
		\frac{ \partial^2 \loss\rb{\vect{w}} }{\partial w_i \partial w_j}
		&= \frac{\pi^2\rb{\vect{w}}}{w_i w_j} + \rb{ \pi\rb{\vect{w}} - 1} \frac{\pi\rb{\vect{w}}}{w_i w_j}
		\,,
	\end{aligned}
\end{equation}
and if $i = j$, it is
\begin{equation}
	\begin{aligned}
		\label{eq:loss partial derivative 3 ii}
		\frac{ \partial^2 \loss\rb{\vect{w}} }{\partial w_i^2}
		&= \frac{\pi^2\rb{\vect{w}}}{w_i^2}
		\,.
	\end{aligned}
\end{equation}
Therefore, if weight $\vect{w}$ is an optimum, i.e. $\pi\rb{\vect{w}} = 1$ then the Hessian is
\begin{equation}
	\begin{aligned}
		\label{eq:Hessian at opt}
		\nabla^2 \loss \rb{\vect{w}}
		&= \pi^2 \rb{\vect{w}} \vect{w}^{-1} \rb{\vect{w}^{-1}}^{T}
		\,.
	\end{aligned}
\end{equation}
Thus, if weight $\vect{w}$ is an optimum, the largest eigenvalue of the Hessian is
\begin{equation}
	\begin{aligned}
		\label{eq:sharpness at opt}
		\lambda_{\max}
		&= \pi^2 \rb{\vect{w}} \rb{\vect{w}^{-1}}^{T} \vect{w}^{-1}\\
		&= s_1(\vect{w})
		\,.
	\end{aligned}
\end{equation}

\subsection{The Dynamics of GD}

The exact update rule of gradient decent with a fixed step size $\eta \in \R_{+}$ on the loss function described in \cref{eq:loss function}, using the gradient in \cref{eq:gradient}, is
\begin{align}
	\label{eq:GD exact update rule}
	w^{(t+1)}_i = w^{(t)}_i - \eta \left( \prod_{j=1}^{D}w^{(t)}_j - 1 \right)\prod_{j\in [D] - \{i\} }w^{(t)}_j
	\,,
\end{align}

We can separate the gradient decent dynamics into two separate dynamics, a dynamic of the weights product and a dynamic of the balances.

To this end, for every $m\in \sqb{D} \cup \left\{ 0 \right\}$, define
\begin{align*}
	s_m(\vect{w}) = \sum_{\substack{I=\text{ subset of } \\ D-m \text{ different} \\ \text{indices from }[D]}}\prod_{i\in I} {w_i^2}
	\,.
\end{align*}
The dynamics of the product of the weight is
\begin{equation}
	\begin{aligned}
		\label{dyn:product of weights}
		\pi^{(t+1)} &= \pi^{(t)} + \sum_{m=1}^{D} \eta^m (1-\pi^{(t)})^{m} {\pi^{(t)}}^{m-1} s_m(\vect{w}^{(t)})
		\,,
	\end{aligned}
\end{equation}
where $\pi^{(t)} = \pi(w^{(t)})$, see Section \ref{Sec: Dynamics calculation} for full calculation.

We define the balances as
\begin{definition}[Balances]
	\label{def:balances appendix}
	The balances $\vect{b}\in\Rd[D\times D]$, of weight $\vect{w}$, are define as
	\begin{align*}
		b_{i,j} \triangleq {w_i}^2 - {w_j}^2 \,, ~\forall i,j \in \sqb{D}
		\,.
	\end{align*}
\end{definition}
The dynamics of the balances, for each $i,j \in[D], i\neq j$, is
\begin{equation}
	\begin{aligned}
		\label{dyn:balances}
		b_{i,j}^{\rb{t+1}}
		&= b_{i,j}^{\rb{t}}\rb{ 1 - \eta^2 \left( \pi^{(t)} - 1 \right)^2 \frac{{\pi^{(t)}}^2}{{w^{(t)}_i}^2 {w^{(t)}_j}^2} }
		\,,
	\end{aligned}
\end{equation}
where $b_{i,j}^{\rb{t}}$ are the balances of $\wt$, see Section \ref{Sec: Dynamics calculation} for full calculation.
Note that while $\frac{{\pi^{(t)}}^2}{{w^{(t)}_i}^2 {w^{(t)}_j}^2}$ is not define if either $w^{(t)}_i$ or $w^{(t)}_j$ are equal to $0$,
the limit of $\frac{{\pi^{(t)}}^2}{{w^{(t)}_i}^2 {w^{(t)}_j}^2}$ when either of $w^{(t)}_i$ or $w^{(t)}_j$ approaches $0$ does exist and equal to $\prod_{k\in [D] - \{i,j\} }{w^{(t)}_k}^2$.

The balances and the product of the weight are sufficient to find the value of ${\vect{w}^{(t)}}^2$, which is needed for calculating the update step in both of the dynamics. Therefore, we can indirectly calculate \cref{eq:GD exact update rule} using \cref{dyn:product of weights} and \cref{dyn:balances}.

\subsubsection{Dynamics calculation}\label{Sec: Dynamics calculation}
The dynamic of the product of the weight, written in \cref{dyn:product of weights}, is found by
\begin{align*}
	\pi^{(t+1)}
	&= \prod_{i=1}^{D} w^{(t+1)}_i\\
	&= \prod_{i=1}^{D} \rb{ w^{(t)}_i - \eta \left( \pi^{(t)} - 1 \right) \frac{\pi^{(t)}}{w^{(t)}_i} }\\
	&= \pi^{(t)} \prod_{i=1}^{D} \rb{ 1 - \eta \left( \pi^{(t)} - 1 \right) \frac{\pi^{(t)}}{{w^{(t)}_i}^2} }\\
	&= \pi^{(t)} + \sum_{m=1}^{D} \eta^m (1-\pi^{(t)})^{m} {\pi^{(t)}}^{m-1} s_m(\vect{w}^{(t)})
	\,.
\end{align*}

The dynamic of the balances, written in \cref{dyn:balances}, is found by
\begin{align*}
	b_{i,j}^{\rb{t+1}}
	&={w^{\rb{t+1}}_i}^2 - {w^{\rb{t+1}}_j}^2\\
	&= \rb{ w^{(t)}_i - \eta \left( \pi^{(t)} - 1 \right) \frac{\pi^{(t)}}{w^{(t)}_i} }^2 - \rb{  w^{(t)}_j - \eta \left( \pi^{(t)} - 1 \right) \frac{\pi^{(t)}}{w^{(t)}_j} }^2\\
	&= {w^{\rb{t}}_i}^2 - 2\eta \left( \pi^{(t)} - 1 \right) \pi^{(t)} + \eta^2 \left( \pi^{(t)} - 1 \right)^2 \frac{{\pi^{(t)}}^2}{{w^{\rb{t}}_i}^2}\\
	&\qquad - {w^{\rb{t}}_j}^2 + 2\eta \left( \pi^{(t)} - 1 \right) \pi^{(t)} - \eta^2 \left( \pi^{(t)} - 1 \right)^2 \frac{{\pi^{(t)}}^2}{{w^{\rb{t}}_j}^2}\\
	&= {w^{\rb{t}}_i}^2  - {w^{\rb{t}}_j}^2 + \eta^2 \left( \pi^{(t)} - 1 \right)^2{\pi^{(t)}}^2 \rb{ \frac{1}{{w^{\rb{t}}_i}^2} - \frac{1}{{w^{\rb{t}}_j}^2} }
\end{align*}
\begin{align*}
	&= \rb{ {w^{\rb{t}}_i}^2 - {w^{\rb{t}}_j}^2 } \rb{ 1 - \eta^2 \left( \pi^{(t)} - 1 \right)^2 \frac{{\pi^{(t)}}^2}{{w^{(t)}_i}^2 {w^{(t)}_j}^2} }\\
	&= b_{i,j}^{\rb{t}}\rb{ 1 - \eta^2 \left( \pi^{(t)} - 1 \right)^2 \frac{{\pi^{(t)}}^2}{{w^{(t)}_i}^2 {w^{(t)}_j}^2} }
	\,.
\end{align*}

\subsubsection{Equivalence of weights}
\label{sec:Equivalence of weights}

Note that in order to calculate the dynamics of the weights (\cref{dyn:product of weights})) and balances (\cref{dyn:balances}), we do not need to know the individual sign of every element of $\vect{w}$.
Instead, only the sign of $\pi\rb{w}$ and the values of the elements of $\vect{w}^2$ are needed to calculate the dynamics.
Therefore, if $\pi\rb{w} \ge 0$, then $\vect{w}$ will have the same GD dynamics as $\abs{\vect{w}}$.

In addition, if $\vect{w}$ is a minimum then $\abs{\vect{w}}$ is also a minimum, and \cref{eq:sharpness at opt main text} implies it has the shame sharpness.

\section{Majorization and Schur-Convexity}
\label{sec:majorization_and_schur-convexity}
Definitions, lemmas and theorems are taken from \cite{marshall2011inequalities}.

We first define Schur-convexity.
\begin{definition}[Majorization]
	For vectors $\vect{u},\vect{v}\in \R^n$ we say that $\vect{v}$ majorizes $\vect{u}$, written as $\vect{u} \maj \vect{v}$, if
	\begin{align*}
		& \sum_{i=1}^{D} u_{\sqb{i}} = \sum_{i=1}^{D} v_{\sqb{i}} ~ \text{and}\\
		& \sum_{i=1}^{k} u_{\sqb{i}} \le \sum_{i=1}^{k} v_{\sqb{i}} ~ \text{, for every } k\in[D] \,.
	\end{align*}
\end{definition}
\begin{definition}[Schur-convexity]
	\label{def:Schur-convexity}
	A function $f:\R^n\mapsto\R$ is called Schur-convex if for every vectors $\vect{u},\vect{v}\in \R^n$ such that $\vect{u} \maj \vect{v}$, then $f(\vect{u})\le f(\vect{v})$.
\end{definition}

A symmetric function is defined as
\begin{definition}[Symmetric function]
	\label{def:symmetric function}
	A function $f:\R^n\mapsto\R$ is symmetric if for every vector $\vect{u}\in \R^n$ and for every permutation $\vect{u}'$ of vector $\vect{u}$ then $f\rb{\vect{u}'}=f\rb{\vect{u}}$.
\end{definition}

In our derivation, we will use the following useful theorems regarding Schur-convex functions.
\begin{theorem}
	\label{thm:symmetric and convex}
	If a function $f:\R^n\mapsto\R$ is symmetric and convex, then $f$ is Schur-convex.
\end{theorem}

\begin{theorem}
	\label{thm:epsilon to convex}
	Let $\A\subseteq\R^n$ be a set with the property
	\begin{equation*}
		\vect{v}\in\A \text{,}\, \vect{u}\in\R^n ~ \text{and} ~ \vect{u} \maj \vect{v} \qquad\text{implies}\quad \vect{u}\in\A \,.
	\end{equation*}
	A continuous function $f:\A\mapsto\R$ is Schur-convex if $f$ is symmetric and for every vector $\vect{v}\in \A$
	and for every $k\in\sqb{n-1}$ then
	\[ f\rb{ \vect{v}_{\sqb{1}}, \dots, \vect{v}_{\sqb{k-1}}, \vect{v}_{\sqb{k}}+\varepsilon, \vect{v}_{\sqb{k+1}} - \varepsilon, \vect{v}_{\sqb{k+2}}, \dots, \vect{v}_{\sqb{n}} } \]
	is raising in $\varepsilon$ over the region
	\[ 0 \le \varepsilon \le \min\left\{ \vect{v}_{\sqb{k-1}} - \vect{v}_{\sqb{k}}, \vect{v}_{\sqb{k+1}} - \vect{v}_{\sqb{k+2}} \right\} \,, \]
	where $\vect{v}_{\sqb{0}} \triangleq \infty$ and $\vect{v}_{\sqb{n+1}} \triangleq -\infty$.
\end{theorem}

Finally, we present a lemma that can be used to apply the previous theorems to what we called log-Schur-convex functions (Definition \ref{def:log Schur-convexity}).
\begin{lemma}
	\label{lem:Schur-convex to log-Schur-convex}
	Let $\A \subseteq \R_{+}^n$.
	For a function $f:\A\mapsto\R$ if $f\rb{\e^{\vect{x}}}$ is Schur-convex on $\log\A$, where $\e^{\vect{x}}$ is preformed element-wise and $\log\A$ is preformed element-wise on each element of $\A$, then for every vectors $\vect{u},\vect{v}\in \R_{+}^n$ such that $\vect{u} \logmaj \vect{v}$ we have that $f(\vect{u})\le f(\vect{v})$, i.e., $f$ is log-Schur-convex.
\end{lemma}
Note that the term log-Schur-convex function is not often used in literature. Instead, \cref{lem:Schur-convex to log-Schur-convex} is used together with Schur-convex functions. However, for convenience, in our derivation, we decided to define and use log-Schur-convexity directly. 

\section{Proof of Lemmas \ref{lem:gd balance}, \ref{lem:balance to majorization}, and \ref{lem:sc functions}}

\subsection{Proof of Lemma \ref{lem:gd balance}}\label{Sec: proof of gd balance lemma}
The proof of \cref{lem:gd balance} relies on two auxiliary lemmas:
\begin{lemma}
	\label{lem:gd balance sum 1}
	For $D\ge2$, $\eta>0$ and $\vect{w}\in\Rd$, if $\ps{\vect{w}} \le \frac{2\sqrt{2}}{\eta}$ and $\pi\rb{\vect{w}}\in\iv{0}{1}$ then
	\begin{align*}
		\sum_{i=1}^{\min\{2,D-1\}}\frac{\eta^2 \left(\pi(\vect{w}) - 1\right)^2 \pi^2(\vect{w})}{w^2_{[D-i]} w^2_{[D]}} \le 2
		\,.
	\end{align*}
\end{lemma}

\begin{lemma}
	\label{lem:gd balance sum 2}
	If $\vect{w}\in\SG$ then
	\begin{align*}
		\sum_{i=1}^{\min\{2,D-1\}}\frac{\eta^2 \left(\pi(\vect{w}) - 1\right)^2 \pi^2(\vect{w})}{w^2_{[D-i]} w^2_{[D]}} \le 2
		\,.
	\end{align*}
\end{lemma}

These Lemmas are proved in Sections \ref{proof lem:gd balance sum 1} and \ref{proof lem:gd balance sum 2} respectively. 

Using these lemmas, we can prove \cref{lem:gd balance}.
\begin{proof}
	Let $\vect{w}\in\SG$.
	We assume without loss of generality that $\vect{w}^2$ is sorted, i.e. that $w_i^2 = w^2_{\sqb{i}}$ for all $i\in\sqb{D}$.
	
	From \cref{dyn:balances} then for all $i,j\in\sqb{D}$, such that $i\neq j$,
	\begin{align}
		\label{eq:balance dyn in proof}
		\gdb{\vect{w}}_i^2 - \gdb{\vect{w}}_{j}^2 = \rb{w_i^2 - w_{j}^2} \rb{ 1 - \eta^2 \left( \pi\rb{\vect{w}} - 1 \right)^2 \frac{\pi^2\rb{\vect{w}}}{w_i^2 w_j^2} }
	\end{align}

	From \cref{lem:gd balance sum 2} we get that
	\begin{align*}
		\sum_{i=1}^{\min\{2,D-1\}}\frac{\eta^2 \left(\pi(\vect{w}) - 1\right)^2 \pi^2(\vect{w})}{w^2_{D-i} w^2_{D}} \le 2
		\,.
	\end{align*}
	Therefore, as $\vect{w}^2$ is sorted,
	\begin{align*}
		&\eta^2 \left(\pi(\vect{w}) - 1\right)^2 \frac{\pi^2(\vect{w})}{w^2_{D-1} w^2_{D}} \le 2	 \text{ and}\\
		&\eta^2 \left(\pi(\vect{w}) - 1\right)^2 \frac{\pi^2(\vect{w})}{w^2_{j} w^2_{i}} \le 1\,, \forall i\in\sqb{D}, j\in\sqb{D-2}, j<i
		\,.
	\end{align*}
	
	Therefore, using \cref{eq:balance dyn in proof}, for all $i\in\sqb{D-2}$ we obtain
	\begin{align}\label{Eq: ordered gradeint steps D-2}
		0 \le \gdb{\vect{w}}_i^2 - \gdb{\vect{w}}_{i+1}^2 \le w_i^2 - w_{i+1}^2\,.
	\end{align}
	Similarly,
	\begin{align}\label{Eq: ordered gradeint steps D-2 and D}
		0 \le \gdb{\vect{w}}_{D-2}^2 - \gdb{\vect{w}}_{D}^2 \le w_{D-2}^2 - w_{D}^2
		\,,
	\end{align}
	and
	\begin{align}\label{Eq: ordered gradeint steps D-1}
		w_{D}^2 - w_{D-1}^2 \le \gdb{\vect{w}}_{D-1}^2 - \gdb{\vect{w}}_{D}^2 \le w_{D-1}^2 - w_{D}^2
		\,.
	\end{align}
    Now, to show that $\gd{\vect{w}} \ble \vect{w}$ we divide into two cases:
	\begin{enumerate}
		\item If $\gdb{\vect{w}}_{D-1}^2 - \gdb{\vect{w}}_{D}^2 \ge 0$ then, using Eq. \ref{Eq: ordered gradeint steps D-2} we obtain
		\begin{align*}
			\gdb{\vect{w}}_{1}^2 \ge \gdb{\vect{w}}_{2}^2 \ge \dots \ge \gdb{\vect{w}}_{D}^2
			\,.
		\end{align*}
		Therefore, from the last equation and equations \ref{Eq: ordered gradeint steps D-2} and \ref{Eq: ordered gradeint steps D-1}, for any $i\in\sqb{D-1}$
		\begin{align*}
			\gdb{\vect{w}}_{\sqb{i}}^2 - \gdb{\vect{w}}_{\sqb{i+1}}^2 \le w_{\sqb{i}}^2 - w_{\sqb{i+1}}^2
			\,.
		\end{align*}
		Therefore, from Definition \ref{Def: Balance quasi-order}, $\gd{\vect{w}} \ble \vect{w}$.
		
		\item If $\gdb{\vect{w}}_{D-1}^2 - \gdb{\vect{w}}_{D}^2 \le 0$ then, using Eqs. \ref{Eq: ordered gradeint steps D-2} and \ref{Eq: ordered gradeint steps D-2 and D} we obtain
		\begin{align*}
			\gdb{\vect{w}}_{1}^2 \ge \gdb{\vect{w}}_{2}^2 \ge \dots \ge \gdb{\vect{w}}_{D-2}^2 \ge \gdb{\vect{w}}_{D}^2 \ge \gdb{\vect{w}}_{D-1}^2
			\,.
		\end{align*}
		Therefore, from the last equation and Eq. \ref{Eq: ordered gradeint steps D-2}, for any $i\in\sqb{D-3}$ we obtain
		\begin{align*}
			\gdb{\vect{w}}_{\sqb{i}}^2 - \gdb{\vect{w}}_{\sqb{i+1}}^2 \le w_{\sqb{i}}^2 - w_{\sqb{i+1}}^2
			\,,
		\end{align*}
		and
		\begin{align*}
			\gdb{\vect{w}}_{\sqb{D-2}}^2 - \gdb{\vect{w}}_{\sqb{D-1}}^2
			&= \gdb{\vect{w}}_{D-2}^2 - \gdb{\vect{w}}_{D}^2\\
			&\le  \gdb{\vect{w}}_{D-2}^2 - \gdb{\vect{w}}_{D-1}^2\\
			&\le w_{D-2}^2 - w_{D-1}^2\\
			&=w_{\sqb{D-2}}^2 - w_{\sqb{D-1}}^2
			\,.
		\end{align*}
		Also, using Eq. \ref{Eq: ordered gradeint steps D-1}
		\begin{align*}
			\gdb{\vect{w}}_{\sqb{D-1}}^2 - \gdb{\vect{w}}_{\sqb{D}}^2
			&= \gdb{\vect{w}}_{D}^2 - \gdb{\vect{w}}_{D-1}^2\\
			&\le w_{D-1}^2 - w_{D}^2\\
			&=w_{\sqb{D-1}}^2 - w_{\sqb{D}}^2
			\,.
		\end{align*}
		Therefore, from Definition \ref{Def: Balance quasi-order},  $\gd{\vect{w}} \ble \vect{w}$.
	\end{enumerate}

	Overall, in both cases, we get that $\gd{\vect{w}} \ble \vect{w}$ which completes our proof.
\end{proof}

\subsubsection{Proof of Lemma \ref{lem:gd balance sum 1}}\label{proof lem:gd balance sum 1}
\begin{proof}
	Let $\vect{w}\in\R^D$ such that $\pi \rb{\vect{w}} \in \sqb{0,1}$ then
	\begin{align*}
		\sum_{i=1}^{\min\{2,D-1\}}\frac{\eta^2 \left(\pi(\vect{w}) - 1\right)^2 \pi^2(\vect{w})}{w^2_{[D-i]} w^2_{[D]}} \le
		\sum_{i=1}^{\min\{2,D-1\}}\frac{\eta^2 \pi^2(\vect{w})}{w^2_{[D-i]} w^2_{[D]}}
		\,.
	\end{align*}
	Our goal is to show that if $\ps{\vect{w}} \le \frac{2\sqrt{2}}{\eta}$ then
     \[
     \sum_{i=1}^{\min\{2,D-1\}}\frac{\eta^2 \pi^2(\vect{w})}{w^2_{[D-i]} w^2_{[D]}}\le2
     \]
     Since the GFS sharpness is constant for all the weights on the GF trajectory, we can focus on weights $\vect{x}\in E_{\iv{0}{1}}(\vect{w})$, and show that
     $\ps{\vect{x}} \le \frac{2\sqrt{2}}{\eta}$
     implies
     \[
     \sum_{i=1}^{\min\{2,D-1\}}\frac{\eta^2 \pi^2(\vect{x})}{x^2_{[D-i]} x^2_{[D]}}\le2\,.
     \]
     Note that
	\begin{align*}
		\sum_{i=1}^{\min\{2,D-1\}}\frac{\eta^2 \pi^2(\vect{x})}{x^2_{[D-i]} x^2_{[D]}}
        =
        \sum_{i=1}^{\min\{2,D-1\}}\eta^2 \prod_{j\in [D-1]-\{D-i\}} x^2_{[j]}
	\end{align*}
	receives the maximum value when $\pi\rb{\vect{x}}=1$, since recall that every $\vect{x}\in E_{\iv{0}{1}}(\vect{w})$ has the same balance.
	
	Thus, since we are only interested on upper bounding $\sum_{i=1}^{\min\{2,D-1\}}\frac{\eta^2 \pi^2(\vect{x})}{x^2_{[D-i]} x^2_{[D]}}$ we can assume that $\pi\rb{\vect{w}}=1$.
	
    Using \cref{eq:sharpness at opt}, we get that
	\begin{equation*}
		\begin{aligned}
			\ps{\vect{w}} = \sum_{i=1}^{D} \frac{1}{{w_i}^2}
			\,.
		\end{aligned}
	\end{equation*}
    This immediately implies that 
    $\frac{1}{{w_{\sqb{D}}}^2} \le \ps{\vect{w}}$ or equivalently $\exists\alpha \in \iv{0}{1}$ such that
	\begin{equation*}
		\begin{aligned}
			\frac{1}{{w_{\sqb{D}}}^2} = \alpha \ps{\vect{w}}
			\,.
		\end{aligned}
	\end{equation*}
	Therefore,
	\begin{equation*}
		\begin{aligned}
			\sum_{i=1}^{\min\{2,D-1\}}\frac{1}{w^2_{[D-i]}} \le \rb{1- \alpha}\ps{\vect{w}}
			\,.
		\end{aligned}
	\end{equation*}
	
	Substituting the last two equations into the expression we aim to bound we obtain,
	\begin{align*}
		\sum_{i=1}^{\min\{2,D-1\}}\frac{\eta^2 \pi^2(\vect{w})}{w^2_{[D-i]} w^2_{[D]}}
		&= \eta^2 \frac{1}{w^2_{[D]}} \sum_{i=1}^{\min\{2,D-1\}}\frac{1}{w^2_{[D-i]}}
		\le \eta^2 \alpha \rb{1- \alpha} \pss{\vect{w}}
         \overset{(1)}{\le} \frac{\eta^2}{4} \pss{\vect{w}}
         \overset{(2)}{\le} 2
		\,,
	\end{align*}
    where in $(1)$ we used the fact that
	$\alpha \rb{1- \alpha}$ receives its maximal value $\frac{1}{4}$ when $\alpha=\frac{1}{2}$, and in $(2)$ we used $\ps{\vect{w}} \le \frac{2\sqrt{2}}{\eta}$.
 
	Therefore, if $\ps{\vect{w}} \le \frac{2\sqrt{2}}{\eta}$ then for every weight $\vect{w}\in E_{\iv{0}{1}}(\vect{w})$ we have
	\begin{align*}
		\sum_{i=1}^{\min\{2,D-1\}}\frac{\eta^2 \left(\pi(\vect{w}) - 1\right)^2 \pi^2(\vect{w})}{w^2_{[D-i]} w^2_{[D]}} \le 2
		\,.
	\end{align*}
\end{proof}

\subsubsection{Proof of Lemma \ref{lem:gd balance sum 2}}\label{proof lem:gd balance sum 2}
\begin{proof}
	Let $\vect{w}\in\SG$.
    We divide the proof into two cases:
    \begin{enumerate}
        \item If $\pi\rb{\vect{w}}\le 1$ then by using \cref{lem:gd balance sum 1} we get that
        \begin{align*}
            \sum_{i=1}^{\min\{2,D-1\}}\frac{\eta^2 \left(\pi(\vect{w}) - 1\right)^2 \pi^2(\vect{w})}{w^2_{[D-i]} w^2_{[D]}} \le 2
        \end{align*}
    	as required.
        \item If $\pi\rb{\vect{w}} \ge 1$. We assume without loss of generality that $w\in\Rpd$ (see \cref{sec:Equivalence of weights}).
        Therefore, using \cref{lem:GD does not change sign}, we obtain that for every $i\in\sqb{D}$
	\begin{align*}
		w_i - \eta \rb{ \pi\rb{\vect{w}} - 1 }\pi\rb{\vect{w}} \frac{1}{w_i} = \gdb{\vect{w}}_i > 0
		\,.
	\end{align*}
	Since we assume  $\pi\rb{\vect{w}} \ge 1$ we get that
	\begin{align*}
		1 \ge \frac{\eta \rb{ \pi\rb{\vect{w}} - 1 }\pi\rb{\vect{w}}}{w_i^2} \ge 0.
	\end{align*}
	Therefore,
	\begin{align*}
		\sum_{i=1}^{\min\{2,D-1\}}\frac{\eta^2 \left(\pi(\vect{w}) - 1\right)^2 \pi^2(\vect{w})}{w^2_{[D-i]} w^2_{[D]}} \le 2
	\end{align*}
	in this case as well.
    \end{enumerate}

\end{proof}

\subsection{Proof of Lemma \ref{lem:balance to majorization}}\label{sec: Proof of balance to majorization lemma}
\begin{proof}
	Let $\vect{u},\vect{v}\in \R^D$ be vectors such that $\vect{u} \ble \vect{v}$ and $\prod_{i=1}^{D} u_i = \prod_{i=1}^{D} v_i$. We need to show that $\abs{\vect{u} }\logmaj \abs{\vect{v}}$.
	
	We assume in contradiction that there exist $k\in\sqb{D}$ such that
	\begin{align}
		\label{eq:majorization toward contradiction}
		&\prod_{i=1}^{k-1} u_{\sqb{i}}^2 \le \prod_{i=1}^{k-1} v_{\sqb{i}}^2 \qquad \text{and} \qquad
		\prod_{i=1}^{k} u_{\sqb{i}}^2 > \prod_{i=1}^{k} v_{\sqb{i}}^2
		\,.
	\end{align}
	This implies that $u_{\sqb{k}}^2 > v_{\sqb{k}}^2$.
	Additionally, since $\vect{u} \ble \vect{v}$ then for any $i\in\sqb{D-1}$
	\begin{align*}
		- \rb{ u_{\sqb{i+1}}^2 - u_{\sqb{i}}^2 }
		\ge - \rb{ v_{\sqb{i+1}}^2 - v_{\sqb{i}}^2 }
	\end{align*}
	Therefore, for any $m\ge k$
	\begin{align*}
		u_{\sqb{m}}^2 &= u_{\sqb{k}}^2 - \sum_{i=k}^{m - 1} \rb{ u_{\sqb{i+1}}^2 - u_{\sqb{i}}^2 }\\
		&> v_{\sqb{k}}^2 - \sum_{i=k}^{m - 1} \rb{ v_{\sqb{i+1}}^2 - v_{\sqb{i}}^2 }\\
		&= v_{\sqb{m}}^2
		\,.
	\end{align*}
	
	Combining these results with \cref{eq:majorization toward contradiction} we obtain
	\begin{align*}
		\prod_{i=1}^{D} u_{\sqb{i}}^2 > \prod_{i=1}^{D} v_{\sqb{i}}^2
		\,,
	\end{align*}
	which contradict $\prod_{i=1}^{D} u_i = \prod_{i=1}^{D} v_i$.
	Therefore, as $\prod_{i=1}^{0} u_i = 1 = \prod_{i=1}^{0} v_i$ and $\prod_{i=1}^{D} u_i = \prod_{i=1}^{D} v_i$, we get that
	\begin{align*}
		& \prod_{i=1}^{D} u_{\sqb{i}}^2 = \prod_{i=1}^{D} v_{\sqb{i}}^2 ~ \text{and}\\
		& \prod_{i=1}^{k} u_{\sqb{i}}^2 \le \prod_{i=1}^{k} v_{\sqb{i}}^2 ~ \text{, for every } k\in[D] \,.
	\end{align*}

	Taking the square root from both sides we get
	\begin{align*}
		& \prod_{i=1}^{D} \abs{\vect{u} }_{\sqb{i}} = \prod_{i=1}^{D} \abs{\vect{v}}_{\sqb{i}} ~ \text{and}\\
		& \prod_{i=1}^{k} \abs{\vect{u} }_{\sqb{i}} \le \prod_{i=1}^{k} \abs{\vect{v}}_{\sqb{i}} ~ \text{, for every } k\in[D] \,,
	\end{align*}
	i.e. $\abs{\vect{u} }\logmaj \abs{\vect{v}}$.
\end{proof}

\subsection{Proof of Lemma \ref{lem:sc functions}
}\label{Sec: proof of sc function lemma}
\subsubsection{Proof of Lemma \titleCrefi{lem:sc functions}{lem idx:sc functions 1}}
We prove \crefi{lem:sc functions}{lem idx:sc functions 1}, i.e., that the function $s_1\rb{\vect{x}}$ is log-Schur-convex in $\Rpd$.
\begin{proof}
In this proof, we first show that the function $s_1\rb{\e^{\vect{x}}}$ is Schur-convex and then use  \cref{lem:Schur-convex to log-Schur-convex} to deduce that $s_1\rb{\vect{x}}$ is  log-Schur-convex.
	\begin{align*}
		s_1\rb{\e^{\vect{x}}}
		&= \pi\rb{\e^{\vect{x}}}^2 \norm{ \e^{\vect{-x}} }^2
		= \pi\rb{\e^{\vect{x}}}^2 \sum_{i=1}^{D} \e^{-2x_i}
		\,.
	\end{align*}

	The function $\sum_{i=1}^{D} \e^{-2x_i}$ is convex, as its Hessian is a diagonal matrix with non-negative elements.
	In addition, $\sum_{i=1}^{D} \e^{-2x_i}$ is a symmetric function (Definition \ref{def:symmetric function}).
	Thus, using \cref{thm:symmetric and convex}, we get that the function $\sum_{i=1}^{D} \e^{-2x_i}$ is Schur-convex.
	
	Additionally, for any two vectors $\vect{u}, \vect{v}\in\Rd$ such that $\vect{u}\maj\vect{v}$, we get (from the majorization definition)
	\begin{align*}
		\pi\rb{e^{\vect{u}}} = \exp\rb{\sum_{i=1}^{D} u_i } = \exp\rb{\sum_{i=1}^{D} v_i } = \pi\rb{e^{\vect{v}}}
		\,.
	\end{align*}
	Therefore, because the function $\sum_{i=1}^{D} \e^{-2x_i}$ is Schur-convex,
	\begin{align*}
		s_1\rb{\e^{\vect{u}}} = \pi^2\rb{\e^{\vect{u}}} \sum_{i=1}^{D} \e^{-2u_i} \le \pi^2\rb{\e^{\vect{v}}} \sum_{i=1}^{D} \e^{-2v_i} = s_1\rb{\e^{\vect{v}}}
		\,.
	\end{align*}
	Thus, the function $s_1\rb{\e^{\vect{x}}}$ is Schur-convex.
	
	Using \cref{lem:Schur-convex to log-Schur-convex}, then the function $s_1\rb{\vect{x}}$ we get that is log-Schur-convex in $\Rpd$.
\end{proof}

\subsubsection{Proof of Lemma \titleCrefi{lem:sc functions}{lem idx:sc functions 2}}
We prove \crefi{lem:sc functions}{lem idx:sc functions 2}, i.e., that the function $-\gdb{\vect{x}}_{\sqb{D}}$ is log-Schur-convex in $\left\{ \vect{x}\in \Rpd | \pi\rb{\vect{x}} \ge 1 \right\}$.
\begin{proof}
	Define $\A = \left\{ \vect{x}\in \Rpd | \pi\rb{\vect{x}} \ge 1 \right\}$.
	
	From definition, for any two vectors $\vect{u}, \vect{v}\in\A$ such that $\vect{u}\logmaj\vect{v}$
	\[ \prod_{i=1}^{D-1} u_{\sqb{i}} \le \prod_{i=1}^{D-1} v_{\sqb{i}} \,,\]
	and
	\[ \prod_{i=1}^{D} u_{\sqb{i}} = \prod_{i=1}^{D} v_{\sqb{i}} \,.\]
	Therefore, $u_{\sqb{D}} \ge v_{\sqb{D}}$.
	
	Since $-\eta\rb{\pi\rb{\vect{u}} - 1} \le 0$, we get that for any $a\ge b > 0$:
	\begin{align*}
		a - \eta\rb{\pi\rb{\vect{u}} - 1}\frac{\pi\rb{\vect{u}} }{a}
		\ge b - \eta\rb{\pi\rb{\vect{u}} - 1}\frac{\pi\rb{\vect{u}} }{b}
		\,.
	\end{align*}
	This implies that,
	\begin{align*}
		\gdb{\vect{u}}_{\sqb{D}} = u_{\sqb{D}} - \eta\rb{\pi\rb{\vect{u}}  - 1}\frac{\pi\rb{\vect{u}} }{u_{\sqb{D}}} ~ \text{and},\\
		\gdb{\vect{v}}_{\sqb{D}} = v_{\sqb{D}} - \eta\rb{\pi\rb{\vect{v}}  - 1}\frac{\pi\rb{\vect{v}} }{v_{\sqb{D}}} \,,		
	\end{align*}
    i.e., that the ordering doesn't change after taking a gradient step.
    
	Thus, as $\pi\rb{\vect{u}}= \pi\rb{\vect{v}}$ and $u_{\sqb{D}} \ge v_{\sqb{D}}$, we get that $\gdb{\vect{u}}_{\sqb{D}}  \ge \gdb{\vect{v}}_{\sqb{D}} $.
	Therefore, the function $-\gdb{\vect{x}}_{\sqb{D}}$ is log-Schur-convex in $\A$.
\end{proof}

\subsubsection{Proof of Lemma \titleCrefi{lem:sc functions}{lem idx:sc functions 3}}
We prove \crefi{lem:sc functions}{lem idx:sc functions 3}, i.e. that the function $\pi\rb{\gd{ \vect{x} }}$ is log-Schur-convex in $\left\{ \vect{x}\in \Rpd | \pi\rb{\vect{x}} \le 1 \right\}$.
\begin{proof}
	First, we show that the function $\pi\rb{\gd{ \e^{\vect{x}} }}$ is Schur-convex in $\A'=\left\{ \vect{x}\in \Rd | \sum_{i=1}^{D} x_i \le 0 \right\}$.
	We use \cref{thm:epsilon to convex} to prove this.
	
	Let $\vect{v}\in\A'$. For every $\vect{u}\in\Rpd$, if $\vect{u}\maj\vect{v}$ then
	\begin{equation*}
		\sum_{i=1}^{D} u_i = \sum_{i=1}^{D} v_i \le 0
		\,,
	\end{equation*}
	and therefore $\vect{u}\in\A'$.
	Therefore, $\A'$ has the required property for \cref{thm:epsilon to convex}.
	
	It is easy to see that $\pi\rb{\gd{ \e^{\vect{x}} }}$ is a continuous symmetric function.
	
	For every $k\in\sqb{D-1}$ define
	\begin{align*}
		h\rb{\varepsilon} = \gd{ \e^{\vect{v}_{\sqb{1}}}, \dots, \e^{\vect{v}_{\sqb{k-1}}}, \e^{\vect{v}_{\sqb{k}}+\varepsilon}, \e^{\vect{v}_{\sqb{k+1}} - \varepsilon}, \e^{\vect{v}_{\sqb{k+2}}}, \dots, \e^{\vect{v}_{\sqb{n}}}}
	\end{align*}
	For every $i\in\sqb{D}$ then if $i\neq k,k+1$
	\begin{align*}
		h\rb{\varepsilon}_i
		&= \e^{\vect{v}_{\sqb{i}}} - \eta \rb{ \exp\rb{\sum_{j=1}^{D}\e^{\vect{v}_{\sqb{j}}}} - 1 }\exp\rb{\sum_{j=1}^{D}\e^{\vect{v}_{\sqb{j}}}} \e^{-\vect{v}_{\sqb{i}}}\\
		&= \e^{\vect{v}_{\sqb{i}}} - \eta \rb{ \pi\rb{\e^{\vect{v}}} - 1 }\pi\rb{\e^{\vect{v}}} \e^{-\vect{v}_{\sqb{i}}}
		\,.
	\end{align*}
	Similarly, if $i = k$ then
	\begin{align*}
		h\rb{\varepsilon}_k
		&= \e^{\vect{v}_{\sqb{k}} + \varepsilon} - \eta \rb{ \pi\rb{\e^{\vect{v}}} - 1 }\pi\rb{\e^{\vect{v}}} \e^{-\vect{v}_{\sqb{k}} - \varepsilon}
		\,,
	\end{align*}
	and if $i = k+1$ then
	\begin{align*}
		h\rb{\varepsilon}_{k+1}
		&= \e^{\vect{v}_{\sqb{k+1}} - \varepsilon} - \eta \rb{ \pi\rb{\e^{\vect{v}}} - 1 }\pi\rb{\e^{\vect{v}}} \e^{-\vect{v}_{\sqb{k+1}} + \varepsilon}
		\,.
	\end{align*}

	From the definition of $\A'$, we get that
	\begin{align}
		\label{eq:step in positive direction}
		-\eta \rb{ \pi\rb{\e^{\vect{v}}} - 1 }\pi\rb{\e^{\vect{v}}}>0
		\,.
	\end{align}
	Therefore, for every $i\in\sqb{D}$ then
	\begin{align}
		\label{eq:h_i is positive}
		h\rb{\varepsilon}_i>0
		\,.
	\end{align}

	Therefore, for every $k\in\sqb{D-1}$ then
	\begin{equation}
		\label{eq:h raise together with}
		\begin{aligned}
			\pi\rb{h\rb{\varepsilon}}
			&= h\rb{\varepsilon}_{k} h\rb{\varepsilon}_{k+1} \prod_{i\in\sqb{D}-\left\{ k, k+1 \right\}} h\rb{\varepsilon}_i\\
			&=\Big( - \eta \rb{ \pi\rb{\e^{\vect{v}}} - 1 }\pi\rb{\e^{\vect{v}}} \rb{ \e^{\vect{v}_{\sqb{k+1}}-\vect{v}_{\sqb{k}} - 2\varepsilon}
				+ \e^{\vect{v}_{\sqb{k}}-\vect{v}_{\sqb{k+1}} + 2\varepsilon} } \\
			&\qquad + \e^{\vect{v}_{\sqb{k}} + \vect{v}_{\sqb{k+1}}}
				+ \eta^2 \rb{ \pi\rb{\e^{\vect{v}}} - 1 }^2\pi^2\rb{\e^{\vect{v}}} \e^{-\vect{v}_{\sqb{k}} -\vect{v}_{\sqb{k+1}}} \Big)
				\prod_{i\in\sqb{D}-\left\{ k, k+1 \right\}} h\rb{\varepsilon}_i
		\end{aligned}
	\end{equation}
	Therefore, from \cref{eq:step in positive direction} and \cref{eq:h_i is positive} we get that $\pi\rb{h\rb{\varepsilon}}$ is raising in $\varepsilon$ if and only if
	\begin{align*}
		\e^{\vect{v}_{\sqb{k+1}}-\vect{v}_{\sqb{k}} - 2\varepsilon} + \e^{\vect{v}_{\sqb{k}}-\vect{v}_{\sqb{k+1}} + 2\varepsilon}
	\end{align*}
	is raising in $\varepsilon$.
	
	Therefore, as $\vect{v}_{\sqb{k}}-\vect{v}_{\sqb{k+1}} + 2\varepsilon \ge \vect{v}_{\sqb{k+1}}-\vect{v}_{\sqb{k}} - 2\varepsilon$,
	\begin{align*}
		\frac{\partial }{\partial \varepsilon} \rb{ \e^{\vect{v}_{\sqb{k+1}}-\vect{v}_{\sqb{k}} - 2\varepsilon} + \e^{\vect{v}_{\sqb{k}}-\vect{v}_{\sqb{k+1}} + 2\varepsilon} }
		&= 2\varepsilon \rb{ - \e^{\vect{v}_{\sqb{k+1}}-\vect{v}_{\sqb{k}} - 2\varepsilon} + \e^{\vect{v}_{\sqb{k}}-\vect{v}_{\sqb{k+1}} + 2\varepsilon} }
		\ge 0
		\,.
	\end{align*}
	Thus
	\begin{align*}
		\e^{\vect{v}_{\sqb{k+1}}-\vect{v}_{\sqb{k}} - 2\varepsilon} + \e^{\vect{v}_{\sqb{k}}-\vect{v}_{\sqb{k+1}} + 2\varepsilon}
	\end{align*}
	is raising in $\varepsilon$.
	Therefore, $\pi\rb{h\rb{\varepsilon}}$, i.e.
	\begin{align*}
		\pi\rb{\gd{ \e^{\vect{v}_{\sqb{1}}}, \dots, \e^{\vect{v}_{\sqb{k-1}}}, \e^{\vect{v}_{\sqb{k}}+\varepsilon}, \e^{\vect{v}_{\sqb{k+1}} - \varepsilon}, \e^{\vect{v}_{\sqb{k+2}}}, \dots, \e^{\vect{v}_{\sqb{n}}}}}
		\,,
	\end{align*}
	is raising in $\varepsilon$.

	Therefore, using \cref{thm:epsilon to convex}, we get that $\pi\rb{\gd{ \e^{\vect{x}} }}$ is Schur-convex on $\A'$.
	
	Define $\A = \left\{ \vect{x}\in \Rpd | \pi\rb{\vect{x}} \le 1 \right\}$.
	We have that $\log \A = \A'$.
	Using \cref{lem:Schur-convex to log-Schur-convex}, we get that $\pi\rb{\gd{ \vect{x} }}$ is log-Schur-convex on $\A$.
\end{proof}

\subsubsection{Extension of Lemma \titleCrefi{lem:sc functions}{lem idx:sc functions 3}}
In this section, we extend \crefi{lem:sc functions}{lem idx:sc functions 3}. This extension will be used in the proof of Lemma \ref{lem:GPGD to GD}.

\begin{lemma}
	\label{lem idx:sc functions 4}
	The function $-\pi\rb{\gd{ \vect{x} }}$ is log-Schur-convex on $\SG \cap \left\{ \vect{x}\in \Rpd | \pi\rb{\vect{x}} \ge 1 \right\}$.
\end{lemma}
\begin{proof}
	The proof of \crefi{lem:sc functions}{lem idx:sc functions 3} is almost entirely true here.
	There are only some differences.
	
	Define $\A=\SG \cap \left\{ \vect{x}\in \Rpd | \pi\rb{\vect{x}} \ge 1 \right\}$.
	For every $\vect{v}\in \A$ and $\vect{u}\in \Rpd$ if $\vect{u}\logmaj\vect{v}$ then $\vect{u}\in\A$.
	This is because $\pi\rb{\vect{v}} = \pi\rb{\vect{u}}$, and because the proof of \cref{thm:porjected sharpness decrease} proved that $\vect{u}\in\SG$.
	Therefore, $\A'=\log \A$ has the required property for \cref{thm:epsilon to convex}.
	
	We get $\cref{eq:h_i is positive}$ simply from the definition of $\SG$.
	Instead of \cref{eq:step in positive direction}, we get that
	\begin{align*}
		-\eta \rb{ \pi\rb{\e^{\vect{v}}} - 1 }\pi\rb{\e^{\vect{v}}} < 0
		\,.
	\end{align*}

	Thus, from $\cref{eq:h raise together with}$, we get that $\pi\rb{h\rb{\varepsilon}}$ is decreasing in $\varepsilon$ if and only if
	\begin{align*}
		\e^{\vect{v}_{\sqb{k+1}}-\vect{v}_{\sqb{k}} - 2\varepsilon} + \e^{\vect{v}_{\sqb{k}}-\vect{v}_{\sqb{k+1}} + 2\varepsilon}
	\end{align*}
	is raising $\varepsilon$.
	
	Therefore, we get that the function $-\pi\rb{\gd{ \vect{x} }}$ is log-Schur-convex on $\A$ (instead of the $\pi\rb{\gd{ \vect{x} }}$ as in the proof of \crefi{lem:sc functions}{lem idx:sc functions 3}).
\end{proof}

\section{GPGD analysis}

\subsection{Proof of Lemma \ref{lem:GPGD to GD}}\label{Sec: GPGD to GD lemma proof}
In the proof, we use the following auxiliary lemma. Note that for vectors $\vect{u},\vect{v}\in\Rd$, $\vect{u} \le \vect{v}$ means that $u_i\le v_i$ for every $i\in\D$.
\begin{lemma}
	\label{lem:g_i decrease}
	For any weights $\vect{u}, \vect{v} \in \Rpd$, if $\vect{u} \bequiv \vect{v}$, $\pi\rb{\vect{u}} \ge \pi\rb{\vect{v}} \ge 1$, and $\ps{\vect{u}} \ge \frac{1}{\eta}$ then
	\begin{align*}
		\gd{\vect{u}} \le \gd{\vect{v}}
		\,.
	\end{align*}
\end{lemma}
The proof can be found in \cref{proof:g_i decrease}.

We now prove \cref{lem:GPGD to GD}.
\begin{proof}
	Let weight $\vect{w}\in\SG$, such that $\pi\rb{\qs{\vect{w}}} \ge 1$ and $\ps{\gd{\vect{w}}} \ge \frac{1}{\eta}$.
	For any $\vect{w}$ such that $\pi\rb{\vect{w}}>0$, we assume without loss of generality that $\vect{w}^{\rb{t}} \in R_{+}^{D}$, see \cref{sec:Equivalence of weights}.

	From the definition of $\qs{\vect{w}}$ then
	\begin{align}
		\label{eq:GPGD equal gd}
		\pi\rb{\gd{\vect{w}}} = \pi\rb{\qs{\vect{w}}} \ge 1
		\,.
	\end{align}

	From \cref{lem:gd balance}, we get that $\gd{\vect{w}} \ble \vect{w}$.
	Thus $\gd{\vect{w}} \ble \qs{\vect{w}}$ since $\vect{w}$ and $ \qs{\vect{w}}$ have the same balances from the GPGD definition (Eq. \ref{Eq: qs GD}).
	Therefore, using the fact that $\pi\rb{\qs{\vect{w}}} = \pi\rb{\gd{\vect{w}}}$ and \cref{lem:balance to majorization} we get that
	\begin{align*}
		\gd{\vect{w}} \maj \qs{\vect{w}}
		\,.
	\end{align*}
	Therefore, 
	\begin{align}
		\label{eq:qs and gd bound 1}
		\pi\rb{\gd{ \gd{\vect{w}} }} \ge \pi\rb{\gd{ \qs{\vect{w}} }}=\pi\rb{\qs{ \qs{\vect{w}} }}
		\,,
	\end{align}
    where the first inequality is true from \cref{lem idx:sc functions 4} and the equality is a direct result of the GPGD definition (Eq. \ref{Eq: qs GD}).

	Additionally, from \cref{lem:g_i decrease}, and because \cref{eq:GPGD equal gd}, we get that
	\begin{align*}
		\pi\rb{\gd{ \gd{\vect{w}} }} \le \pi\rb{\gd{ \Pgf{\gd{\vect{w}}} }} = \pi\rb{\Pgf{\gd{\vect{w}}}} = 1
		\,.
	\end{align*}
	Therefore,  
	\begin{align}
		\label{eq:qs and gd bound}
		0<\pi\rb{\qs{ \qs{\vect{w}} }} \le \pi\rb{\gd{ \gd{\vect{w}} }} \le 1
		\,,
	\end{align}
    where the first inequality is a result of \cref{lem:GD does not change sign}, the second inequality is a result of \cref{eq:qs and gd bound 1}, and the last inequality is from the previous equation.

	Using \cref{eq:qs and gd bound} and the loss definition (Eq. \ref{eq:loss function}) we get that
	\begin{align*}
		\loss\rb{\gd{\gd{\vect{w}}}}
		= \frac{1}{2}\rb{ \pi\rb{\gd{\gd{\vect{w}}}} - 1 }^2
		\le \frac{1}{2}\rb{ \pi\rb{\qs{\qs{\vect{w}}}} - 1 }^2
		=\loss\rb{\qs{\qs{\vect{w}}}}
		\,.
	\end{align*}
\end{proof}

\subsubsection{Proof of Lemma \ref{lem:g_i decrease}}
\label{proof:g_i decrease}
\begin{proof}
	Assume that the balances $\vect{b}$, as defined in \cref{def:balances appendix}, are constant.
	Let $\vect{w} \in \Rpd$ be a weight such that $\vect{w}$ has the balances $\vect{b}$, $\pi\rb{\vect{w}}\ge1$ and $\ps{\vect{w}} \ge \frac{1}{\eta}$.
	For every $i,j\in\sqb{D}$, we can write
	\begin{align*}
		w_j^2 = w_i^2 + b_{j,i}
		\,.
	\end{align*}
	Therefore,
	\begin{align*}
		\frac{\partial w_j}{\partial w_i} 
		&= \frac{\partial \sqrt{w_i^2 + b_{j,i}}}{\partial w_i}
		= \frac{2w_i}{2\sqrt{w_i^2 + b_{j,i}}}
		= \frac{w_i}{w_j}
		\,.
	\end{align*}
	
	Therefore,
	\begin{align*}
		\frac{\partial \gdb{\vect{w}}_{i}}{\partial w_i}
		&= \frac{ \partial \rb{ w_i -\eta \rb{\prod_{j=1}^{D}w_j - 1 }\prod_{j\in\sqb{D}-\{i\}}w_j } }{\partial w_i}\\
		&= 1 - \eta \rb{\sum_{j=1}^{D} \frac{w_i}{w_j} \prod_{k\in\sqb{D}-\{j\}} w_k} \prod_{j\in\sqb{D}-\{i\}}w_j
		-\eta \rb{\prod_{j=1}^{D}w_j - 1 } \sum_{j\in\sqb{D}-\{i\}} \frac{w_i}{w_j} \prod_{k\in\sqb{D}-\{j,i\}} w_k\\
		&= 1 - \eta \pi\rb{\vect{w}}^2 \sum_{j=1}^{D} \frac{1}{w_j^2}
		-\eta \rb{\pi\rb{\vect{w}} - 1 }\pi\rb{\vect{w}} \sum_{j\in\sqb{D}-\{i\}} \frac{1}{w_j^2}
	\end{align*}
	Therefore, as $\pi\rb{\vect{w}}\ge1$, then
	\begin{align*}
		\frac{\partial \gdb{\vect{w}}_{i}}{\partial w_i} \le 1 - \eta \pi^2\rb{\vect{w}} \sum_{j=1}^{D} \frac{1}{w_j^2} = 1- \eta s_1\rb{\vect{w}}
		\,.
	\end{align*}
	
	It is easy to see that for constant balances $\vect{b}$, the weights increase as the value of $\pi(w)$ increases.
	Therefore, the value of $s_1\rb{\vect{w}}$ increases as the value of $\pi(w)$ increases.
	Thus, as $\pi\rb{\vect{w}} \ge 1$, then
	\begin{align*}
		\frac{\partial \gdb{\vect{w}}_{i}}{\partial w_i} \le 1- \eta s_1\rb{\vect{w}} \le 1 - \eta s_1\rb{\Pgf{\vect{w}}}
		\,.
	\end{align*}
	From $s_1$ definition (\cref{eq:sharpness at opt}),
	\begin{align*}
		\frac{\partial \gdb{\vect{w}}_{i}}{\partial w_i} \le 1- \eta s_1\rb{\Pgf{\vect{w}}} = 1- \eta \ps{\vect{w}}
		\,.
	\end{align*}
	Therefore, because $\ps{\vect{w}} \ge \frac{1}{\eta}$, then
	\begin{align*}
		\frac{\partial \gdb{\vect{w}}_{i}}{\partial w_i} \le 1- \eta \ps{\vect{w}} \le 1 - 1 = 0
		\,.
	\end{align*}
	
	Therefore, $\gdb{\vect{w}}_{i}$ decrease as $w_i$ increase.
	
	For vectors $\vect{u},\vect{v}\in\Rpd$ s.t. $\vect{u} \bequiv \vect{v}$ (i.e., $\vect{u},\vect{v}$ have the same balances), $\pi\rb{\vect{u}} \ge \pi\rb{\vect{v}} \ge 1$ implies that $\forall i\in\sqb{D}:\ u_i>v_i$. Thus, as we also assumed that $\ps{\vect{u}} \ge \frac{1}{\eta}$, we get that
	\begin{align*}
		\gd{\vect{u}} \le \gd{\vect{v}}
		\,.
	\end{align*}
\end{proof}

\subsection{Proof for Lemma \ref{lem:GFS sharpness decrease}}\label{Sec: GFS sharpness decrease lemma proof}

\begin{proof}
	Let $\wt \in \SG$.
	Let assume without loss of generality that $w_1^2\ge w_2^2 \ge \dots \ge w_D^2$.
	Let $\wst \triangleq \Pgf{\wt}$.
	
	For every $i\in\D$, from \cref{dyn:balances}, we get that
	\begin{align*}
		\rb{\ws^{\rb{t+1}}_1}^2 - \rb{\ws^{\rb{t+1}}_i}^2
		&= {w^{\rb{t+1}}_1}^2 - {w^{\rb{t+1}}_i}^2\\
		&= {w^{\rb{t}}_1}^2 - {w^{\rb{t}}_i}^2 - \rb{{w^{\rb{t}}_1}^2 - {w^{\rb{t}}_i}^2} \eta^2 \left( \pi^{(t)} - 1 \right)^2 \frac{{\pi^{(t)}}^2}{{w^{(t)}_1}^2 {w^{(t)}_i}^2}\\
		&= {w^{\rb{t}}_{\D[1]}}^2 - {w^{\rb{t}}_i}^2 - \rb{{w^{\rb{t}}_{\D[1]}}^2 - {w^{\rb{t}}_i}^2} \eta^2 \left( \pi^{(t)} - 1 \right)^2 \frac{{\pi^{(t)}}^2}{{w^{(t)}_{\D[1]}}^2 {w^{(t)}_i}^2}\\
		&\ge {w^{\rb{t}}_{\D[1]}}^2 - {w^{\rb{t}}_i}^2 - \eta^2 \left( \pi^{(t)} - 1 \right)^2 \frac{{\pi^{(t)}}^2}{{w^{(t)}_i}^2}\\
		&= \rb{\ws^{\rb{t}}_{\D[1]}}^2 - \rb{\ws^{\rb{t}}_i}^2 - \eta^2 \left( \pi^{(t)} - 1 \right)^2 \frac{{\pi^{(t)}}^2}{{w^{(t)}_i}^2}
		\,.
	\end{align*}
	Therefore,
	\begin{align}
		\rb{\ws^{\rb{t+1}}_i}^2 - \rb{\ws^{\rb{t}}_i}^2
		&\le \rb{\ws^{\rb{t+1}}_1}^2 - \rb{\ws^{\rb{t}}_{\D[1]}}^2 + \eta^2 \left( \pi^{(t)} - 1 \right)^2 \frac{{\pi^{(t)}}^2}{{w^{(t)}_i}^2}\nonumber\\
		&\le \rb{\ws^{\rb{t+1}}_{\D[1]}}^2 - \rb{\ws^{\rb{t}}_{\D[1]}}^2 + \eta^2 \left( \pi^{(t)} - 1 \right)^2 \frac{{\pi^{(t)}}^2}{{w^{(t)}_i}^2} \label{eq:GFS sharpness decrease eq1}
		\,.
	\end{align}

	Let assume toward contradiction that $\rb{\ws^{\rb{t+1}}_{\D[1]}}^2 > \rb{\ws^{\rb{t}}_1}^2$.
	From \cref{lem:gd balance} we get that $\wst[t+1] \ble \wst$.
	Combining these results we get that, for every $i\in\D$: $\rb{\ws^{\rb{t+1}}_{\D[i]}}^2 > \rb{\ws^{\rb{t}}_{\D[i]}}^2$.
	Therefore, $\pi^2\rb{\wst[t+1]} > \pi^2\rb{\wst}$, which contradicts that $\pi^2\rb{\wst[t+1]} = 1 = \pi^2\rb{\wst}$.
	Therefore,
	\begin{align*}
		\rb{\ws^{\rb{t+1}}_{\D[1]}}^2 - \rb{\ws^{\rb{t}}_{\D[1]}}^2 \le 0
		\,.
	\end{align*}
	
	Combining the last equation with  \cref{eq:GFS sharpness decrease eq1}, we get that for every $i\in\D$
	\begin{align}
		\label{eq:GFS sharpness decrease eq2}
		\rb{\ws^{\rb{t+1}}_i}^2 - \rb{\ws^{\rb{t}}_i}^2 \le \eta^2 \left( \pi^{(t)} - 1 \right)^2 \frac{{\pi^{(t)}}^2}{{w^{(t)}_i}^2}
		\,.
	\end{align}

	From \cref{eq:sharpness at opt} we know that $\ps{\wt}=\sum_{j=1}^{D} \frac{1}{\rb{\ws^{\rb{t}}_j}^2}$. Thus, for every $i\in\D$
	\begin{align*}
		\frac{1}{\rb{\ws^{\rb{t}}_i}^2}
		&\le \sum_{j=1}^{D} \frac{1}{\rb{\ws^{\rb{t}}_j}^2}
		= \ps{\wt}
		= \ps{\wt} \frac{\rb{\ws^{\rb{t}}_i}^2}{\rb{\ws^{\rb{t}}_i}^2}\le \rb{\ws^{\rb{t}}_i}^2 \pss{\wt}
		\,.
	\end{align*}
	Overall,
	\begin{align}
		\label{eq:GFS sharpness decrease eq3}
		\frac{1}{\rb{\ws^{\rb{t}}_i}^2}
		&\le \rb{\ws^{\rb{t}}_i}^2 \pss{\wt}
		\,.
	\end{align}

    Next, we divide into two cases and lower bound $\ps{\wt[t+1]}$ for each case to obtain the desired result:
	\begin{enumerate}
		\item If $\pi^{(t)} \ge 1$ then for every $i\in\D$ we get that ${w^{(t)}_i}^2 \ge \rb{\ws^{\rb{t}}_i}^2\Rightarrow\frac{1}{{w^{(t)}_i}^2} \le \frac{1}{\rb{\ws^{\rb{t}}_i}^2}$.
		Substituting this result into \cref{eq:GFS sharpness decrease eq2} and \cref{eq:GFS sharpness decrease eq3}, we get that
		\begin{align*}
			\rb{\ws^{\rb{t+1}}_i}^2
			&\le \rb{\ws^{\rb{t}}_i}^2 + \eta^2 \left( \pi^{(t)} - 1 \right)^2 \frac{{\pi^{(t)}}^2}{\rb{\ws^{\rb{t}}_i}^2}\\
			&\le \rb{\ws^{\rb{t}}_i}^2 + \eta^2 \left( \pi^{(t)} - 1 \right)^2 {\pi^{(t)}}^2 \rb{\ws^{\rb{t}}_i}^2 \pss{\wt}\\
			&= \rb{\ws^{\rb{t}}_i}^2 \rb{1 + \eta^2 \pss{\wt} \left( \pi^{(t)} - 1 \right)^2 {\pi^{(t)}}^2 }
			\,.
		\end{align*}

		The last equation enables us to lower bound $\ps{\wt[t+1]}$, using \cref{eq:sharpness at opt}:
		\begin{align*}
			\ps{\wt[t+1]}
			&= \sum_{i=1}^{D} \frac{1}{\rb{\ws^{\rb{t+1}}_i}^2}\\
			&\ge \frac{1}{1 + \eta^2 \pss{\wt} \left( \pi^{(t)} - 1 \right)^2 {\pi^{(t)}}^2} \sum_{i=1}^{D} \frac{1}{\rb{\ws^{\rb{t}}_i}^2}\\
			&= \frac{\ps{\wt}}{1 + \eta^2 \pss{\wt} \left( \pi^{(t)} - 1 \right)^2 {\pi^{(t)}}^2}
			\,.
		\end{align*}
	
		\item If $\pi^{(t)} \le 1$ then for every $j\in\D$ we get that ${w^{(t)}_j}^2 \le \rb{\ws^{\rb{t}}_j}^2$.
		Therefore, for every $i\in\D$ we get that
		\begin{align*}
			\frac{\pi^2\rb{\wt}^2}{{w^{(t)}_i}} \le \frac{\pi^2\rb{\ws^{\rb{t}}}}{\rb{\ws^{\rb{t}}_i}^2}=\frac{1}{\rb{\ws^{\rb{t}}_i}^2}
			\,.
		\end{align*}
		Combining the last equation with \cref{eq:GFS sharpness decrease eq2} and \cref{eq:GFS sharpness decrease eq3}, we obtain
		\begin{align*}
			\rb{\ws^{\rb{t+1}}_i}^2
			&\le \rb{\ws^{\rb{t}}_i}^2 + \eta^2 \left( \pi^{(t)} - 1 \right)^2 \frac{1}{\rb{\ws^{\rb{t}}_i}^2}\\
			&\le \rb{\ws^{\rb{t}}_i}^2 + \eta^2 \left( \pi^{(t)} - 1 \right)^2 \rb{\ws^{\rb{t}}_i}^2 \pss{\wt}\\
			&= \rb{\ws^{\rb{t}}_i}^2 \rb{1 + \eta^2 \pss{\wt} \left( \pi^{(t)} - 1 \right)^2 }
			\,.
		\end{align*}
	
		Once again, the last equation enables us to lower bound $\ps{\wt[t+1]}$, using \cref{eq:sharpness at opt}:
		\begin{align*}
			\ps{\wt[t+1]}
			&= \sum_{i=1}^{D} \frac{1}{\rb{\ws^{\rb{t+1}}_i}^2}\\
			&\ge \frac{1}{1 + \eta^2 \pss{\wt} \left( \pi^{(t)} - 1 \right)^2} \sum_{i=1}^{D} \frac{1}{\rb{\ws^{\rb{t}}_i}^2}\\
			&= \frac{\ps{\wt}}{1 + \eta^2 \pss{\wt} \left( \pi^{(t)} - 1 \right)^2}
			\,.
		\end{align*}
	\end{enumerate}

	Combining the result from both cases, we get that
	\begin{align*}
		\ps{\wt[t+1]}
		&\ge \frac{\ps{\wt}}{1 + \rb{ \eta \ps{\wt} \max\left\{ 1, \pi^{\rb{t}} \right\} }^2 \left( \pi^{(t)} - 1 \right)^2}\\
		&= \frac{\ps{\wt}}{1 + 4\rb{  \frac{\ps{\wt}}{2/\eta} \max\left\{ 1, \pi^{\rb{t}} \right\} }^2 \left( \pi^{(t)} - 1 \right)^2}\\
		&= \frac{\ps{\wt}}{1 + 8\rb{  \frac{\ps{\wt}}{2/\eta} \max\left\{ 1, \pi^{\rb{t}} \right\} }^2 \loss\rb{\wt}}
	\end{align*}
\end{proof}

\subsection{Proof of Lemma \ref{lem:qs GD convergance}}\label{Sec: qs GD convergance lemma proof}
In this section, we assume that the balances $\vect{b}$, as defined in \cref{def:balances appendix}, are constant.
We define $\psc$ to be the GFS sharpness of the weights with balances $\vect{b}$.
We define $\vect{w}(x)$ as the weight with balances $\vect{b}$ such that $\pi\rb{\vect{w}(x)}=x$, as explained in \cref{sec:Equivalence of weights}, all possible $\vect{w}(x)$ are equivalent.
We define $\sm{x} : \R \mapsto \R$ as
\begin{align*}
	\sm{x} \triangleq s_m\rb{\vect{w}(x)}
	\,.
\end{align*}

As the balances $\vect{b}$ are constant then \cref{dyn:product of weights} characterizes the dynamics of the product of the weights under GPGD.
Therefore, as the  product of the weights $x$ is equivalent to the weight $\vect{w}(x)$, the function $\qs{w}$ is equivalent to the function $\qg{x} : \R \mapsto \R$ defined as
\begin{align}
	\label{eq:def of q}
	\qg{x} = x + \sum_{m=1}^{D} \eta^m (1-x)^{m} {x}^{m-1} \sm{x}
	\,.
\end{align}
This definition is equivalent to the GPGD, as by using \cref{dyn:product of weights} we obtain that
\begin{align*}
	\pi\rb{\qs{\vect{w}(x)}} = \qg{x}
	\,.
\end{align*}
Similarly, for any weight $\vect{w}$ then
\begin{align*}
	\pi\rb{\qs{\vect{w}}} = \qg{\pi\rb{\vect{w}}}\,.
\end{align*}

We use the following auxiliaries lemmas
\begin{lemma}
	\label{lem:derivative quasi-static}
	We have that
	\begin{align*}
		\frac{\partial w^2_i(x)}{\partial x} &= \frac{2x}{\sm[1]{x}} && \forall i \in \D \,,\\
		\frac{\partial \sm{x}}{\partial x} &= 2\rb{m+1}x\frac{\sm[m+1]{x}}{\sm[1]{x}}  && \forall m\in\D[D-1]\cup\{0\} \,\text{and}\\
		\frac{\partial \sm[D]{x}}{\partial x} &= 0
		\,.
	\end{align*}
\end{lemma}

\begin{lemma}
	\label{lem:s2 over s_1}
	The function
	\begin{align*}
		\frac{\sm[2]{x}}{\sm[1]{x}}
	\end{align*}
	is decreasing in $x$, for $x>0$.
\end{lemma}

\begin{lemma}
	\label{lem:q rate increases}
	If $\psc\le\frac{2}{\eta}$, then the function
	\begin{align*}
		\frac{\qg{x} - 1}{1-x}
	\end{align*}
	is increasing in $x$ over the region $(0,1]$.
\end{lemma}

\begin{lemma}
	\label{lem:derivative upper bound}
	Define $\delta \triangleq 2 - \eta \psc$.
	If
	\begin{align*}
		0 &\le \delta \le 0.5 \quad \text{and}\\
		1 &\le x \le 1 + 0.1 \delta
		\,.
	\end{align*}
	then
	\begin{align*}
		\frac{\partial \frac{\qg{x} - 1}{1-x}}{\partial x} \le \eta \sm[2]{x} \rb{ \eta \rb{1-2x} + 2 x \frac{2}{\sm[1]{x}} }
		\,.
	\end{align*}
\end{lemma}

In addition, in the proof of \cref{lem:qs GD convergance}, we use some equation from the proof of \cref{lem:q rate increases}.
The proofs of the lemmas can be found in \cref{proof:derivative quasi-static}, \cref{proof:s2 over s_1}, \cref{proof:q rate increases}, and \cref{proof:derivative upper bound}

Using these Lemmas, we prove \cref{lem:qs GD convergance}.
\begin{proof}
	Define $c \triangleq\eta \psc$ and $\delta = 2 - c\in(0,0.5]$.
	This implies, $1.5 \le c < 2$.
	Let $x$ such that $x \le 1$ and $x \ge 1 - 0.1\frac{\delta}{1-\delta}$.
	Define that $\qgs^{\circ t}\rb{x}$ is performing $t$ times $\qgs\rb{x}$, i.e. $\qgs^{\circ t}\rb{x} = \qgs^{\circ \rb{t-1}}\rb{\qg{x}}$ and $\qgs^{\circ 0}\rb{x} = x$.
	
	The main proof steps are:
	\begin{enumerate}[ref=\theenumi]
		\item\label{step:1} Show that because $x$ is close enough to $1$ then
		\begin{align*}
			\frac{1}{2}\le x\le 1 ~ \text{and} ~ \qg{x} \ge 1
			\,.
		\end{align*}
	
		\item\label{step:2} Show how much $\qg{x}$ is close to $1$ comparatively to how much $x$ is close to 1, i.e. that
		\begin{align*}
			0 \le \qg{x} - 1 \le \rb{c-1}\rb{1-x}
			\,.
		\end{align*}
	
		\item\label{step:3} Show how much $\qg{\qg{x}}$ is close to $1$ comparatively to how much $x$ is close to 1.
		\begin{enumerate}[ref=\theenumi.\theenumii]
			\item\label{step:3.1} First we show that
			\begin{align*}
				1 - \qg{\qg{x}}
				\le \rb{c-1}^2 \rb{1-x}
				\,.
			\end{align*}
			
			\item\label{step:3.2} Then we show that
			\begin{align*}
				1 - \qg{\qg{x}}
				\ge 0
				\,.
			\end{align*}
		\end{enumerate}
	
		\item\label{step:4} Finally, using step \ref{step:3}, we conclude that $\qg{\qg{x}}$ fulfill all the requirement of this the lemma, i.e. $1\ge \qg{\qg{x}} \ge 1 - 0.1 \frac{\delta}{1-\delta}$.
		Therefore, by iteratively using steps \ref{step:3} and \ref{step:2}, we obtain that	for all $t\ge0$:
		\begin{align*}
			\loss\rb{\vect{w}\rb{\qgs^{\circ t}\rb{x}}} \le \rb{1-\delta}^{2t} \loss\rb{\vect{w}(x)}
			\,,
		\end{align*}
		for even $t\ge0$: $\qgs^{\circ t}\rb{x} \le 1$,
		and for odd $t>0$: $\qgs^{\circ t}\rb{x} \ge 1$.
	\end{enumerate}
	
	\hyperref[step:1]{\textbf{Step~\labelcref*{step:1}:}}
	
	We first show that
	\begin{align*}
		\frac{1}{2}\le x\le 1 ~ \text{and} ~ \qg{x} \ge 1
		\,.
	\end{align*}

	As $\delta\in(0,0.5]$, we have that
	\begin{align}
		\label{eq:x bound 0.9}
		1\ge x \ge 1 - 0.1\frac{\delta}{1-\delta} \ge 1 - 0.1 \ge 0.5
		\,.
	\end{align}
	
	For $\frac{1}{2}\le x\le 1$, then by using the definition of $\qg{x}$ (\cref{eq:def of q}),
	\begin{align*}
		\qg{x} -1 \ge x-1 + \eta \rb{1-x}\sm[1]{x}
		= x-1 + \eta \rb{1-x}x^2 \sum_{i=1}^{D} \frac{1}{w^2_i(x)}
		\,.
	\end{align*}
	Using \cref{lem:derivative quasi-static}, we obtain that
	\begin{align*}
		\qg{x} -1
		&\ge x-1 + \eta \rb{1-x}x^2 \sum_{i=1}^{D} \frac{1}{w^2_i(x)}\\
		&\ge x-1 + \eta \rb{1-x}x^2 \sum_{i=1}^{D} \frac{1}{w^2_i(1)}\\
		&= x-1 + \eta \rb{1-x}x^2 \sm[1]{1}
		\,.
	\end{align*}
	Using, $ \sm[1]{1}=\psc$ (from definition) we obtain that
	\begin{align*}
		\qg{x} -1
		&\ge x-1 + \eta\psc \rb{1-x}x^2\\
		&= \rb{1-x} \rb{-1 + c x^2}
		\,.
	\end{align*}
	We solve
	\begin{align*}
		\rb{1-x} \rb{-1 + c x^2} = 0
	\end{align*}
	and get solutions $x=1$ and $x=\frac{1}{\sqrt{c}}$.
	Therefore, if $x\in[\frac{1}{\sqrt{c}}, 1]$ then $\qg{x}\ge1$.
	Thus, using \cref{eq:x bound 0.9}, and as $c\ge1.5$,
	\begin{align*}
		\frac{1}{\sqrt{c}} \le \frac{1}{\sqrt{1.5}} < 0.9 \le x
		\Rightarrow \qg{x}\ge1 \,.
	\end{align*}
	
	Overall, we obtain
	\begin{align}
		\label{eq:oscilate}
		\frac{1}{2}\le x\le 1 ~ \text{and} ~ \qg{x} \ge 1
		\,.
	\end{align}
	
	\hyperref[step:2]{\textbf{Step~\labelcref*{step:2}:}}
	
	We now show that
	\begin{align*}
		0 \le \qg{x} - 1 \le \rb{c-1}\rb{1-x}
		\,.
	\end{align*}
	
	From $\qg{x}$ definition (Eq. \eqref{eq:def of q}), we have that for $x=1$ then
	\begin{align}
		\frac{\qg{x} - 1}{1-x} &= -1 + \sum_{m=1}^{D} \eta^{m} (1-x)^{m-1} {x}^{m-1} \sm{x}\nonumber\\
		&= -1 + \eta \sm[1]{x} \nonumber\\
		&= -1 + \eta \psc \nonumber\\
        & = c - 1 \label{eq:q at opt}
		\,.
	\end{align}

	Therefore,
	\begin{align}
		\label{eq:qs gd conv eq1}
		0 \le \frac{\qg{x} - 1}{1-x} \le c-1
		\,,
	\end{align}
	where the first inequality is from \cref{eq:oscilate}, and the second inequality is from \cref{eq:q at opt} and \cref{lem:q rate increases}.
	This implies,
	\begin{align}
		\label{eq:conv rate 1 step}
		0 \le \qg{x} - 1 \le \rb{c-1}\rb{1-x}
		\,.
	\end{align}
	
	\hyperref[step:3]{\textbf{Step~\labelcref*{step:3}:}}
	
	Using \cref{eq:q at opt} we obtain that
	\begin{align*}
		\frac{\qg{\qg{x}} - 1}{x-1}
		&= \frac{\qg{x} - 1}{1 -x} \cdot \frac{\qg{\qg{x}} - 1}{1 - \qg{x}}\\
		&=\rb{ \rb{c-1} + \int_{1}^{x}\frac{\partial \frac{\qg{z} - 1}{1-z}}{\partial z} dz } \rb{ \rb{c-1} + \int_{1}^{\qg{x}}\frac{\partial \frac{\qg{z} - 1}{1-z}}{\partial z} dz }
		\,.
	\end{align*}
	Thus,
	\begin{equation}
		\begin{aligned}
			\label{eq:two step of q}
			\frac{\qg{\qg{x}} - 1}{x-1}
			= \rb{c-1}^2 &+ \rb{c-1} \rb{ -\int_{x}^{1}\frac{\partial \frac{\qg{z} - 1}{1-z}}{\partial z} dz + \int_{1}^{\qg{x}}\frac{\partial \frac{\qg{z} - 1}{1-z}}{\partial z} dz }\\
			&-\int_{x}^{1}\frac{\partial \frac{\qg{z} - 1}{1-z}}{\partial z} dz \int_{1}^{\qg{x}}\frac{\partial \frac{\qg{z} - 1}{1-z}}{\partial z} dz
			\,.
		\end{aligned}
	\end{equation}
	
	\hyperref[step:3.1]{\textbf{Step~\labelcref*{step:3.1}:}}
	
	Our goal is to prove that
	\begin{align}
		\label{eq:qs gd conv goal}
		\frac{\qg{\qg{x}} - 1}{x-1} \le \rb{c-1}^2
		\,.
	\end{align}
	
	If
	\begin{align*}
		\int_{1}^{\qg{x}}\frac{\partial \frac{\qg{z} - 1}{1-z}}{\partial z} dz \le 0
	\end{align*}
	then
	\begin{align*}
		\frac{\qg{\qg{x}} - 1}{1 - \qg{x}} = \rb{c-1} + \int_{1}^{\qg{x}}\frac{\partial \frac{\qg{z} - 1}{1-z}}{\partial z} dz \le \rb{c-1}
		\,.
	\end{align*}
	Therefore, using \cref{eq:qs gd conv eq1}, we obtain that
	\begin{align}
		\label{eq:qs gd conv eq2}
		\frac{\qg{\qg{x}} - 1}{x-1} \le \rb{c-1}^2
		\,.
	\end{align}
	
	We will now show that \cref{eq:qs gd conv goal} is true even when
	\begin{align*}
		\int_{1}^{\qg{x}}\frac{\partial \frac{\qg{z} - 1}{1-z}}{\partial z} dz > 0
		\,.
	\end{align*}
	
	From \cref{lem:q rate increases} we obtain that
	\begin{align*}
		\int_{x}^{1}\frac{\partial \frac{\qg{z} - 1}{1-z}}{\partial z} dz \ge 0
		\,.
	\end{align*}
	Therefore,
	\begin{align}
		\label{eq:two step q elem3}
		-\int_{x}^{1}\frac{\partial \frac{\qg{z} - 1}{1-z}}{\partial z} dz \int_{1}^{\qg{x}}\frac{\partial \frac{\qg{z} - 1}{1-z}}{\partial z} dz \le 0
		\,.
	\end{align}

	Because
	\begin{align*}
		1 \ge x \ge 1 - 0.1\frac{\delta}{1 - \delta}
	\end{align*}
	then
	\begin{align*}
		1 \le \qg{x} \le 1 + 0.1\delta
		\,,
	\end{align*}
	where the first inequality is because of \cref{eq:oscilate}, and the second inequality is because of \cref{eq:conv rate 1 step}.
	From \cref{lem:derivative upper bound}, we obtain that
	\begin{align*}
		\int_{1}^{\qg{x}}\frac{\partial \frac{\qg{z} - 1}{1-z}}{\partial z} dz
		&\le \int_{1}^{\qg{x}} \eta \sm[2]{z} \rb{ \eta \rb{1-2z} + 2 z \frac{2}{\sm[1]{z}} }dz\\
		&=\int_{1}^{\qg{x}} \rb{ \eta^2 \rb{1-2z}\sm[2]{z} + 4 \eta z \frac{\sm[2]{z}}{\sm[1]{z}} }dz
		\,.
	\end{align*}
	Therefore, as from \cref{lem:derivative quasi-static} we get that $\sm[2]{z}$ is increasing for $z>0$, then
	\begin{align*}
		\int_{1}^{\qg{x}}\frac{\partial \frac{\qg{z} - 1}{1-z}}{\partial z} dz
		&\le \int_{1}^{\qg{x}} \rb{ \eta^2 \rb{1-2z}\sm[2]{1} + 4 \eta z \frac{\sm[2]{z}}{\sm[1]{z}} }dz
		\,.
	\end{align*}
	Therefore, using \cref{lem:s2 over s_1} we get that
	\begin{align*}
		\int_{1}^{\qg{x}}\frac{\partial \frac{\qg{z} - 1}{1-z}}{\partial z} dz
		&\le \int_{1}^{\qg{x}} \rb{ \eta^2 \rb{1-2z}\sm[2]{1} + 4 \eta z \frac{\sm[2]{1}}{\sm[1]{1}} }dz\\
		&=  \eta^2\sm[2]{1} \int_{1}^{\qg{x}} \rb{ \rb{1-2z} + 4 z \frac{1}{\eta\sm[1]{1}} }dz
		\,.
	\end{align*}
	Thus, by using \cref{eq:sharpness at opt},
	\begin{align}
		\label{eq:qs gd conv eq3}
		\int_{1}^{\qg{x}}\frac{\partial \frac{\qg{z} - 1}{1-z}}{\partial z} dz
		&\le \eta^2\sm[2]{1} \int_{1}^{\qg{x}} \rb{ 1 + 2z\rb{\frac{2}{c} - 1} }dz
		\,.
	\end{align}
	
	From \cref{eq:derivative of q} and \cref{eq:derivative of q element are positive} we obtain that
	\begin{align*}
		-\int_{x}^{1}\frac{\partial \frac{\qg{z} - 1}{1-z}}{\partial z} dz
		&\le -\int_{x}^{1} \eta \sm[2]{z} \rb{ \eta \rb{1-2z} + 2 z \frac{2}{\sm[1]{z}} }dz\\
		&= -\int_{x}^{1} \eta \rb{ \eta \rb{1-2z}\sm[2]{z} + 2 z \frac{2\sm[2]{z}}{\sm[1]{z}} }dz
		\,.
	\end{align*}
	Therefore, as from \cref{lem:derivative quasi-static} we get that $\sm[2]{z}$ is increasing for $z>0$, and that $1-2z \le 1- \frac{2}{2} \le 0$ (\cref{eq:oscilate}), then
	\begin{align*}
		-\int_{x}^{1}\frac{\partial \frac{\qg{z} - 1}{1-z}}{\partial z} dz
		&\le -\int_{x}^{1} \eta \rb{ \eta \rb{1-2z}\sm[2]{1} + 2 z \frac{2\sm[2]{z}}{\sm[1]{z}} }dz
		\,.
	\end{align*}
	Therefore, using \cref{lem:s2 over s_1} we get that
	\begin{align*}
		-\int_{x}^{1}\frac{\partial \frac{\qg{z} - 1}{1-z}}{\partial z} dz
		&\le -\int_{x}^{1} \eta \rb{ \eta \rb{1-2z}\sm[2]{1} + 2 z \frac{2\sm[2]{1}}{\sm[1]{1}} }dz\\
		&= -\sm[2]{1}\eta^2\int_{x}^{1} \rb{ \rb{1-2z} + 2 z \frac{2}{\eta\sm[1]{1}} }dz
		\,.
	\end{align*}
	Thus, by using \cref{eq:sharpness at opt},
	\begin{align}
		\label{eq:qs gd conv eq4}
		-\int_{x}^{1}\frac{\partial \frac{\qg{z} - 1}{1-z}}{\partial z} dz
		&\le -\sm[2]{1}\eta^2\int_{x}^{1} \rb{ 1 + 2 z \rb{ \frac{2}{c} - 1 }}dz
		\,.
	\end{align}
	
	Because \cref{eq:oscilate} and \cref{eq:conv rate 1 step} there exist $y\in\R$ such that $x\le y \le 1$ and $1+\rb{c-1}\rb{1-y}=\qg{x}$.
	Using \cref{eq:qs gd conv eq3} and \cref{eq:qs gd conv eq4}, we obtain
	\begin{align*}
		-\int_{y}^{1}\frac{\partial \frac{\qg{z} - 1}{1-z}}{\partial z} dz + \int_{1}^{1+\rb{c-1}\rb{1-y}}\frac{\partial \frac{\qg{z} - 1}{1-z}}{\partial z} dz
		&\le -\sm[2]{1}\eta^2\int_{y}^{1} \rb{ 1 + 2 z \rb{ \frac{2}{c} - 1 }}dz\\
		&\qquad+ \eta^2\sm[2]{1} \int_{1}^{1+\rb{c-1}\rb{1-y}} \rb{ 1 + 2z\rb{\frac{2}{c} - 1} }dz
		\,.
	\end{align*}
	Therefore,
	\begin{equation}
		\label{eq:integral with y}
		\begin{aligned}
			-\int_{y}^{1}\frac{\partial \frac{\qg{z} - 1}{1-z}}{\partial z} dz &+ \int_{1}^{1+\rb{c-1}\rb{1-y}}\frac{\partial \frac{\qg{z} - 1}{1-z}}{\partial z} dz\\
			&\le \sm[2]{1}\eta^2\Bigg( -\int_{y}^{1} \rb{ 1 + 2 z \rb{ \frac{2}{c} - 1 }}dz\\
			&\qquad	+ \int_{1}^{1+\rb{c-1}\rb{1-y}} \rb{ 1 + 2z\rb{\frac{2}{c} - 1} }dz \Bigg)
			\,.
		\end{aligned}
	\end{equation}
	
	We have
	\begin{align*}
		-\int_{y}^{1} & \rb{ 1 + 2 z \rb{ \frac{2}{c} - 1 }}dz + \int_{1}^{1+\rb{c-1}\rb{1-y}} \rb{ 1 + 2z\rb{\frac{2}{c} - 1} }dz\\
		&= \left( \left( 1 + (1-y)(c-1) \right) + \left( 1 + (1-y)(c-1) \right)^2 \left( \frac{2}{c} - 1 \right) \right)\\
		&\qquad+ \left( y +  y^2 \left( \frac{2}{c} - 1 \right) \right) - 2 \left( 1 + \left( \frac{2}{c} - 1 \right) \right)\\
	\end{align*}
	Because
	\begin{align*}
		\left( 1 + (1-y)(c-1) \right) + y - 2
		&= 1 - 1 + c - yc + y + y - 2
		= (2-c)(y - 1)
		\,,
	\end{align*}
	we obtain that
	\begin{equation}
		\label{eq:integrals sum}
		\begin{aligned}
			-\int_{y}^{1} & \rb{ 1 + 2 z \rb{ \frac{2}{c} - 1 }}dz + \int_{1}^{1+\rb{c-1}\rb{1-y}} \rb{ 1 + 2z\rb{\frac{2}{c} - 1} }dz\\
			&= (2-c)(y - 1) + \left( 1 + (1-y)(c-1) \right)^2 \left( \frac{2}{c} - 1 \right)
			+ y^2 \left( \frac{2}{c} - 1 \right) - 2 \left( \frac{2}{c} - 1 \right)\\
			&= \frac{2-c}{c} \left( cy - c + \left( 1 + (1-y)(c-1) \right)^2 + y^2 - 2\right)
			\,.
		\end{aligned}
	\end{equation}
	We have
	\begin{equation*}
		\begin{aligned}
			cy - c + \left( 1 + (1-y)(c-1) \right)^2 + y^2 - 2
			&= cy - c + \left( y + c(1-y) \right)^2 + y^2 -2\\
			&= cy - c + c^2(1-y)^2 + 2c(1-y)y + 2y^2 - 2\\
			&= cy - c + c^2 -2c^2y +  c^2y^2 + 2cy - 2cy^2 + 2y^2 - 2\\
			&= c\left( y - 1 + c -2cy +  cy^2 + 2y - 2y^2 \right) + 2 (y - 1)(y + 1)\\
			&= c \left( 1 - y \right)\left( - 1 + c - cy + 2y \right) + 2 (y - 1)(y + 1)
			\,.
		\end{aligned}
	\end{equation*}
	And because $0<c\le2$ and $0<y\le1$ we get that
	\begin{equation*}
		\begin{aligned}
			cy - c + \left( 1 + (1-y)(c-1) \right)^2 + y^2 - 2
			&\le c \left( 1 - y \right)\left( - 1 + c - cy + 2y \right) + c (y - 1)(y + 1)\\
			&= c \left( 1 - y \right)\left( - 1 + c - cy + 2y - y - 1 \right)\\
			&= c \left( 1 - y \right)\left( - 2 + c - cy + y \right)\\
			&\le c \left( 1 - y \right)\left( - cy + y \right)\\
			&\le 0
			\,.
		\end{aligned}
	\end{equation*}
	Therefore, using \cref{eq:integrals sum},
	\begin{align*}
		-\int_{y}^{1} & \rb{ 1 + 2 z \rb{ \frac{2}{c} - 1 }}dz + \int_{1}^{1+\rb{c-1}\rb{1-y}} \rb{ 1 + 2z\rb{\frac{2}{c} - 1} }dz\\
		&= \frac{2-c}{c} \left( cy - c + \left( 1 + (1-y)(c-1) \right)^2 + y^2 - 2\right)
		\le 0
		\,.
	\end{align*}
	Thus, by using \cref{eq:integral with y} we obtain that
	\begin{align*}
		-\int_{y}^{1}\frac{\partial \frac{\qg{z} - 1}{1-z}}{\partial z} dz + \int_{1}^{1+\rb{c-1}\rb{1-y}}\frac{\partial \frac{\qg{z} - 1}{1-z}}{\partial z} dz
		&\le 0
		\,.
	\end{align*}
	Therefore, by using the fact that $y\ge x$, $\qg{x}=1+\rb{c-1}\rb{1-y}$ and \cref{lem:q rate increases} then
	\begin{align*}
		-\int_{x}^{1}\frac{\partial \frac{\qg{z} - 1}{1-z}}{\partial z} dz &+ \int_{1}^{\qg{x}}\frac{\partial \frac{\qg{z} - 1}{1-z}}{\partial z} dz\\
		&= -\int_{x}^{y}\frac{\partial \frac{\qg{z} - 1}{1-z}}{\partial z} dz -\int_{y}^{1}\frac{\partial \frac{\qg{z} - 1}{1-z}}{\partial z} dz + \int_{1}^{1+\rb{c-1}\rb{1-y}}\frac{\partial \frac{\qg{z} - 1}{1-z}}{\partial z} dz\\
		&\le -\int_{y}^{1}\frac{\partial \frac{\qg{z} - 1}{1-z}}{\partial z} dz + \int_{1}^{1+\rb{c-1}\rb{1-y}}\frac{\partial \frac{\qg{z} - 1}{1-z}}{\partial z} dz\\
		&\le 0
		\,.
	\end{align*}
	Therefore,
	\begin{align}
		\label{eq:two step q elem2}
		-\int_{x}^{1}\frac{\partial \frac{\qg{z} - 1}{1-z}}{\partial z} dz+ \int_{1}^{\qg{x}}\frac{\partial \frac{\qg{z} - 1}{1-z}}{\partial z} dz
		&\le 0
		\,.
	\end{align}
	
	Combining the result of \cref{eq:two step q elem3} and \cref{eq:two step q elem2} into \cref{eq:two step of q} we obtain that
	\begin{align}
		\label{eq:qs gd conv eq5}
		\frac{\qg{\qg{x}} - 1}{x-1}
		\le \rb{c-1}^2
		\,.
	\end{align}
	
	Therefore, both in \cref{eq:qs gd conv eq2} and \cref{eq:qs gd conv eq5}, i.e. all cases, we obtain that
	\begin{align}
		\label{eq:two step of q upper bound}
		\frac{\qg{\qg{x}} - 1}{x-1}
		\le \rb{c-1}^2
		\,.
	\end{align}
	
	\hyperref[step:3.2]{\textbf{Step~\labelcref*{step:3.2}:}}
	We now show that
	\begin{align*}
		\frac{\qg{\qg{x}} - 1}{x-1}
		\ge 0
		\,.
	\end{align*}
	Let's assume toward contradiction that
	\begin{align*}
		\frac{\qg{\qg{x}} - 1}{x-1}
		< 0
		\,.
	\end{align*}
	Therefore, as result of \cref{eq:oscilate},
	\begin{align*}
		\qg{\qg{x}}  > 1
		\,.
	\end{align*}
	Define
	\begin{align*}
		z_0 = \min \left\{ z\in\R ~|~ z>1\wedge \qg{z}=0 \right\}
		\,.
	\end{align*}
	If
	\begin{align*}
		z_0 \ge \qg{x} \ge 1
	\end{align*}
	then, as result of \cref{lem:g_i decrease}, we obtain
	\begin{align*}
		0 = \qg{z_0} \le \qg{\qg{x}} \le \qg{1} = 1
		\,,
	\end{align*}
	which contradict that $\qg{\qg{x}} > 1$.
	Therefore, as $\qg{x} \ge 1$ (\cref{eq:oscilate}), we obtain
	\begin{align*}
		1 < z_0 < \qg{x}
		\,.
	\end{align*}
	Therefore, there exist $x'\in\R$ such that $\qg{x'} = z_0$ and $x<x'<1$.
	Therefore, as $x' > x \ge \frac{1}{2}$ and $1\le c<2$,
	\begin{align*}
		\frac{\qg{\qg{x'}} - 1}{x-1}
		&= \frac{\qg{z_0} - 1}{x-1}
		= \frac{0 - 1}{x-1}
		=\frac{1}{1-x}
		\ge \frac{1}{1-\frac{1}{2}}
		= 2
		> \rb{c-1}^2
		\,,
	\end{align*}
	which contradicts \cref{eq:two step of q upper bound}.
	
	Consequently, using \cref{eq:two step of q upper bound}, we obtain
	\begin{align*}
		0\le
		\frac{\qg{\qg{x}} - 1}{x-1}
		\le \rb{c-1}^2
		\,.
	\end{align*}
	Thus, as $x\in(0,1]$,
	\begin{align}
		\label{eq:two step of q bound}
		0 \le 1 - \qg{\qg{x}} \le \rb{c-1}^2 \rb{1-x}
		\,.
	\end{align}
	
	\hyperref[step:4]{\textbf{Step~\labelcref*{step:4}:}}
	
	From \cref{eq:conv rate 1 step} and \cref{eq:two step of q bound}, we obtain that for $t\in\{0,1,2\}$
	\begin{align*}
		0 \le \rb{1-\qgs^{\circ t}\rb{x}} \rb{-1}^t \le \rb{1-\delta}^{t}\rb{1 - x}
		\,.
	\end{align*}
	Therefore,
	\begin{align*}
		1 - 0.1 \frac{\delta}{1 - \delta} \le x \le \qgs^{\circ 2}\rb{x} \le 1
		\,.
	\end{align*}
	As a consequence, by using induction, we obtain that for every $t\ge0$
	\begin{align*}
		0 \le \rb{1-\qgs^{\circ t}\rb{x}} \rb{-1}^t \le \rb{1-\delta}^{t}\rb{1 - x}
		\,.
	\end{align*}
	
	Therefore, for all $t\ge0$:
	\begin{align*}
		\loss\rb{\vect{w}\rb{\qgs^{\circ t}\rb{x}}} \le \rb{1-\delta}^{2t} \loss\rb{\vect{w}(x)}
		\,,
	\end{align*}
	for even $t\ge0$: $\qgs^{\circ t}\rb{x} \le 1$,
	and for odd $t>0$: $\qgs^{\circ t}\rb{x} \ge 1$.
\end{proof}

\subsubsection{Proof of Lemma \ref{lem:derivative quasi-static}}
\label{proof:derivative quasi-static}
\begin{proof}
	We have
	\begin{align*}
		\prod_{i=1}^{D} \rb{ w^2_1(x) + b_{i,1} } = \prod_{i=1}^{D} w^2_i(x) = \pi^2\rb{\vect{w}(x)} = x^2
		\,.
	\end{align*}
	Therefore, by performing derivative on both sides, we obtain that 
	\begin{align*}
		2x
		&= \sum_{i=1}^{D} \frac{\partial w^2_1(x)}{\partial x} \prod_{j\in \D-\{i\}}\rb{ w^2_1(x) + b_{j,1} }\\
		&= \frac{\partial w^2_1(x)}{\partial x} \sum_{i=1}^{D} \prod_{j\in \D-\{i\}}w^2_j(x)\\
		&= \frac{\partial w^2_1(x)}{\partial x} \pi^2\rb{\vect{w}(x)} \sum_{i=1}^{D} \frac{1}{w^2_i(x)}\\
		&= \frac{\partial w^2_1(x)}{\partial x} \sm[1]{x}
		\,.
	\end{align*}
	Therefore, for every $i\in\D$,
	\begin{align}
		\label{eq:derivative quasi-static eq1}
		\frac{\partial w^2_i(x)}{\partial x} = \frac{\partial w^2_1(x)}{\partial x} = \frac{2x}{\sm[1]{x}}
		\,.
	\end{align}

	Therefore, for every $m\in\D[D-1]\cup\{0\}$,
	\begin{align}
		\frac{\partial \sm{x}}{\partial x}
		&= \sum_{\substack{I=\text{ subset of } \\ D-m \text{ different} \\ \text{indices from }[D]}}\sum_{i\in I} \frac{\partial w^2_i(x)}{\partial x} \prod_{j\in I-\{i\}} {w_j}^2\nonumber\\
		&= \frac{2x}{\sm[1]{x}} \sum_{\substack{I=\text{ subset of } \\ D-m \text{ different} \\ \text{indices from }\D}}\sum_{i\in I} \prod_{j\in I-\{i\}} {w_j}^2 \label{eq:s_m derivative 1}
		\,.
	\end{align}

	There are $\binom{D}{D-m}$ subset of $D-m$ different indices from $\D$.
	Therefore,
	\begin{align*}
		\sum_{\substack{I=\text{ subset of } \\ D-m \text{ different} \\ \text{indices from }\D}}\sum_{i\in I} \prod_{j\in I-\{i\}} {w_j}^2
	\end{align*}
	have $\rb{D-m}\binom{D}{D-m}$ elements in the summation.
	There are $\binom{D}{D-m-1}$ subset of $D-m-1$, therefore
	\begin{align*}
		\sm[m+1]{x} = \sum_{\substack{I=\text{ subset of } \\ D-m-1 \text{ different} \\ \text{indices from }\D}}\prod_{i\in I} {w_i}^2
	\end{align*}
	have $\binom{D}{D-m-1}$ element in the summation.
	
	Therefore, using \cref{eq:s_m derivative 1}, we obtain that for every $m\in\D[D-1]\cup\{0\}$
	\begin{align*}
		\frac{\partial \sm{x}}{\partial x}
		&= \frac{2x}{\sm[1]{x}} \sum_{\substack{I=\text{ subset of } \\ D-m \text{ different} \\ \text{indices from }\D}}\sum_{i\in I} \prod_{j\in I-\{i\}} {w_j}^2\\
		&= \frac{2x}{\sm[1]{x}} \cdot \frac{\rb{D-m}\binom{D}{D-m}}{\binom{D}{D-m-1}} \sm[m+1]{x}\\
		&= \frac{2x}{\sm[1]{x}} \cdot \rb{D-m} \frac{D - \rb{D-m} + 1 }{D-m} \cdot \frac{\binom{D}{D-m-1}}{\binom{D}{D-m-1}} \sm[m+1]{x}\\
		&= 2\rb{m+1}x\frac{\sm[m+1]{x}}{\sm[1]{x}} 
		\,.
	\end{align*}
	Thus
	\begin{align}
		\label{eq:derivative quasi-static eq2}
		\frac{\partial \sm{x}}{\partial x}
		&= 2\rb{m+1}x\frac{\sm[m+1]{x}}{\sm[1]{x}}
		\,.
	\end{align}

	In addition,
	\begin{align}
		\label{eq:derivative quasi-static eq3}
		\frac{\partial \sm[D]{x}}{\partial x} = \frac{\partial 1}{\partial x} = 0
		\,.
	\end{align}
	
	Combining \cref{eq:derivative quasi-static eq1}, \cref{eq:derivative quasi-static eq2} and \cref{eq:derivative quasi-static eq3}, we obtain
	\begin{align*}
		\frac{\partial w^2_i(x)}{\partial x} &= \frac{2x}{\sm[1]{x}} && \forall i \in \D \,,\\
		\frac{\partial \sm{x}}{\partial x} &= 2\rb{m+1}x\frac{\sm[m+1]{x}}{\sm[1]{x}}  && \forall m\in\D[D-1]\cup\{0\} \,\text{and}\\
		\frac{\partial \sm[D]{x}}{\partial x} &= 0
		\,.
	\end{align*}
\end{proof}

\subsubsection{Proof of Lemma \ref{lem:s2 over s_1}}
\label{proof:s2 over s_1}
\begin{proof}
	If $D=2$ then $\sm[2]{x}=1$.
	Therefore, we get that
	\begin{align*}
		\frac{\sm[2]{x}}{\sm[1]{x}} = \frac{1}{\sm[1]{x}}
	\end{align*}
	is decreasing from \cref{lem:derivative quasi-static}.
	
	We now handle the case that $D\ge 3$.
	
	We have
	\begin{align*}
		\frac{\partial \frac{\sm[2]{x}}{\sm[1]{x}}}{\partial x}
		&= \frac{\frac{\partial \sm[2]{x}}{\partial x}}{\sm[1]{x}} - \frac{\sm[2]{x} \frac{\partial \sm[1]{x}}{\partial x} }{\smp[1]{x}{2}}
		\,.
	\end{align*}
	Using \cref{lem:derivative quasi-static} we obtain that
	\begin{align*}
		\frac{\partial \frac{\sm[2]{x}}{\sm[1]{x}}}{\partial x}
		&= 6x\frac{\sm[3]{x}}{\smp[1]{x}{2}} - 4x\frac{\smp[2]{x}{2} }{\smp[1]{x}{3}}
		\,.
	\end{align*}
	Therefore,
	\begin{align}
		\label{eq:s2 over s_1 eq1}
		\frac{\partial \frac{\sm[2]{x}}{\sm[1]{x}}}{\partial x}
		&= \frac{2x}{\smp[1]{x}{3}} \rb{ 3\sm[3]{x}\sm[1]{x} - 2\smp[2]{x}{2} }
		\,.
	\end{align}

	We will now show that
	\begin{align*}
		3\sm[3]{x}\sm[1]{x} - 2\smp[2]{x}{2} < 0
		\,.
	\end{align*}
	We have
	\begin{align*}
		3\sm[3]{x}\sm[1]{x} - 2\smp[2]{x}{2}
		= x^4 \rb{ 3 \sum_{\substack{i,j,k\in\D \\ i<j<k}} \frac{1}{w^2_i\rb{x} w^2_j\rb{x} w^2_k\rb{x}} \sum_{i\in\D} \frac{1}{w^2_i\rb{x}} - 2 \sum_{\stackrel{i,j\in\D}{i<j}} \frac{1}{w^2_i\rb{x} w^2_j\rb{x}} }
		\,.
	\end{align*}
	The summation
	\begin{align*}
		 3 \sum_{\substack{i,j,k\in\D \\ i<j<k}} \frac{1}{w^2_i\rb{x} w^2_j\rb{x} w^2_k\rb{x}} \sum_{i\in\D} \frac{1}{w^2_i\rb{x}}
	\end{align*}
	have only elements of the forms
	\begin{align*}
		\frac{1}{w^2_i\rb{x} w^2_j\rb{x} w^2_k\rb{x} w^2_m\rb{x}} ~ \text{and} ~ \frac{1}{w^2_i\rb{x} w^2_i\rb{x} w^2_j\rb{x} w^2_k\rb{x}}
		\,,
	\end{align*}
	for some $i,j,k,m\in\D$ such that $i,j,k,m$ are different from each other.
	For  $i,j,k,m\in\D$ such that $i,j,k,m$ are different from each other (only possible for $D\ge 4$) then
	\begin{align*}
		\frac{1}{w^2_i\rb{x} w^2_j\rb{x} w^2_k\rb{x} w^2_m\rb{x}}
	\end{align*}
	appears in
	\begin{align*}
		3 \sum_{\substack{i,j,k\in\D \\ i<j<k}} \frac{1}{w^2_i\rb{x} w^2_j\rb{x} w^2_k\rb{x}} \sum_{i\in\D} \frac{1}{w^2_i\rb{x}}
	\end{align*}
	$3\binom{4}{1}=12$ times,
	and in
	\begin{align*}
		2 \sum_{\stackrel{i,j\in\D}{i<j}} \frac{1}{w^2_i\rb{x} w^2_j\rb{x}}
	\end{align*}
	it appears $2 \binom{4}{2} = 12$ times.
	For  $i,j,k\in\D$ such that $i,j,k$ are different from each other then
	\begin{align*}
		\frac{1}{w^2_i\rb{x} w^2_i\rb{x} w^2_j\rb{x} w^2_k\rb{x}}
	\end{align*}
	appears in
	\begin{align*}
		3 \sum_{\substack{i,j,k\in\D \\ i<j<k}} \frac{1}{w^2_i\rb{x} w^2_j\rb{x} w^2_k\rb{x}} \sum_{i\in\D} \frac{1}{w^2_i\rb{x}}
	\end{align*}
	$3$ times,
	and in
	\begin{align*}
		2 \sum_{\stackrel{i,j\in\D}{i<j}} \frac{1}{w^2_i\rb{x} w^2_j\rb{x}}
	\end{align*}
	it appears $2 \binom{2}{1} = 4$ times.
	Therefore, every element in
	\begin{align*}
		3 \sum_{\substack{i,j,k\in\D \\ i<j<k}} \frac{1}{w^2_i\rb{x} w^2_j\rb{x} w^2_k\rb{x}} \sum_{i\in\D} \frac{1}{w^2_i\rb{x}}
	\end{align*}
	is also an element in
	\begin{align*}
		2 \sum_{\stackrel{i,j\in\D}{i<j}} \frac{1}{w^2_i\rb{x} w^2_j\rb{x}}
		\,.
	\end{align*}
	Therefore,
	\begin{align*}
		3\sm[3]{x}\sm[1]{x} - 2\smp[2]{x}{2}
		&= x^4 \rb{ 3 \sum_{\substack{i,j,k\in\D \\ i<j<k}} \frac{1}{w^2_i\rb{x} w^2_j\rb{x} w^2_k\rb{x}} \sum_{i\in\D} \frac{1}{w^2_i\rb{x}} - 2 \sum_{\stackrel{i,j\in\D}{i<j}} \frac{1}{w^2_i\rb{x} w^2_j\rb{x}} }\\
		&\le 0
		\,.
	\end{align*}
	Thus, by using \cref{eq:s2 over s_1 eq1}, we obtain that if $x>0$ then
	\begin{align*}
		\frac{\partial \frac{\sm[2]{x}}{\sm[1]{x}}}{\partial x} \le 0
		\,.
	\end{align*}

	Therefore, for all $D\ge2$, if $x>0$ then
	\begin{align*}
		\frac{\partial \frac{\sm[2]{x}}{\sm[1]{x}}}{\partial x} \le 0
		\,.
	\end{align*}
	Finally, we obtain that the function
	\begin{align*}
		\frac{\sm[2]{x}}{\sm[1]{x}}
	\end{align*}
	is decreasing in $x$, for $x>0$.
\end{proof}

\subsubsection{Proof of Lemma \ref{lem:q rate increases}}
\label{proof:q rate increases}
\begin{proof} From the definition of $\qg{x}$ (\cref{eq:def of q}), we obtain
	\begin{align*}
		\qg{x} - 1 = x-1 + \sum_{m=1}^{D} \eta^m (1-x)^{m} {x}^{m-1} \sm{x}
		\,.
	\end{align*}
	Therefore,
	\begin{align*}
		\frac{\qg{x} - 1}{1-x} = -1 + \sum_{m=1}^{D} \eta^{m} (1-x)^{m-1} {x}^{m-1} \sm{x}
		\,.
	\end{align*}

	Therefore,
	\begin{align*}
		\frac{\partial \frac{\qg{x} - 1}{1-x}}{\partial x}
		&= \sum_{m=1}^{D} \eta^{m} \rb{m-1} (1-x)^{m-2} {x}^{m-2} \rb{1-2x} \sm{x}
		+ \sum_{m=1}^{D} \eta^{m} (1-x)^{m-1} {x}^{m-1} \frac{\partial \sm{x}}{\partial x}
		\,.
	\end{align*}
	By using \cref{lem:derivative quasi-static} we obtain
	\begin{align*}
		\frac{\partial \frac{\qg{x} - 1}{1-x}}{\partial x}
		&= \sum_{m=2}^{D} \eta^{m} \rb{m-1} (1-x)^{m-2} {x}^{m-2} \rb{1-2x} \sm{x}\\
		&\qquad+ 2\sum_{m=1}^{D-1} \eta^{m} \rb{m+1} (1-x)^{m-1} {x}^{m-1} x\frac{\sm[m+1]{x}}{\sm[1]{x}}\\
		&= \sum_{m=2}^{D} \eta^{m} \rb{m-1} (1-x)^{m-2} {x}^{m-2} \rb{1-2x} \sm{x}\\
		&\qquad+ 2\sum_{m=2}^{D} \eta^{m-1} m (1-x)^{m-2} {x}^{m-2} x\frac{\sm[m]{x}}{\sm[1]{x}}\\
		\,.
	\end{align*}
	Thus,
	\begin{align}
		\label{eq:derivative of q}
		\frac{\partial \frac{\qg{x} - 1}{1-x}}{\partial x}
		&= \sum_{m=2}^{D} \eta^{m-1} (1-x)^{m-2} {x}^{m-2} \sm{x} \rb{ \rb{m-1} \eta \rb{1-2x} + m x \frac{2}{\sm[1]{x}} }
		\,.
	\end{align}

	From \cref{lem:derivative quasi-static} we have $\frac{\partial \sm{x}}{\partial x} > 0$ ,therefore ,$\sm[1]{x}$ is raising in $x$ over the region $(0,1]$.
	Therefore, for all $m\in\D-\{1\}$
	\begin{align*}
		\rb{m-1} \eta \rb{1-2x} + m x \frac{2}{\sm[1]{x}}
		&\ge \rb{m-1} \eta \rb{1-2x} + m x \frac{2}{\sm[1]{1}}\\
		&\ge \rb{m-1} \rb{\eta \rb{1-2x} + x \frac{2}{\sm[1]{1}}}\\
		\,.
	\end{align*}
	Using \cref{eq:sharpness at opt}, we obtain that
	\begin{align*}
		\rb{m-1} \eta \rb{1-2x} + m x \frac{2}{\sm[1]{x}}
		&\ge \rb{m-1} \rb{\eta \rb{1-2x} + x \frac{2}{\sm[1]{1}}}\\
		&= \rb{m-1} \rb{\eta \rb{1-2x} + x \frac{2}{\psc}}\\
		\,.
	\end{align*}
	Therefore, if $\psc\le\frac{2}{\eta}$ and $x\in(0,1]$, then
	\begin{align*}
		\rb{m-1} \eta \rb{1-2x} + m x \frac{2}{\sm[1]{x}}
		&= \rb{m-1} \rb{\eta \rb{1-2x} + x \frac{2}{\psc}}\\
		&\ge \rb{m-1} \rb{\eta \rb{1-2x} + \eta x }\\
		&= \rb{m-1} \eta \rb{1-x}\\
		&\ge 0
		\,.
	\end{align*}
	
	Thus, if $\psc\le\frac{2}{\eta}$ and $x\in(0,1]$, then
	\begin{equation}
		\label{eq:derivative of q element are positive}
		\begin{aligned}
			\eta^{m-1} (1-x)^{m-2} {x}^{m-2} \sm{x} & \rb{ \rb{m-1} \eta \rb{1-2x} + m x \frac{2}{\sm[1]{x}} }\\
			&\ge \eta^{m-1} (1-x)^{m-2} {x}^{m-2} \sm{x} \cdot 0\\
			&\ge 0
			\,.
		\end{aligned}
	\end{equation}
	Therefore, using \cref{eq:derivative of q}, we obtain that
	\begin{align*}
		\frac{\partial \frac{\qg{x} - 1}{1-x}}{\partial x}
		&= \sum_{m=2}^{D} \eta^{m-1} (1-x)^{m-2} {x}^{m-2} \sm{x} \rb{ \rb{m-1} \eta \rb{1-2x} + m x \frac{2}{\sm[1]{x}} }\\
		&\ge 0
		\,.
	\end{align*}
	Therefore, if $\psc\le\frac{2}{\eta}$, then the function
	\begin{align*}
		\frac{\qg{x} - 1}{1-x}
	\end{align*}
	is raising in $x$ over the region $(0,1]$.
\end{proof}

\subsubsection{Proof of Lemma \ref{lem:derivative upper bound}}
\label{proof:derivative upper bound}
\begin{proof}
	Define $c\triangleq\eta \psc$ and $\delta \triangleq 2 - c$.
	Let $\delta\in [0, 0.5]$ and $1\le x \le 1 + 0.1 \delta$.

	From \cref{Fig: lambda assumptions} we obtain that
	\begin{align}
		0.1 \delta & \le -1 + \frac{2-\delta + \sqrt{ \rb{2-\delta}^2 + 16\rb{2-\delta}}}{4\rb{2-\delta}} = -1 + \frac{c + \sqrt{ c^2 + 16c}}{4c} \,,\text{and} \label{eq: fig bound 1}\\
		0.1 \delta & \le \frac{- \rb{3 + \delta} + \sqrt{\rb{3+\delta}^2+4 \delta^2 }}{4\delta} = \frac{- \rb{5 -  c} + \sqrt{\rb{5 -  c}^2+4 \rb{2 - c}^2 }}{4\rb{2 - c}} \label{eq: fig bound 2}
		\,.
	\end{align}
	Thus, as $x \le 1 + 0.1 \delta$, we obtain that
	\begin{align}
		x & \le -1 + \frac{c + \sqrt{ c^2 + 16c}}{4c} \,,\text{and} \label{eq:derivative upper bound eq num1 b1}\\
		x & \le \frac{- \rb{5 -  c} + \sqrt{\rb{5 -  c}^2+4 \rb{2 - c}^2 }}{4\rb{2 - c}} \label{eq:derivative upper bound eq num1 b2}
		\,.
	\end{align}

	\begin{figure*}[t]
		\centering
		\subfloat{\includegraphics[width=0.5\textwidth]{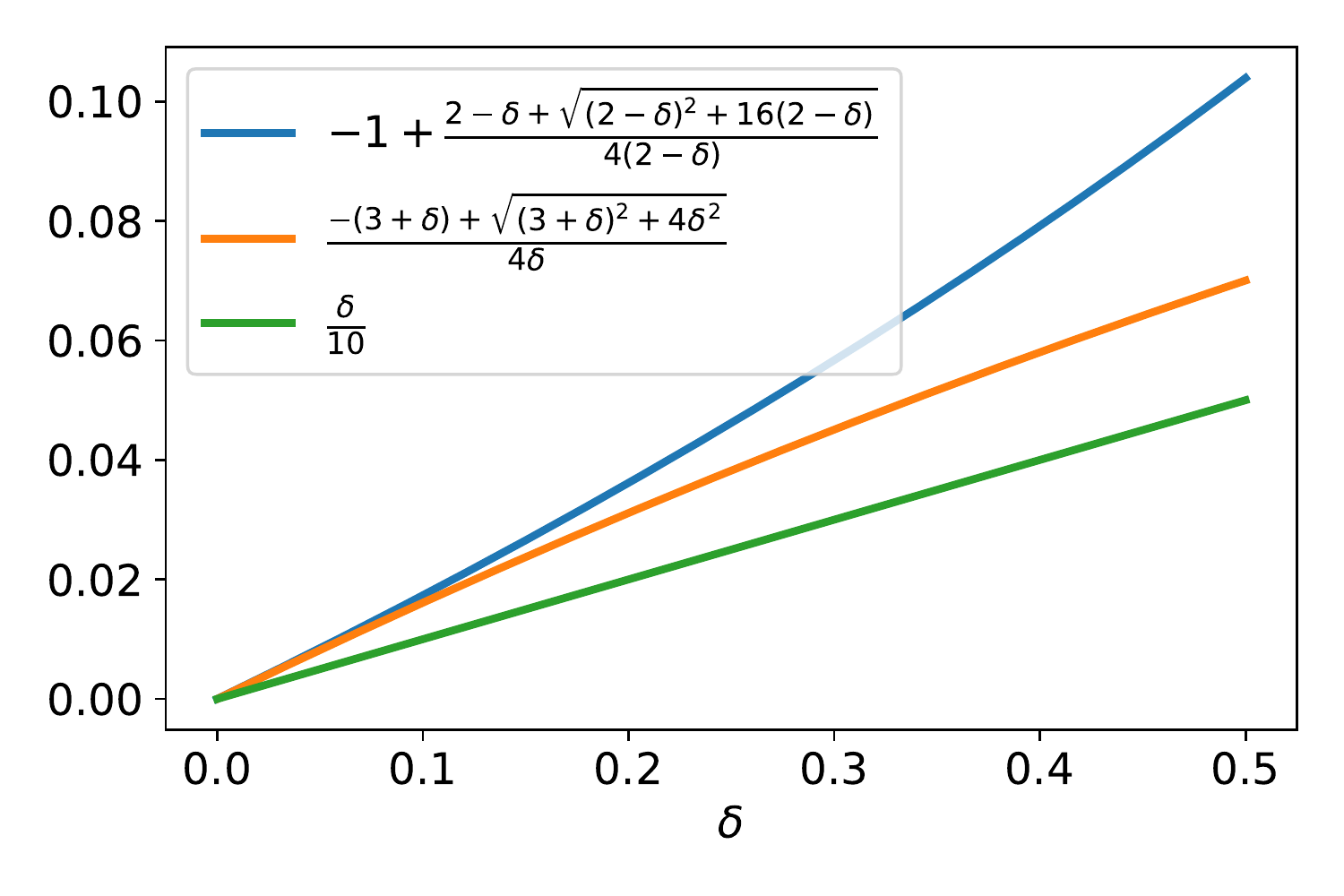}}
		\caption{
			Illustration of the bounds in  \cref{eq: fig bound 1,eq: fig bound 2}.}
		\label{Fig: lambda assumptions}
	\end{figure*}

	First, we show that for every $m\in\D-\{1\}$
	\begin{align*}
		\rb{m-1} \eta \rb{1-2x} + m x \frac{2}{\sm[1]{x}} \ge 0
		\,.
	\end{align*}

	From \cref{lem:derivative quasi-static} we obtain that
	\begin{align*}
		\frac{\sm[1]{x}}{x^2} = \sum_{i=1}^{D} \frac{1}{w_i^2(x)}
	\end{align*}
	is decreasing in $x$ for $x>0$.
	Therefore, as $x\ge1$,
	\begin{align*}
		\rb{m-1} \eta \rb{1-2x} + m x \frac{2}{\sm[1]{x}}
		&= \rb{m-1} \eta \rb{1-2x} + m \frac{2}{x}\cdot \frac{x^2}{\sm[1]{x}}\\
		&\ge \rb{m-1} \eta \rb{1-2x} + m \frac{2}{x\sm[1]{1}}\\
		&= \rb{m-1} \eta \rb{1-2x} + m \frac{2}{x\psc}\\
		&= \rb{m-1} \eta \rb{1-2x}\frac{c}{\psc} + m \frac{2}{x\psc}
		\,.
	\end{align*}
	Thus,
	\begin{align}
		\label{eq:negative derivative eq1}
		\rb{m-1} \eta \rb{1-2x} + m x \frac{2}{\sm[1]{x}}
		&\ge \frac{1}{\psc} \rb{ \rb{m-1}\rb{1-2x} c + 2m \frac{1}{x} }
		\,.
	\end{align}

	We will now solve
	\begin{align*}
		\rb{m-1}\rb{1-2x} c + 2m \frac{1}{x} = 0
		\,.
	\end{align*}
	This will be equal to $0$, if and only if
	\begin{align*}
		0 &= c\rb{m-1}\rb{1-2x}x + 2m\\
		&= -2c\rb{m-1}x^2 + c\rb{m-1}x + 2m
		\,.
	\end{align*}
	Therefore, the solutions to this equation are
	\begin{align*}
		\frac{-c\rb{m-1} \pm \sqrt{ c^2\rb{m-1}^2 + 16c\rb{m-1}m}}{-4c\rb{m-1}}
		&= \frac{c\rb{m-1} \pm \sqrt{ c^2\rb{m-1}^2 + 16c\rb{m-1}m}}{4c\rb{m-1}}
		\,.
	\end{align*}
	Thus, as $c>0$, the only solution above 1 is
	\begin{align*}
		\frac{c\rb{m-1} + \sqrt{ c^2\rb{m-1}^2 + 16c\rb{m-1}m}}{4c\rb{m-1}}
		\,.
	\end{align*}
	We have that
	\begin{align*}
		\frac{c\rb{m-1} + \sqrt{ c^2\rb{m-1}^2 + 16c\rb{m-1}m}}{4c\rb{m-1}}
		&\ge \frac{c\rb{m-1} + \sqrt{ c^2\rb{m-1}^2 + 16c\rb{m-1}^2}}{4c\rb{m-1}}\\
		&= \frac{c + \sqrt{ c^2 + 16c}}{4c}
		\,.
	\end{align*}
	Therefore, as for $x=1$ we have
	\begin{align*}
		\rb{m-1}\rb{1-2x} c + 2m \frac{1}{x}
		&> c\rb{m-1}\rb{1-2x} + c\rb{m-1} \frac{1}{x}\\
		&= c\rb{m-1}\rb{ \rb{1-2} + 1 }\\
		&= 0\,
	\end{align*}
	we obtain that
	\begin{align*}
		\rb{m-1}\rb{1-2x} c + 2m \frac{1}{x} \ge 0
	\end{align*}
	for all $x$ such that
	\begin{align*}
		1 \le x \le \frac{c + \sqrt{ c^2 + 16c}}{4c}
		\,.
	\end{align*}
	Therefore, from \cref{eq:negative derivative eq1}, and using \cref{eq:derivative upper bound eq num1 b1}, we obtain that
	\begin{align}
		\label{eq:negative derivative eq2}
		\rb{m-1} \eta \rb{1-2x} + m x \frac{2}{\sm[1]{x}}
		&\ge 0
		\,.
	\end{align}

	It is easy to see that for every $m\in\D\cap\{0\}$ then
	\begin{align}
		\label{eq:s_m+1 to s_m}
		\sm[m+1]{x} \ge \sm{x} \sm[1]{x} \frac{1}{x^2}
	\end{align}
	
	Let $m\in\D-\{1\}$ be an odd number.
	Because $m$ is odd, and \cref{eq:negative derivative eq2}, then 
	\begin{align*}
		\eta^{m} (1-x)^{m-1} {x}^{m-1} \rb{ m \eta \rb{1-2x} + \rb{m+1} x \frac{2}{\sm[1]{x}} } \ge 0
		\,.
	\end{align*}
	Therfore, By summing the $m$ and $m+1$ elements of \cref{eq:derivative of q}, and using \cref{eq:s_m+1 to s_m}, we obtain
	\begin{align*}
		\eta^{m-1} & (1-x)^{m-2} {x}^{m-2} \sm{x} \rb{ \rb{m-1} \eta \rb{1-2x} + m x \frac{2}{\sm[1]{x}} }\\
		&\qquad+ \eta^{m} (1-x)^{m-1} {x}^{m-1} \sm[m+1]{x} \rb{ m \eta \rb{1-2x} + \rb{m+1} x \frac{2}{\sm[1]{x}} }\\
		&\ge \eta^{m-1} (1-x)^{m-2} {x}^{m-2} \sm{x} \rb{ \rb{m-1} \eta \rb{1-2x} + m x \frac{2}{\sm[1]{x}} }\\
		&\qquad+ \eta^{m} (1-x)^{m-1} {x}^{m-1} \sm{x} \frac{\sm[1]{x} }{x^2} \rb{ m \eta \rb{1-2x} + \rb{m+1} x \frac{2}{\sm[1]{x}} }\\
		&= -\eta^{m-1} (1-x)^{m-2} {x}^{m-2} \sm{x} \Bigg( -\rb{m-1} \eta \rb{1-2x} - m x \frac{2}{\sm[1]{x}}\\
		&\qquad+ \rb{ m \eta \rb{1-2x} + \rb{m+1} x \frac{2}{\sm[1]{x}} } \eta \sm[1]{x} \rb{1 - \frac{1}{x}} \Bigg)
		\,.
	\end{align*}
	As $m$ is odd, and $x\ge1$ then
	\begin{align*}
		-\eta^{m-1} (1-x)^{m-2} {x}^{m-2} \sm{x} \ge 0
		\,.
	\end{align*}
	Therefore,
	\begin{align*}
		0 &\ge \eta^{m-1} (1-x)^{m-2} {x}^{m-2} \sm{x} \rb{ \rb{m-1} \eta \rb{1-2x} + m x \frac{2}{\sm[1]{x}} }\\
		&\qquad+ \eta^{m} (1-x)^{m-1} {x}^{m-1} \sm[m+1]{x} \rb{ m \eta \rb{1-2x} + \rb{m+1} x \frac{2}{\sm[1]{x}} } 
	\end{align*}
	if
	\begin{align}
		\label{eq:negative derivative eq3}
		-\rb{m-1} \eta \rb{1-2x} - m x \frac{2}{\sm[1]{x}} + \rb{ m \eta \rb{1-2x} + \rb{m+1} x \frac{2}{\sm[1]{x}} } \eta \sm[1]{x} \rb{1 - \frac{1}{x}} \le 0
		\,.
	\end{align}

	We have that
	\begin{align*}
		-\rb{m-1} & \eta \rb{1-2x} - m x \frac{2}{\sm[1]{x}} + \rb{ m \eta \rb{1-2x} + \rb{m+1} x \frac{2}{\sm[1]{x}} } \eta \sm[1]{x} \rb{1 - \frac{1}{x}}\\
		&= -\rb{m-1} \eta \rb{1-2x} - 2m \frac{x^2}{x\sm[1]{x}} + \rb{ m \eta^2 \rb{1-2x}\sm[1]{x} + 2\eta \rb{m+1} x  } \rb{1 - \frac{1}{x}}
	\end{align*}
	From \cref{lem:derivative quasi-static} we obtain that $\sm[1]{x}$ increase in $x$ if $x>0$, and that $\frac{\sm[1]{x}}{x^2}$ decreases.
	Therefore, as $x\ge 1$,
	\begin{align*}
		-\rb{m-1} & \eta \rb{1-2x} - m x \frac{2}{\sm[1]{x}} + \rb{ m \eta \rb{1-2x} + \rb{m+1} x \frac{2}{\sm[1]{x}} } \eta \sm[1]{x} \rb{1 - \frac{1}{x}}\\
		&\le -\rb{m-1} \eta \rb{1-2x} - 2m \frac{1}{x\sm[1]{1}} + \rb{ m \eta^2 \rb{1-2x}\sm[1]{1} + 2\eta \rb{m+1} x  } \rb{1 - \frac{1}{x}}
		\,.
	\end{align*}
	By using \cref{eq:sharpness at opt} and that $\eta=\frac{c}{\psc}$, we obtain that
		\begin{align*}
		-\rb{m-1} & \eta \rb{1-2x} - m x \frac{2}{\sm[1]{x}} + \rb{ m \eta \rb{1-2x} + \rb{m+1} x \frac{2}{\sm[1]{x}} } \eta \sm[1]{x} \rb{1 - \frac{1}{x}}\\
		&\le -\rb{m-1} \frac{c}{\psc} \rb{1-2x} - 2m \frac{1}{x\psc} + \rb{ m \frac{c^2}{\psc^2} \rb{1-2x}\psc + 2 \frac{c}{\psc} \rb{m+1} x  } \rb{1 - \frac{1}{x}}\\
		&=  \frac{1}{\psc} \rb{-\rb{m-1} c \rb{1-2x} - 2m \frac{1}{x} + \rb{ m c^2 \rb{1-2x} + 2 c \rb{m+1} x  } \rb{1 - \frac{1}{x}}}
		\,.
	\end{align*}
	As $\frac{1}{\psc}>0$, the equation above will be negative if and only if
	\begin{align}
		\label{eq:negative derivative eq4}
		-\rb{m-1} c \rb{1-2x} - 2m \frac{1}{x} + \rb{ m c^2 \rb{1-2x} + 2 c \rb{m+1} x  } \rb{1 - \frac{1}{x}} < 0
		\,.
	\end{align}

	We have that
	\begin{align*}
		-\rb{m-1} & c \rb{1-2x} - 2m \frac{1}{x} + \rb{ m c^2 \rb{1-2x} + 2 c \rb{m+1} x  } \rb{1 - \frac{1}{x}}\\
		&= -\rb{m-1} c + 2x\rb{m-1} c - 2m \frac{1}{x} + \rb{ m c^2 - 2x m c^2 + 2 c \rb{m+1} x  } \rb{1 - \frac{1}{x}}\\
		&= -\rb{m-1} c + 2x\rb{m-1} c - 2m \frac{1}{x} + m c^2 - 2x m c^2 + 2 c \rb{m+1} x\\
		&\qquad- m c^2\frac{1}{x} + 2 m c^2 - 2 c \rb{m+1}\\
		&= \frac{1}{x} \Big( - 2m - m c^2
		+\rb{-\rb{m-1} c + m c^2 + 2 m c^2 - 2 c \rb{m+1}}x\\
		&\qquad+\rb{2\rb{m-1} c - 2 m c^2 + 2 c \rb{m+1}}x^2
		\Big)
		\,.
	\end{align*}
	Define $\varepsilon \triangleq x-1$. We obtain
	\begin{align*}
		&= \frac{1}{x} \Big( - 2m - m c^2
		+\rb{-\rb{m-1} c + m c^2 + 2 m c^2 - 2 c \rb{m+1}}\rb{1 + \varepsilon}\\
		&\qquad+\rb{2\rb{m-1} c - 2 m c^2 + 2 c \rb{m+1}}\rb{1 + \varepsilon}^2
		\Big)\\
		&= \frac{1}{x} \Big( - 2m - m c^2 -\rb{m-1} c + m c^2 + 2 m c^2 - 2 c \rb{m+1} + 2\rb{m-1} c - 2 m c^2 + 2 c \rb{m+1} \\
		&\qquad+\rb{-\rb{m-1} c + m c^2 + 2 m c^2 - 2 c \rb{m+1} + 4\rb{m-1} c - 4 m c^2 + 4 c \rb{m+1}}\varepsilon\\
		&\qquad+\rb{2\rb{m-1} c - 2 m c^2 + 2 c \rb{m+1}}\varepsilon^2
		\Big)\\
		&= \frac{1}{x} \rb{ \rb{- 2m  + \rb{m - 1} c}
			+ \rb{5m - 1 - m c}c\varepsilon
			+\rb{4- 2  c}mc\varepsilon^2 }\\
		&\le \frac{m}{x} \rb{ \rb{- 2 + c}
			+ \rb{5 -  c}c\varepsilon
			+\rb{4- 2  c}c\varepsilon^2 }
		\,.
	\end{align*}

	Recalling \cref{eq:negative derivative eq3} and \cref{eq:negative derivative eq4} ,and because $x>0$, we obtain that if
	\begin{align}
		\label{eq:negative derivative eq5}
		\rb{- 2 + c} + \rb{5 -  c}c\varepsilon +\rb{4- 2  c}c\varepsilon^2 \le 0
	\end{align}
	then
	\begin{align*}
		0 &\ge \eta^{m-1} (1-x)^{m-2} {x}^{m-2} \sm{x} \rb{ \rb{m-1} \eta \rb{1-2x} + m x \frac{2}{\sm[1]{x}} }\\
		&\qquad+ \eta^{m} (1-x)^{m-1} {x}^{m-1} \sm[m+1]{x} \rb{ m \eta \rb{1-2x} + \rb{m+1} x \frac{2}{\sm[1]{x}} }
		\,.
	\end{align*}

	We will solve
	\begin{align*}
		\rb{- 2 + c} + \rb{5 -  c}c\varepsilon + 2\rb{2 - c}c\varepsilon^2 = 0
		\,.
	\end{align*}
	The solutions to this equation are
	\begin{align*}
		\frac{- \rb{5 -  c}c \pm \sqrt{\rb{5 -  c}^2 c^2 +8 \rb{2 - c}^2 c }}{4\rb{2 - c}c}
		\,,
	\end{align*}                                                                                                    
	where the non-negative solution is
	\begin{align*}
		\frac{- \rb{5 -  c}c + \sqrt{\rb{5 -  c}^2 c^2 +8 \rb{2 - c}^2 c }}{4\rb{2 - c}c}
		\,.
	\end{align*}
	Therefore, as for $\varepsilon=0$ we get that
	\begin{align*}
		\rb{- 2 + c} + \rb{5 -  c}c\varepsilon + 2\rb{2 - c}c\varepsilon^2 = -2 +c \le 0
		\,,
	\end{align*}
	then for every $\varepsilon$ such that
	\begin{align*}
		0\le\varepsilon\le\frac{- \rb{5 -  c}c + \sqrt{\rb{5 -  c}^2 c^2 +8 \rb{2 - c}^2 c }}{4\rb{2 - c}c}
	\end{align*}
	we get that
	\begin{align*}
		\rb{- 2 + c} + \rb{5 -  c}c\varepsilon +\rb{4- 2  c}c\varepsilon^2 \le 0
		\,.
	\end{align*}
	Using \cref{eq:negative derivative eq5}, we get that if
	\begin{align*}
		1\le x \le 1+ \frac{- \rb{5 -  c}c + \sqrt{\rb{5 -  c}^2 c^2 +8 \rb{2 - c}^2 c }}{4\rb{2 - c}c}
	\end{align*}
	then
	\begin{equation}
		\label{eq:negative derivative eq6}
		\begin{aligned}
			0 &\ge \eta^{m-1} (1-x)^{m-2} {x}^{m-2} \sm{x} \rb{ \rb{m-1} \eta \rb{1-2x} + m x \frac{2}{\sm[1]{x}} }\\
			&\qquad+ \eta^{m} (1-x)^{m-1} {x}^{m-1} \sm[m+1]{x} \rb{ m \eta \rb{1-2x} + \rb{m+1} x \frac{2}{\sm[1]{x}} }
			\,.
		\end{aligned}
	\end{equation}
	As, $c\le2$,
	\begin{align*}
		\frac{- \rb{5 -  c}c + \sqrt{\rb{5 -  c}^2 c^2 +8 \rb{2 - c}^2 c }}{4\rb{2 - c}c}
		&\ge \frac{- \rb{5 -  c}c + \sqrt{\rb{5 -  c}^2 c^2 +4 \rb{2 - c}^2 c^2 }}{4\rb{2 - c}c}\\
		&=\frac{- \rb{5 -  c} + \sqrt{\rb{5 -  c}^2+4 \rb{2 - c}^2 }}{4\rb{2 - c}}
		\,.
	\end{align*}
	Therefore, because \cref{eq:derivative upper bound eq num1 b2}, then from \cref{eq:negative derivative eq6} we obtain that
	\begin{equation*}
		\begin{aligned}
			0 &\ge \eta^{m-1} (1-x)^{m-2} {x}^{m-2} \sm{x} \rb{ \rb{m-1} \eta \rb{1-2x} + m x \frac{2}{\sm[1]{x}} }\\
			&\qquad+ \eta^{m} (1-x)^{m-1} {x}^{m-1} \sm[m+1]{x} \rb{ m \eta \rb{1-2x} + \rb{m+1} x \frac{2}{\sm[1]{x}} }
			\,.
		\end{aligned}
	\end{equation*}

	Therefore, we have that for every odd $m$ then
	\begin{align*}
		0 &\ge \eta^{m-1} (1-x)^{m-2} {x}^{m-2} \sm{x} \rb{ \rb{m-1} \eta \rb{1-2x} + m x \frac{2}{\sm[1]{x}} }\\
		&\qquad+ \eta^{m} (1-x)^{m-1} {x}^{m-1} \sm[m+1]{x} \rb{ m \eta \rb{1-2x} + \rb{m+1} x \frac{2}{\sm[1]{x}} }
		\,.
	\end{align*}
	If $D$ is odd, from \cref{eq:negative derivative eq2}, then
	\begin{align*}
		0 \ge \eta^{D-1} (1-x)^{D-2} {x}^{D-2} \sm[D]{x} \rb{ \rb{D-1} \eta \rb{1-2x} + D x \frac{2}{\sm[1]{x}} }
		\,.
	\end{align*}
	Therefore, using \cref{eq:derivative of q},
	\begin{align*}
		\frac{\partial \frac{\qg{x} - 1}{1-x}}{\partial x}
		&= \sum_{m=2}^{D} \eta^{m-1} (1-x)^{m-2} {x}^{m-2} \sm{x} \rb{ \rb{m-1} \eta \rb{1-2x} + m x \frac{2}{\sm[1]{x}} }\\
		&\le \eta^{2-1} (1-x)^{2-2} {x}^{2-2} \sm[2]{x} \rb{ \rb{2-1} \eta \rb{1-2x} + 2 x \frac{2}{\sm[1]{x}} }\\
		&= \eta \sm[2]{x} \rb{ \eta \rb{1-2x} + 2 x \frac{2}{\sm[1]{x}} }
		\,.
	\end{align*}
	
\end{proof}

\section{Proof of Theorem \ref{Thm: convergence theorem}}\label{app:convergence-theorem-proof}

Before going into the proof, we state the following lemma, which we use to show that the iterate sequence converges; see proof in  \cref{sec: Iterates convergences}.
\begin{lemma}
	\label{lem: Iterates convergences}
	If for some $t$ we have that $\wt \in \SG$ and $\lim\limits_{k\rightarrow\infty}\loss\rb{\wt[k]} = 0$ then there exist $\vect{w}^\star\in\Rd$ such that $\lim\limits_{k\rightarrow\infty}\wt[k] = \vect{w}^\star$.
\end{lemma}

We now prove \cref{Thm: convergence theorem}.
\begin{proof}
	Define $\delta^{\rb{k}} \triangleq 2 - \eta \ps{\wt[k]}$.
	Let $\wt\in\SG$ such that $\delta^{\rb{t}} \in (0, 0.40]$, and
	\begin{align}
		\label{eq:loss assumption}
		\loss(\wt) \le \frac{{\delta^{\rb{t}}}^2}{200}
		\,.
	\end{align}
	
	The proof outline is as follows:
	\begin{enumerate}[ref=\theenumi]
		\item\label{itm:thm3.3 case <1} We prove the theorem in the case that $\pi\rb{\wt}\le1$. Note that our assumptions on $\delta^{\rb{t}}$ and the loss will imply
	   \begin{align*}
			\loss\rb{\wt}\le 0.005 \frac{{\delta^{\rb{t}}}^2}{\rb{1-\delta^{\rb{t}}}^2} ~ \text{and} ~ \delta^{\rb{t}} \in (0, 0.44]
			\,.
		\end{align*} 	
	   Our proof for this case consists of two main steps.
		\begin{enumerate}[ref=\theenumi.\theenumii]
			\item\label{itm:thm3.3 case <1, induction} We first prove by induction that for every $k\ge0$
			\begin{align*}
				0 \le \rb{1-\pi\rb{\wt[t+k]}} \rb{-1}^{k} \le \rb{1-\delta^{\rb{t}}}^{k}\rb{1 - \pi\rb{\wt}}
			\end{align*}
			by using \cref{lem:qs GD convergance}, \cref{lem:GPGD to GD} and \cref{lem:GFS sharpness decrease}.
			
			\item\label{itm:thm3.3 case <1, thm} We conclude, using the result received from the use of \cref{lem:GFS sharpness decrease} in the induction, that the theorem is true, i.e. that for any $k\ge0$
			\begin{align*}
				\ps{\wt[t+k]}
				&\ge \frac{2}{\eta}\rb{1 - \delta^{\rb{t}}}
				\,.
			\end{align*}
		\end{enumerate}
	
		\item\label{itm:thm3.3 case >1} We prove the theorem in the case that $1 \le \pi\rb{\wt}$.
		\begin{enumerate}[ref=\theenumi.\theenumii]
			\item\label{itm:thm3.3 case >1, t+1 bound} We show that, in this case, $\wt[t+1]$ fulfill all the conditions for the first case, i.e. that
			\begin{align*}
				1 \ge \pi\rb{\wt[t+1]} , ~ \loss\rb{\wt[t+1]}\le 0.005 \frac{{\delta^{\rb{t+1}}}^2}{\rb{1-\delta^{\rb{t+1}}}^2} ~ \text{and} ~ \delta^{\rb{t+1}} \in (0, 0.44]
				\,,
			\end{align*}
            and thus the theorem applies to $\wt[t+1]$.
            
            \item\label{itm:thm3.3 case >1, error bound} We show that, in this case
            \begin{align*}
            	0 \le 1-\pi\rb{\wt[t+1]} \le 1.26\rb{1-\delta^{\rb{t}}}\rb{\pi\rb{\wt} - 1}
            	\,.
            \end{align*}
            
			\item\label{itm:thm3.3 case >1, thm} We conclude by showing that the theorem is also true for $\wt$. 
		\end{enumerate}
	
		\item\label{itm:thm3.3 step 3} Finally, we use \cref{lem: Iterates convergences} to show that the iterates converges..
	\end{enumerate}

	\hyperref[itm:thm3.3 case <1]{\textbf{Case~\labelcref*{itm:thm3.3 case <1}:}}
	
	First, we prove the theorem for the case where  $1 \ge \pi\rb{\wt}$.
	In this case we use that $\delta^{\rb{t}} \in (0, 0.44]$.
	
	As $\delta^{\rb{t}} \in (0, 0.44]$, and from using \cref{eq:loss assumption}, we obtain that
	\begin{align*}
		\frac{1}{2} \rb{1 - \pi\rb{\wt}}^2 \le \frac{{\delta^{\rb{t}}}^2}{200} \le 0.005 \frac{{\delta^{\rb{t}}}^2}{\rb{1-\delta^{\rb{t}}}^2}
		\,.
	\end{align*}
	Therefore,
	\begin{align}
		\label{eq:pi assumptions}
		1 \ge \pi\rb{\wt} \ge & 1 - 0.1 \frac{\delta^{\rb{t}}}{1 - \delta^{\rb{t}}}
	\end{align}
	From now on, while proving \hyperref[itm:thm3.3 case <1]{case~\labelcref*{itm:thm3.3 case <1}}, we use \cref{eq:pi assumptions} and not \cref{eq:loss assumption}.
	
	\hyperref[itm:thm3.3 case <1, induction]{\textbf{Step~\ref*{itm:thm3.3 case <1, induction}:}}
	
	We prove by induction that for every $k\ge0$ then
	\begin{align*}
		0 \le \rb{1-\pi\rb{\wt[t+k]}} \rb{-1}^{k} \le \rb{1-\delta^{\rb{t}}}^{k}\rb{1 - \pi\rb{\wt}}
		\,,
	\end{align*}
	$\delta^{\rb{t+k}} \in (0,\frac{1}{2}]$, and if $k$ is even then also
	\begin{align*}
		\pi\rb{\wt[t+k]} \ge & 1 - 0.1 \frac{\delta^{\rb{t+k}}}{1 - \delta^{\rb{t+k}}}
		\,.
	\end{align*}
	
	For $k=0$ it is true as $\delta^{\rb{t}} \in (0, 0.44]$, $\pi\rb{\wt[t+0]} = \pi\rb{\wt}$, and from \cref{eq:pi assumptions}.
	
	Let assume that for $k_1 \ge 0$ and for every $k$ such that $0\le k \le k_1$ then
	\begin{align}
		\label{eq:induction assption1}
		0 \le \rb{1-\pi\rb{\wt[t+k]}} \rb{-1}^{k} \le \rb{1-\delta^{\rb{t}}}^{k}\rb{1 - \pi\rb{\wt}}
		\,,
	\end{align}
	$\delta^{\rb{t+k}} \in (0,\frac{1}{2}]$, and if $k$ is even then also
	\begin{align}
		\label{eq:induction assption3}
		\pi\rb{\wt[t+k]} \ge & 1 - 0.1 \frac{\delta^{\rb{t+k}}}{1 - \delta^{\rb{t+k}}}
		\,.
	\end{align}
	We now show it is also true for $k_1+1$. We address both the case that $k_1$ is even and that $k_1$ is odd separately.
	
	 \textbf{If $k_1$ is even} then from \cref{eq:induction assption1} we obtain that $\pi\rb{\wt[t+k_1]}\le1$.
	By using \cref{lem:qs GD convergance}, initialized using $\wt[t+k_1]$, we get that
	\begin{align*}
		\loss\rb{\qs{\wt[t+k_1]}} \le \rb{1 - \delta^{\rb{t+k_1}}}^{2} \loss\rb{\wt[t+k_1]}
		\,,
	\end{align*}
	and that $\pi\rb{\qs{\wt[t+k_1]}}\ge1$.
	Therefore, since $\pi\rb{\wt[t+k_1+1]} = \pi\rb{\qs{\wt[t+k_1]}}$ (from the definition of $\qs{\cdot}$), we obtain that
	\begin{align*}
		0 \le -\rb{1-\pi\rb{\wt[t+k_1+1]}}\le \rb{1-\delta^{\rb{t+k_1}}}\rb{1 - \pi\rb{\wt[t+k_1]}}
		\,.
	\end{align*}
	This implies,
	\begin{align*}
		0 \le -\rb{1-\pi\rb{\wt[t+k_1+1]}}\le \rb{1-\delta^{\rb{t}}}\rb{1 - \pi\rb{\wt[t+k_1]}}
		\,,
	\end{align*}
    since $\delta^{\rb{t+k_1}} \ge \delta^{\rb{t}} \ge 0$ (from \cref{thm:porjected sharpness decrease}) and because $\delta^{\rb{t+k_1}}  \le \frac{1}{2}$.
	Therefore, using \cref{eq:induction assption1}, we obtain that
	\begin{align*}
		0 \le \rb{1-\pi\rb{\wt[t+k_1+1]}} \rb{-1}^{k_1+1} \le \rb{1-\delta^{\rb{t}}}^{k_1+1}\rb{1 - \pi\rb{\wt}}
		\,.
	\end{align*}

    \textbf{If $k_1$ is odd} then from \cref{eq:induction assption1} we obtain that $\pi\rb{\wt[t+k_1-1]}\le1$ and $\pi\rb{\wt[t+k_1]}\ge1$.
	By using \cref{lem:qs GD convergance}, initialized using $\wt[t+k_1-1]$, we get that
	\begin{align*}
		\loss\rb{\qs{\qs{\wt[t+k_1-1]}}} \le \rb{1 - \delta^{\rb{t+k_1-1}}}^{4} \loss\rb{\wt[t+k_1-1]}
		\,.
	\end{align*}
	As $\delta^{(t+k_1)}\le0.5$ and $\wt[t+k_1-1] \in \SG$, by using \cref{lem:GPGD to GD} we obtain that
	\begin{align*}
		\loss\rb{\wt[t+k_1+1]} \le \rb{1 - \delta^{\rb{t+k_1-1}}}^{4} \loss\rb{\wt[t+k_1-1]}
		\,,
	\end{align*}
	and that $\pi\rb{\wt[t+k_1+1]}\le1$.
	This implies,
	\begin{align*}
		0 \le \rb{1-\pi\rb{\wt[t+k_1+1]}}\le \rb{1-\delta^{\rb{t+k_1-1}}}^2\rb{1 - \pi\rb{\wt[t+k_1-1]}}
		\,.
	\end{align*}
	Therefore, since $\delta^{\rb{t+k_1-1}} \ge \delta^{\rb{t}} \ge 0$ (from \cref{thm:porjected sharpness decrease}) and because $\delta^{\rb{t+k_1-1}}  \le \frac{1}{2}$, we obtain that
	\begin{align*}
		0 \le \rb{1-\pi\rb{\wt[t+k_1+1]}}\le \rb{1-\delta^{\rb{t}}}^2\rb{1 - \pi\rb{\wt[t+k_1-1]}}
		\,.
	\end{align*}
	Therefore, using \cref{eq:induction assption1}, we obtain that
	\begin{align*}
		0 \le \rb{1-\pi\rb{\wt[t+k_1+1]}} \rb{-1}^{k_1+1} \le \rb{1-\delta^{\rb{t}}}^{k_1+1}\rb{1 - \pi\rb{\wt}}
		\,.
	\end{align*}

	Overall, both for even and odd $k_1$ we obtain that
	\begin{align}
		\label{eq:induction assption1 k+1}
		0 \le \rb{1-\pi\rb{\wt[t+k_1+1]}} \rb{-1}^{k_1+1} \le \rb{1-\delta^{\rb{t}}}^{k_1+1}\rb{1 - \pi\rb{\wt}}
		\,.
	\end{align}
	
	Next, to show that $\delta^{\rb{t+k_1+1}} \in (0,\frac{1}{2}]$, we define $e\triangleq 1 - \pi\rb{\wt}$ and $M\triangleq 1 + \rb{1-\delta^{\rb{t}}}e$.
	As $\pi\rb{\wt} \le 1$, we get that $e\ge0$ and $e^2 = 2\loss\rb{\wt}$.
	From \cref{eq:pi assumptions} we obtain that
	\begin{align}
		\label{eq:bound on e}
		0\le e\le0.1\frac{\delta^{\rb{t}}}{1 - \delta^{\rb{t}}}
		\,.
	\end{align}
	Thus,
	\begin{align}
		\label{eq:bound on M}
		M = 1+\rb{1 - \delta^{\rb{t}}}e
		\le 1+0.1\delta^{\rb{t}}
		\,.
	\end{align}
	
	Combining \cref{eq:induction assption1} and $\delta^{\rb{t}} \in (0,0.44]$ we obtain that $\max\{1,\pi\rb{\wt[k]}\}\le M$ for every $k$ such that $t\le k\le t+k_1$.
	
	This implies that for every $k$ such that $t\le k\le t+k_1$.
	\begin{equation*}
		\begin{aligned}
			\ps{\wt[k+1]}
			&\ge \frac{\ps{\wt[k]}}{1 + 2\rb{2-\delta^{\rb{k}}}^2  M^2 \loss(\wt[k] )}\\
			&\ge \rb{1 - 2\rb{2-\delta^{\rb{k}}}^2  M^2 \loss(\wt[k] )}\ps{\wt[k]}\\
			&\ge \rb{1 - 2\rb{2-\delta^{\rb{t}}}^2  M^2 \loss(\wt[k] )}\ps{\wt[k]}
			\,,
		\end{aligned}
	\end{equation*}
	where the first inequality is from \cref{lem:GFS sharpness decrease} and the last inequality is because $\delta^{\rb{t}} \le \delta^{\rb{k}}$.
	Therefore,
	\begin{align*}
		\ps{\wt[t+k_1+1]}
		\ge \ps{\wt} \prod_{k=0}^{k_1} \rb{1 - 2\rb{2-\delta^{\rb{t}}}^2  M^2 \loss(\wt[t+k] )}
		\,.
	\end{align*}
	Using \cref{eq:induction assption1} we obtain that
	\begin{align}
		\label{eq:PS lim}
		\ps{\wt[t+k_1+1]}
		\ge \ps{\wt} \prod_{k=0}^{k_1} \rb{1 - 2\rb{2-\delta^{\rb{t}}}^2 \rb{1-\delta^{\rb{t}}}^{2k} M^2 \loss(\wt )}
		\,.
	\end{align}

	Note that for any $k\ge0$
	\begin{align*}
		2\rb{2-\delta^{\rb{t}}}^2 \rb{1-\delta^{\rb{t}}}^{2k} M^2 \loss(\wt)
		&\overset{(1)}{\le} 2\rb{2-\delta^{\rb{t}}}^2 M^2 \loss(\wt)\\
		&\overset{(2)}{\le} 2\rb{2-\delta^{\rb{t}}}^2 \rb{1 + 0.1 \delta^{\rb{t}}}^2 \loss(\wt)\\
		&\overset{(3)}{\le} 8\cdot 1.05^2 \loss(\wt)\\
		&\overset{(4)}{\le} 4 \cdot 1.05^2 \cdot \rb{0.1\frac{\delta^{\rb{t}}}{1 - \delta^{\rb{t}}}}^2\\
		&\overset{(5)}{\le} 4\cdot 1.05^2 \cdot \rb{0.1 \cdot \frac{0.5}{1-0.5}}^2\\
		&< 1
		\,,
	\end{align*}
	where in $(1)\,,(3)\,,$ and $(5)$ inequalities we used $\delta^{\rb{t}} \in (0,0.5]$, in $(2)$ we used \cref{eq:bound on M}, and $(4)$ inequality is because of \cref{eq:bound on e} and $e^2=2\loss\rb{\wt}$.
	We also have that
	\begin{align*}
		2\rb{2-\delta^{\rb{t}}}^2 \rb{1-\delta^{\rb{t}}}^{2k} M^2 \loss(\wt) \ge 0
		\,.
	\end{align*}
	Therefore, combining these results with \cref{eq:PS lim} we obtain that
	\begin{align*}
		\ps{\wt[t+k_1+1]}
		&\ge \ps{\wt} \prod_{k=0}^{k_1} \rb{1 - 2\rb{2-\delta^{\rb{t}}}^2 \rb{1-\delta^{\rb{t}}}^{2k} M^2 \loss(\wt )}\\
		&\ge \ps{\wt} \prod_{k=0}^{\infty} \rb{1 - 2\rb{2-\delta^{\rb{t}}}^2 \rb{1-\delta^{\rb{t}}}^{2k} M^2 \loss(\wt )}\\
		&\ge \ps{\wt} \rb{1 - \sum_{k=0}^{\infty} 2\rb{2-\delta^{\rb{t}}}^2 \rb{1-\delta^{\rb{t}}}^{2k} M^2 \loss(\wt )}
		\,.
	\end{align*}
	By summing the geometric series we obtain that
	\begin{align*}
		\ps{\wt[t+k_1+1]}
		&\ge\ps{\wt} \rb{1 - 2\rb{2-\delta^{\rb{t}}}^2 \frac{1}{1 - \rb{1-\delta^{\rb{t}}}^{2}} M^2 \loss(\wt )}\\
		&= \ps{\wt} \rb{1 - 2 \frac{\rb{2-\delta^{\rb{t}}}^2}{\rb{1 - \rb{1-\delta^{\rb{t}}}}\rb{1 + \rb{1-\delta^{\rb{t}}}}} M^2 \loss(\wt )}\\
		&= \ps{\wt} \rb{1 - 2 \frac{2-\delta^{\rb{t}}}{\delta^{\rb{t}}} M^2 \loss(\wt )}
		\,.
	\end{align*}
	Therefore,
	\begin{align*}
        \ps{\wt[t+k_1+1]}
		&\overset{(1)}{\ge} \ps{\wt} \rb{1 - 2 \frac{2}{\delta^{\rb{t}}} M^2 \loss(\wt )}\\
		&\overset{(2)}{\ge} \ps{\wt} \rb{1 - \frac{0.2\sqrt{2}}{1 - \delta^{\rb{t}}} M^2 \sqrt{\loss(\wt)}}\\
		&\overset{(3)}{\ge} \ps{\wt} \rb{1 - \frac{0.2\sqrt{2}}{1 - \delta^{\rb{t}}} \rb{1 + 0.1 \delta^{\rb{t}}}^2 \sqrt{\loss(\wt)}}
		\,,
	\end{align*}
    where $(1)$ is because $\delta^{\rb{t}} \in (0,0.5]$,
	$(2)$ is because of \cref{eq:bound on e} and $e^2=2\loss\rb{\wt}$, and $(3)$ is because of \cref{eq:bound on M}.
	Therefore, as $\delta^{\rb{t}} \le 0.44$,
	\begin{align}
		\label{eq:ps lim 2}
		\ps{\wt[t+k_1+1]}
		&\ge \ps{\wt} \rb{1 - 0.551 \sqrt{\loss(\wt )}}
		\,.
	\end{align}

	We now show that $\delta^{\rb{t+k_1+1}} \le \frac{1}{2}$, i.e. that $\ps{\wt[t+k_1+1]}\eta \ge 1.5$.
	From \cref{eq:ps lim 2}, we have that,
	\begin{align*}
		\ps{\wt[t+k_1+1]}\eta
		&\ge \eta\ps{\wt} \rb{1 - 0.551 \sqrt{\loss(\wt )}}\\
		&= \rb{2-\delta^{\rb{t}}} \rb{1 - 0.551 \sqrt{\loss(\wt )}}
		\,.
	\end{align*}
	Using \cref{eq:bound on e} and that $2\loss\rb{\wt}=e^2$, we we obtain
	\begin{align}
		\label{eq:ps lim 3}
		\ps{\wt[t+k_1+1]}\eta
		&\ge \rb{2-\delta^{\rb{t}}} \rb{1 - 0.02755 \sqrt{2} \frac{\delta^{\rb{t}}}{1 - \delta^{\rb{t}}}}
		\,.
	\end{align}
	Therefore, as $\delta^{\rb{t}} \le 0.44$, then
	\begin{align*}
		\ps{\wt[t+k_1+1]}\eta
		&\ge 1.5
		\,,
	\end{align*}
	as we want.

 Next, we want to show that if $k_1$ is odd then also
	\begin{align*}
		\pi\rb{\wt[t+k_1+1]} \ge & 1 - 0.1 \frac{\delta^{\rb{t+k_1+1}}}{1 - \delta^{\rb{t+k_1+1}}}
		\,.
	\end{align*}
	If $k_1+1$ is even, i.e. $k_1$ is odd, then from \cref{eq:induction assption1} and \cref{eq:induction assption1 k+1} we obtain that $\pi\rb{\wt[t+k_1+1]} \ge \pi\rb{\wt[t+k_1-1]}$. We have that $\delta^{\rb{t+k_1+1}} \le \frac{1}{2}$, and from \cref{thm:porjected sharpness decrease} we have that $\delta^{\rb{t+k_1+1}} \ge \delta^{\rb{t+k_1-1}}$, therefore
	\begin{align*}
		1 - 0.1 \frac{\delta^{\rb{t+k_1+1}}}{1 - \delta^{\rb{t+k_1+1}}} \le 1 - 0.1 \frac{\delta^{\rb{t+k_1-1}}}{1 - \delta^{\rb{t+k_1-1}}}
		\,.
	\end{align*}
	As result, using \cref{eq:induction assption3}, we obtain that
	\begin{align*}
		\pi\rb{\wt[t+k_1+1]} \ge 1 - 0.1 \frac{\delta^{\rb{t+k_1+1}}}{1 - \delta^{\rb{t+k_1+1}}}
		\,.
	\end{align*}

	This concludes the induction step, i.e., we showed that all the induction assumptions are true for $k_1+1$.
	To summarize, we proved by induction that for every $k\ge0$ then
	\begin{align}
		\label{eq:pi non-monotonic decrease}
		0 \le \rb{1-\pi\rb{\wt[t+k]}} \rb{-1}^{k} \le \rb{1-\delta^{\rb{t}}}^{k}\rb{1 - \pi\rb{\wt}}
		\,,
	\end{align}
	$\delta^{\rb{k}} \in (0,\frac{1}{2}]$, and if $k$ is even then also
	\begin{align*}
		\pi\rb{\wt[t+k]} \ge & 1 - 0.1 \frac{\delta^{\rb{t+k}}}{1 - \delta^{\rb{t+k}}}
		\,.
	\end{align*}
	
	\hyperref[itm:thm3.3 case <1, thm]{\textbf{Step~\ref*{itm:thm3.3 case <1, thm}:}}
	
	Finally, we can now prove the theorem for \hyperref[itm:thm3.3 case <1]{case~\ref*{itm:thm3.3 case <1}}, i.e. the case where $\pi\rb{\wt}\le1$.
	Using \cref{eq:ps lim 3} and that $\delta^{\rb{t}} \in (0,0.44]$, for all $k\ge0$,
	\begin{align*}
		\ps{\wt[t+k]} 
		&\ge \frac{2-\delta^{\rb{t}}}{\eta} \rb{1 - 0.02755 \sqrt{2} \frac{\delta^{\rb{t}}}{1 - \delta^{\rb{t}}}}\\
		&\ge \frac{2-\delta^{\rb{t}}}{\eta} \rb{1 - 0.08 \delta^{\rb{t}}}
		\,.
	\end{align*}
	Therefore,
	\begin{align}
		\label{eq:ps decrease for k}
		\ps{\wt[t+k]} 
		&\ge \frac{1}{\eta}\rb{2 - \delta^{\rb{t}} - 0.08\delta^{\rb{t}}}
		\,.
	\end{align}
	Therefore, for any $k\ge0$,
	\begin{align*}
		\ps{\wt[t+k]}
		&\ge \frac{2}{\eta}\rb{1 - \delta^{\rb{t}}}
		\,.
	\end{align*}
	Finally, from \cref{eq:pi non-monotonic decrease} we obtain that for every $k\ge0$
	\begin{align*}
		\loss\rb{\wt[t+k]} \le \rb{1-\delta^{\rb{t}}}^{2k} \loss\rb{\wt}
		\,.
	\end{align*}
    which concludes the proof for this case.
	
	\hyperref[itm:thm3.3 case >1]{\textbf{Case~\labelcref*{itm:thm3.3 case >1}:}}

	We now handle the case where $1 \le \pi\rb{\wt}$.
	
	Define $e\triangleq \pi\rb{\wt}-1$.
	From \cref{eq:loss assumption}, we obtain that
	\begin{align}
		\label{eq:error bound}
		0 \le e \le \frac{\delta^{\rb{t}}}{10}
		\,.
	\end{align}

	\hyperref[itm:thm3.3 case >1, t+1 bound]{\textbf{Step~\ref*{itm:thm3.3 case >1, t+1 bound}:}}
	
	We now find a lower bound on $\pi\rb{\vect{w}^{\rb{t+1}}}$.
	To find this lower bound, we use the function $\qg{x}$, where the balances used in $\qg{x}$ are the balances of $\vect{w}^{\rb{t}}$.
	We have,
	\begin{align*}
		\int_{1}^{\pi\rb{\vect{w}^{\rb{t}}}}\frac{\partial \frac{\qg{x} - 1}{1-x}}{\partial x} dx
		&\overset{(1)}{\le} \eta  \int_{1}^{\pi\rb{\vect{w}^{\rb{t}}}} \sm[2]{x} \rb{ \eta \rb{1-2x} + 2 x \frac{2}{\sm[1]{x}} } dx\\
		&\overset{(2)}{\le} \eta  \int_{1}^{\pi\rb{\vect{w}^{\rb{t}}}} \rb{ \eta \rb{1-2x} \sm[2]{1}  + 2 x \frac{2\sm[2]{x} }{\sm[1]{x}} } dx\\
		&\overset{(3)}{\le} \eta  \int_{1}^{\pi\rb{\vect{w}^{\rb{t}}}} \rb{ \eta \rb{1-2x} \sm[2]{1}  + 2 x \frac{2\sm[2]{1} }{\sm[1]{1}} } dx\\
		&\overset{(4)}{\le} \eta^2 \ps{\wt[t]}^2  \int_{1}^{\pi\rb{\vect{w}^{\rb{t}}}} \rb{ 1 + 2x\rb{\frac{2}{2-\delta^{\rb{t}}} - 1} }dx
		\,,
	\end{align*}
	where $(1)$ is from \cref{lem:derivative upper bound}, $(2)$ is from \cref{lem:derivative quasi-static}, $(3)$ is from \cref{lem:s2 over s_1}, and $(4)$ is from that $\sm[2]{1} \le \smp[1]{1}{2}$ and that $\sm[1]{1} = \ps{\wt[t]} = \frac{2 - \delta^{\rb{t}}}{\eta}$.
	Therefore,
	\begin{align}
		\int_{1}^{\pi\rb{\vect{w}^{\rb{t}}}}\frac{\partial \frac{\qg{x} - 1}{1-x}}{\partial x} dx
		&\le \eta^2 \pss{\wt[t]}  \int_{1}^{\pi\rb{\vect{w}^{\rb{t}}}} \rb{ 1 + 2x\rb{\frac{2}{2-\delta^{\rb{t}}} - 1} }dx \nonumber\\
		&= \rb{2 - \delta^{\rb{t}}}^2 \rb{ \pi\rb{\vect{w}^{\rb{t}}}  - 1 + \rb{\pi\rb{\vect{w}^{\rb{t}}}^2 - 1} \rb{\frac{2}{2-\delta^{\rb{t}}} - 1} }
		\,.
		\label{eq:q div integral bound}
	\end{align}

	Because from \cref{eq:q at opt} we have that
	\begin{align*}
		\lim\limits_{x\rightarrow 1}\frac{\qg{x} - 1}{1-x} = 1 - \delta^{\rb{t}}
		\,,
	\end{align*}
	we obtain that
	\begin{align*}
		\frac{1 - \pi\rb{\vect{w}^{\rb{t+1}}}}{e}
		= \rb{1 - \delta^{\rb{t}}} + \int_{1}^{\pi\rb{\vect{w}^{\rb{t}}}}\frac{\partial \frac{\qg{z} - 1}{1-z}}{\partial z} dz
		\,.
	\end{align*}
	Consequently, using \cref{eq:q div integral bound},
	\begin{align}
		1 - \pi\rb{\vect{w}^{\rb{t+1}}}
		&=  e \cdot \rb{1 - \delta^{\rb{t}}} + e \int_{1}^{\pi\rb{\vect{w}^{\rb{t}}}}\frac{\partial \frac{\qg{z} - 1}{1-z}}{\partial z} dz\nonumber\\
		&\le e \cdot \rb{1 - \delta^{\rb{t}}} + e \cdot \rb{2 - \delta^{\rb{t}}}^2 \rb{ e + e \cdot \rb{2 + e} \rb{\frac{2}{2-\delta^{\rb{t}}} - 1} }\nonumber\\
		&= e \cdot \rb{1 - \delta^{\rb{t}}} + e^2 \rb{2 - \delta^{\rb{t}}}^2 \rb{ \frac{2}{2-\delta^{\rb{t}}} + \rb{1 + e} \rb{\frac{2}{2-\delta^{\rb{t}}} - 1} }\nonumber\\
		&= e \cdot \rb{ \rb{1 - \delta^{\rb{t}}} + e \cdot \rb{2 - \delta^{\rb{t}}} \rb{ 2 + \rb{1 + e} \delta^{\rb{t}} } } \label{eq:e growth after step}
	\end{align}
	Therefore,
	\begin{align*}
		1 - \pi\rb{\vect{w}^{\rb{t+1}}}
		&\le e \cdot \rb{ \rb{1 - \delta^{\rb{t}}} + e \cdot \rb{4 - {\delta^{\rb{t}}}^2} + e^2 \cdot \rb{2 - \delta^{\rb{t}}} \delta^{\rb{t}} }
		\,.
	\end{align*}
	Since $\delta^{\rb{t}} \in (0,2)$ and $\rb{2 - \delta^{\rb{t}}} \delta^{\rb{t}} \le 1$ then
	\begin{align*}
		1 - \pi\rb{\vect{w}^{\rb{t+1}}}
		&\le e \cdot \rb{ 1 + 4e + e^2 }
		\,.
	\end{align*}
	From \cref{eq:error bound} we obtain that
	\begin{align}
		\label{eq:next pi}
		1 - \pi\rb{\vect{w}^{\rb{t+1}}}
		&\le 0.1\delta^{\rb{t}}\rb{ 1 + 4\frac{\delta^{\rb{t}}}{10} + \frac{{\delta^{\rb{t}}}^2}{100} }
		\,.
	\end{align}
	
	We now prove that
	\begin{align*}
		\rb{1 - \delta^{\rb{t}}} + 4\frac{\delta^{\rb{t}}}{10}\rb{1 - \delta^{\rb{t}}} + \frac{{\delta^{\rb{t}}}^2}{100}\rb{1 - \delta^{\rb{t}}} \le 1
		\,.
	\end{align*}
	We have that, for $x\ge0$,
	\begin{align*}
		\frac{\partial \rb{ \rb{1 - x} + 4\frac{x}{10}\rb{1 - x} + \frac{{x}^2}{100}\rb{1 - x} }}{\partial x}
		&= -1 + \frac{4}{10} - \frac{8}{10}x + \frac{2}{100}x - \frac{3}{100}x^2\\
		&\le 0
		\,.
	\end{align*}
	Consequently,
	\begin{align*}
		\rb{1 - x} + 4\frac{x}{10}\rb{1 - x} + \frac{{x}^2}{100}\rb{1 - x}
	\end{align*}
	is monotonically decreasing over $x\ge0$.
	Thus, as $\delta^{\rb{t}} \in (0,0.4]$, we have that
	\begin{align*}
		\rb{1 - \delta^{\rb{t}}} + 4\frac{\delta^{\rb{t}}}{10}\rb{1 - \delta^{\rb{t}}} + \frac{{\delta^{\rb{t}}}^2}{100}\rb{1 - \delta^{\rb{t}}}
		&\le \rb{1 - 0} + 4\frac{0}{10}\rb{1 - 0} + \frac{0^2}{100}\rb{1 - 0}
		= 1
		\,.
	\end{align*}
	Therefore,
	\begin{align*}
		1 + 4\frac{\delta^{\rb{t}}}{10} + \frac{{\delta^{\rb{t}}}^2}{100} \le \frac{1}{1 - \delta^{\rb{t}}}
		\,.
	\end{align*}
	Hence, by using \cref{eq:next pi} we obtain
	\begin{align*}
		1 - \pi\rb{\vect{w}^{\rb{t+1}}}
		&\le 0.1\frac{\delta^{\rb{t}}}{1 - \delta^{\rb{t}}}
		\,.
	\end{align*}
	In addition, because $\wt\in\SG$, \cref{lem:g_i decrease} and that \cref{eq:q at opt}, we obtain that
	\begin{align}
		\label{eq:pi after one step}
		\pi\rb{\vect{w}^{\rb{t+1}}} \le 1
		\,.
	\end{align}
	
	From \cref{lem:GFS sharpness decrease} and \cref{eq:error bound}, we obtain that
	\begin{equation*}
		\begin{aligned}
			\ps{\wt[t+1]}
			&\ge \frac{\ps{\wt[k]}}{1 + \rb{2-\delta^{\rb{k}}}^2  \rb{1 + e}^2 e^2}\\
			&\ge \rb{1 - 4 \rb{1 + e}^2 e^2}\ps{\wt}\\
			&\ge \rb{1 - 4 \rb{1 + \frac{\delta^{\rb{t}}}{10}}^2 \frac{{\delta^{\rb{t}}}^2}{100}}\ps{\wt}
			\,.
		\end{aligned}
	\end{equation*}
	Therefore, as $\delta^{\rb{t}} \in (0,0.4]$,
	\begin{align}
		\eta\ps{\wt[t+1]}
		&\ge \rb{1 - 0.0174 \delta^{\rb{t}}}\eta\ps{\wt}\nonumber\\
		&\ge \rb{1 - 0.0174 {\delta^{\rb{t}}}} \rb{2 - \delta^{\rb{t}}} \label{eq:ps t to t+1}\\
		&\ge 1.58 \nonumber
		\,.
	\end{align}
	Consequently, by using that $\delta^{\rb{t}} > 0$ and that \cref{thm:porjected sharpness decrease}, we get that
	\begin{align*}
		\delta^{\rb{t+1}} \in (0,0.44]
		\,.
	\end{align*}

	Finally, we get that $\wt[t+1]$ satisfy all the assumption needed for \hyperref[itm:thm3.3 case <1]{case~\ref*{itm:thm3.3 case <1}}, i.e the case where $\pi\rb{\vect{w}^{\rb{t}}} \le 1$.
	
	\hyperref[itm:thm3.3 case >1, error bound]{\textbf{Step~\ref*{itm:thm3.3 case >1, error bound}:}}
	From \cref{eq:e growth after step} and \cref{eq:pi after one step} we obtain that
	\begin{align*}
		0 \le
		1 - \pi\rb{\vect{w}^{\rb{t+1}}}
		&\le e \cdot \rb{1 - \delta^{\rb{t}}} \cdot \rb{ 1 + e \cdot \rb{1 + \frac{1}{1 - \delta^{\rb{t}}}} \rb{ 2 + \rb{1 + e} \delta^{\rb{t}} } }
		\,.
	\end{align*}
	Therefore, using \cref{eq:error bound} we get that
	\begin{align*}
		0 \le
		1 - \pi\rb{\vect{w}^{\rb{t+1}}}
		&\le e \cdot \rb{1 - \delta^{\rb{t}}} \cdot \rb{ 1 + \frac{\delta^{\rb{t}}}{10} \cdot \rb{1 + \frac{1}{1 - \delta^{\rb{t}}}} \rb{ 2 + \rb{1 + \frac{\delta^{\rb{t}}}{10}} \delta^{\rb{t}} } }
		\,.
	\end{align*}
	And as $\delta^{\rb{t}} \in (0,0.4]$ we obtain that
	\begin{align}
		\label{eq:pi non-monotonic decrease one step}
		0 \le
		1 - \pi\rb{\vect{w}^{\rb{t+1}}}
		&\le 1.26 \rb{1 - \delta^{\rb{t}}} \rb{\pi\rb{\vect{w}^{\rb{t}}} - 1}
		\,.
	\end{align}
	
	\hyperref[itm:thm3.3 case >1, thm]{\textbf{Step~\ref*{itm:thm3.3 case >1, thm}:}}
	
	We now use the fact that $\wt[t+1]$ satisfy all the assumption needed for \hyperref[itm:thm3.3 case <1]{case~\ref*{itm:thm3.3 case <1}}, i.e the case where $\pi\rb{\vect{w}^{\rb{t}}} \le 1$, to prove that the theorem is true for \hyperref[itm:thm3.3 case >1]{case~\ref*{itm:thm3.3 case >1}}, i.e the case where $\pi\rb{\vect{w}^{\rb{t}}} \ge 1$.
	
	As $\wt[t+1]$ satisfy all the assumption needed for \hyperref[itm:thm3.3 case <1]{case~\ref*{itm:thm3.3 case <1}}, then by using \cref{eq:ps decrease for k}, we obtain that for every $k\ge1$
	\begin{align}
		\label{eq:ps t+1 to t+k}
		\ps{\wt[t+k]} 
		&\ge \frac{1}{\eta}\rb{2 - \delta^{\rb{t+1}} - 0.08\delta^{\rb{t+1}}}
		\,.
	\end{align}
	By using \cref{eq:ps t to t+1}, we obtain
	\begin{align*}
		\delta^{\rb{t+1}} &\le 2 - \rb{1 - 0.0174 \delta^{\rb{t}}} \rb{2 - \delta^{\rb{t}}}\\
		&\le \delta^{\rb{t}} + 0.0348 \delta^{\rb{t}}
		\,.
	\end{align*}
	As a consequence from \cref{eq:ps t+1 to t+k}, we obtain that
	\begin{align*}
		\ps{\wt[t+k]} 
		&\ge \frac{1}{\eta}\rb{2 - \delta^{\rb{t+1}} - 0.08\delta^{\rb{t+1}}}\\
		&\ge \frac{1}{\eta}\rb{2 - \delta^{\rb{t}} - 0.0348 \delta^{\rb{t}} - 0.08\delta^{\rb{t}} - 0.08 \cdot 0.0348 \delta^{\rb{t}}}\\
		&\ge \frac{1}{\eta}\rb{2 - 1.15\delta^{\rb{t}} }
		\,.
	\end{align*}
	Therefore,
	\begin{align*}
		\ps{\wt[t+k]}
		&\ge \frac{2}{\eta}\rb{1 - \delta^{\rb{t}}}
		\,.
	\end{align*}
	
	In addition, as $\wt[t+1]$ satisfy all the assumption needed for \hyperref[itm:thm3.3 case <1]{case~\ref*{itm:thm3.3 case <1}}, then by using \cref{eq:pi non-monotonic decrease} and \cref{eq:pi non-monotonic decrease one step} we obtain that for every $k\ge0$ then
	\begin{align*}
		0 \le \rb{1-\pi\rb{\wt[t+k]}} \rb{-1}^{k+1} \le 1.26\rb{1-\delta^{\rb{t}}}^{k}\rb{ \pi\rb{\wt} - 1 }
		\,.
	\end{align*}
	Thus, for every $k\ge0$ then
	\begin{align*}
		\loss\rb{\wt[t+k]} \le 1.6 \rb{1-\delta^{\rb{t}}}^{2k} \loss\rb{\wt}
		\,.
	\end{align*}

	\hyperref[itm:thm3.3 step 3]{\textbf{Step~\labelcref*{itm:thm3.3 step 3}:}}
	
	Finally, as conclusion from both cases of $\pi\rb{\wt} \le1$ and $\pi\rb{\wt} \ge1$, we obtain that for $\wt\in\SG$ such that $\delta^{\rb{t}} \in (0, 0.40]$, and
	\begin{align*}
		\loss(\wt) \le \frac{{\delta^{\rb{t}}}^2}{200}
		\,,
	\end{align*}
	then for any $k\ge0$
	\begin{align*}
		&\ps{\wt[t+k]}
		\ge \frac{2}{\eta}\rb{1 - \delta^{\rb{t}}} ~~\text{and}\\
		&\loss\rb{\wt[t+k]} \le 1.6 \rb{1-\delta^{\rb{t}}}^{2k} \loss\rb{\wt}
		\,.
	\end{align*}
	By using \cref{lem: Iterates convergences} we get that there exist minimum $\vect{w}^\star$ such that
	\begin{align*}
		&\lim\limits_{k\rightarrow\infty}\wt[k] = \vect{w}^\star ~~\text{and}\\
		&\ps{\vect{w}^\star}
		\ge \frac{2}{\eta}\rb{1 - \delta^{\rb{t}}}
		\,.
	\end{align*}
	
\end{proof}

\subsection{Proof of Lemma \ref{lem: Iterates convergences}}
\label{sec: Iterates convergences}

\begin{proof}
	Let $\wt \in \SG$ and  $\lim\limits_{k\rightarrow\infty}\loss\rb{\wt[k]} = 0$.
	Therefore, because of the definition of the loss function $\loss$ (see \cref{eq:loss function})
	\begin{align}
		&\lim\limits_{k\rightarrow\infty}\rb{\pi\rb{\wt[k]} - 1}^2 = 0 ~~\text{and} \label{eq:convergace of pi and err err}\\
		&\lim\limits_{k\rightarrow\infty}\pi^2\rb{\wt[k]} = 1 \label{eq:convergace of pi and err pi}
			\,.
	\end{align}

	Define $\tilde{\vect{w}}^{\rb{k}} \triangleq \Pgf{\wt[k]}$.
	Because that $\wt \in \SG$ and \cref{thm:porjected sharpness decrease} we get that for any $k\ge t$ then
	\begin{align*}
		\wt[k] \in \SG~~\text{and}\\
		\ps{\wt[k]} \le \ps{\wt}
		\,.
	\end{align*}
	Therefore, from \cref{eq:sharpness at opt main text}, we obtain that for any $k\ge t$ and for any $i,j\in\D$
	\begin{align}
		\label{eq:bound on frac wk}
		\frac{1}{ \rb{\tilde{\vect{w}}^{\rb{k}}_i \tilde{\vect{w}}^{\rb{k}}_j}^2 } \le s_1^2\rb{ \tilde{\vect{w}}^{\rb{k}} } = \pss{\wt[k]} \le \pss{\wt}
		\,.
	\end{align}

	For $k\ge t$ and $i,j\in\D$ we have two cases, $\pi\rb{\wt[k]}\le 1$ and $\pi\rb{\wt[k]}\ge 1$.
	\begin{itemize}
		\item If $\pi\rb{\wt[k]}\le 1$ then, by using \cref{lem:derivative quasi-static}, we obtain that
		\begin{align*}
			\frac{\pi^2\rb{\wt[k]} }{ \rb{\vect{w}^{\rb{k}}_i \vect{w}^{\rb{k}}_j}^2 }
			\le \frac{ \pi^2\rb{\tilde{\vect{w}}^{\rb{k}}} }{ \rb{\tilde{\vect{w}}^{\rb{k}}_i \tilde{\vect{w}}^{\rb{k}}_j}^2 }
			= \frac{ 1 }{ \rb{\tilde{\vect{w}}^{\rb{k}}_i \tilde{\vect{w}}^{\rb{k}}_j}^2 }
			\,.
		\end{align*}
		Consequently, using \cref{eq:bound on frac wk}, we obtain that
		\begin{align*}
			\eta^2 \rb{ \pi\rb{\wt[k]} - 1 }^2 \frac{ \pi^2\rb{\wt[k]} }{ \rb{\vect{w}^{\rb{k}}_i \vect{w}^{\rb{k}}_j}^2 }
			\le \eta^2 \rb{ \pi\rb{\wt[k]} - 1 }^2 \pss{\wt}
			\,.
		\end{align*}
	
		\item If $\pi\rb{\wt[k]} \ge 1$ then, by using \cref{lem:derivative quasi-static}, we obtain that
		\begin{align*}
			\frac{ 1 }{ \rb{\vect{w}^{\rb{k}}_i \vect{w}^{\rb{k}}_j}^2 }
			\le \frac{ 1 }{ \rb{\tilde{\vect{w}}^{\rb{k}}_i \tilde{\vect{w}}^{\rb{k}}_j}^2 }
			\,.
		\end{align*}
		Therefore, using \cref{eq:bound on frac wk}, we obtain that
		\begin{align*}
			\eta^2 \rb{ \pi\rb{\wt[k]} - 1 }^2 \frac{ \pi^2\rb{\wt[k]} }{ \rb{\vect{w}^{\rb{k}}_i \vect{w}^{\rb{k}}_j}^2 }
			\le \eta^2 \rb{ \pi\rb{\wt[k]} - 1 }^2 \pi^2\rb{\wt[k]} \pss{\wt}
			\,.
		\end{align*}
	\end{itemize}
	Therefore, as a consequence of \cref{eq:convergace of pi and err err} and \cref{eq:convergace of pi and err pi}, we obtain that there exist $K\ge t$ such that for all $k\ge K$ and $i,j\in \D$
	\begin{align*}
		\eta^2 \rb{ \pi\rb{\wt[k]} - 1 }^2 \frac{ \pi^2\rb{\wt[k]} }{ \rb{\vect{w}^{\rb{k}}_i \vect{w}^{\rb{k}}_j}^2 }
		< 1
		\,.
	\end{align*}
	Define $\vect{b}^{\rb{k}}$ as the balances of $\wt[k]$ (as defined in \cref{def:balances appendix}).
	Consequently, from \cref{dyn:balances}, we obtain that there exist $\vect{b}^\star \in \Rd[D \times D]$ such that
	\begin{align*}
		\lim\limits_{k\rightarrow\infty} \vect{b}^{\rb{k}} = \vect{b}^\star
		\,.
	\end{align*}

	Let $\vect{w}^\star \in \Rd$ be the weight with balances of $\vect{b}^\star$, such that $\pi\rb{\vect{w}^\star} = 1$ and $\vect{w}^\star$ have the same signs (element-wise) as $\wt$.
	Because $\lim\limits_{k\rightarrow\infty}\pi^2\rb{\wt[k]} = 1$ (\cref{eq:convergace of pi and err pi}) and $\lim\limits_{k\rightarrow\infty} \vect{b}^{\rb{k}} = \vect{b}^\star$ we obtain that $\lim\limits_{k\rightarrow\infty} {\wt[k]}^2 = {\vect{w}^\star}^2$.
	Furthermore, because of \cref{lem:GD does not change sign}, we get that $\lim\limits_{k\rightarrow\infty} \wt[k] = \vect{w}^\star$.
\end{proof}

\section{Discussion of \titleSG}\label{Sec: Discussion on assumption}

For any $\vect{w}\in \SG$ there exist $B>1$ such that for every $\vect{w}'\in E_{\ivr{0}{B}}(\vect{w})$ then $\pi\rb{\gd{\vect{w}'}}\in \ivr{0}{B}$.
We explain when such $B$ exist.

We define the maximal value that the product of weights in $E_{\iv{0}{1}}(\vect{w})$ can have after a single GD step. 
\begin{align*}
	B^{-}_{\eta}\rb{ \vect{w} } \triangleq \max\{ \pi\rb{ \gd{ \vect{w'} } } | \vect{w'}\in E_{\iv{0}{1}}(\vect{w}) \}
	\,.
\end{align*}

We define the minimal value of the product of weights in $E_{\ivi{1}}(\vect{w})$, where performing a single GD step leads that the product of the weight being zero.
\begin{align*}
	B^{+}_{\eta}\rb{ \vect{w} } \triangleq \min\{ \pi\rb{ \vect{ w' } } | w'\in E_{\ivi{1}}(\vect{w}) \wedge \pi\rb{ \gd{ \vect{w'} } } = 0 \}
	\,.
\end{align*}

A value $B>1$ fulfill that for every $\vect{w}'\in E_{\ivr{0}{B}}(\vect{w})$ then $\pi\rb{\gd{\vect{w}'}}\in \ivr{0}{B}$, if and only if
$B \in \left( B^{-}_{\eta}\rb{ \vect{w} }, B^{+}_{\eta}\rb{ \vect{w} } \right]$.

\subsection{Proof of Lemma \ref{lem:GD does not change sign}}
\label{sec:GD does not change sign}

\begin{proof}
	Let $\vect{w} \in \SG$.
	Define $I\triangleq \iv{\pi\rb{\vect{w}}}{1}$ if $\pi\rb{\vect{w}} \le 1$, and $I\triangleq \iv{1}{\pi\rb{\vect{w}}}$ if $\pi\rb{\vect{w}} \ge 1$.
	
	As $\vect{w}\in \SG$ then for any $\vect{w}'\in E_I\rb{\vect{w}}$ we have that $\vect{w}'\in \SG$.
	Therefore, $\pi\rb{\gd{\vect{w}'}} \neq 0$, and thus for every $i\in\D$ then $\gdb{\vect{w}'}_i\neq 0$.
	Therefore, as GD (\cref{eq:GD exact update rule}) is continuous, $\gd{\Pgf{\vect{w}}} = \Pgf{\vect{w}}$ and for any $\vect{w}'\in E_I\rb{\vect{w}}\,, i\in\D$ then $w'_i$ has the same sign as $w_i$, we obtain that for any $\vect{w}'\in E_I\rb{\vect{w}}\,,i\in\D$ then $\gdb{\vect{w}}_i$ has the same sign as $w_i$.
	
	Therefore, for any $\vect{w} \in \SG\,, i \in \D$ then $\gdb{\vect{w}}_i$ has the same sign as $w_i$.
\end{proof}

\subsection{Subset order of \titleSG}

\begin{lemma}
	\label{lem:subset order of SG}
	For any step sizes $\eta>\eta'>0$ we have $\SG \subseteq \SG[\eta']$.
\end{lemma}

\begin{proof}
	 Let step sizes $\eta>0$, $\eta'\in(0,\eta]$ and weights $\vect{w}\in\Rd$ such that $\vect{w}\in\SG$.
	 
	 Because of \cref{asmp:proj_sharp_dec}, we get that there exist $B>1$ s.t. \begin{enumerate}
	 	\item $\pi(\vect{ w }) \in \ivr{0}{B}$.
	 	\item $\forall \vect{w}'\in E_{\ivr{0}{B}}(\vect{w})\,:\,\pi\rb{\gd{\vect{w}'}}\in \ivr{0}{B}$.
	 	\item $\ps{\vect{w}} \le \frac{2\sqrt{2}}{\eta}$.
	 \end{enumerate}
	 
	 Let $\vect{w}'\in E_{\ivr{0}{B}}(\vect{w})$.
	 We assume without a loss of generality that $\vect{w}, \vect{w}' \in \Rpd$ (see \cref{sec:Equivalence of weights}).
	 We now show that $\pi\rb{\gd[\eta']{\vect{w}'}}\in \ivr{0}{B}$. We handle the cases that $\vect{w}'\in E_{\ivr{0}{1}}(\vect{w})$ and $\vect{w}'\in E_{\ivr{1}{B}}(\vect{w})$ separately.
	 \begin{itemize}
	 	\item If $\pi\rb{\vect{w}'} \le 1$ then for every $i\in\D$
	 	\begin{align*}
	 		0 \le
	 		-\eta' \rb{\pi\rb{\vect{w}'} - 1} \frac{\pi\rb{\vect{w}'}}{w'_i}
	 		\le -\eta \rb{\pi\rb{\vect{w}'} - 1} \frac{\pi\rb{\vect{w}'}}{w'_i}
	 		\,.
	 	\end{align*}
 		Using the GD update rule, as defined in \cref{eq:GD exact update rule},
 		\begin{align*}
 			0 < \pi\rb{\vect{w}'} \le \pi\rb{\gd[\eta']{\vect{w}'}}
 			\,,
 		\end{align*}
 		and
 		\begin{align*}
 			\pi\rb{\gd[\eta']{\vect{w}'}} \le \pi\rb{\gd{\vect{w}'}} < B
 			\,.
 		\end{align*}
 		Thus, we obtain that $\pi\rb{\gd[\eta']{\vect{w}'}}\in \ivr{0}{B}$.
 		
 		\item If $\pi\rb{\vect{w}'} \ge 1$ then for every $i\in\D$
 		\begin{align*}
 			0\ge -\eta' \rb{\pi\rb{\vect{w}'} - 1} \frac{\pi\rb{\vect{w}'}}{w'_i}
 			\ge \eta \rb{\pi\rb{\vect{w}'} - 1} \frac{\pi\rb{\vect{w}'}}{w'_i}
 			\,.
 		\end{align*}
 		Therefore, for every $i\in\D$
 		\begin{align*}
 			w'_i
 			\ge \gd[\eta']{\vect{w}'}_i
 			\ge \gdb{\vect{w}'}_i \overset{(1)}{>} 0
 			\,,
 		\end{align*}
 		where $(1)$ if because of \cref{lem:GD does not change sign} and $\pi\rb{\gd{\vect{w}'}} \neq 0$.
 		Consequently,
 		\begin{align*}
 			0
 			< \pi\rb{\gd{\vect{w}'}}
 			\le \pi\rb{\gd[\eta']{\vect{w}'}}
 			\,,
 		\end{align*}
 		and
 		\begin{align*}
 			\pi\rb{\gd[\eta']{\vect{w}'}}
 			\le \pi\rb{\vect{w}'}
 			< B
 			\,.
 		\end{align*}
	 	Thus, we obtain that $\pi\rb{\gd[\eta']{\vect{w}'}}\in \ivr{0}{B}$.
	 \end{itemize}
	 
	 Therefore, for all $\vect{w}'\in E_{\ivr{0}{B}}(\vect{w})$ then $\pi\rb{\gd[\eta']{\vect{w}'}}\in \ivr{0}{B}$.
	 In addition $\pi(\vect{ w }) \in \ivr{0}{B}$, and $\ps{\vect{w}} \le \frac{2\sqrt{2}}{\eta} \le \frac{2\sqrt{2}}{\eta'}$.

	 Thus, we obtain that $\vect{w}\in\SG[\eta']$.
	 Therefore, $\SG \subseteq \SG[\eta']$.
\end{proof}

\end{document}